%% file: preprint_version.tex
\documentclass{article} 
\usepackage{iclr2027_conference,times}

\input{math_commands.tex}

\usepackage{hyperref}
\usepackage{url}
\usepackage{booktabs}       
\usepackage{amsfonts}       
\usepackage{nicefrac}       
\usepackage{microtype}      
\usepackage{xcolor}         

\usepackage{bm}
\usepackage{amsmath}
\usepackage{amssymb}
\usepackage{mathtools}
\usepackage{amsthm}

\usepackage{multirow, multicol}

\theoremstyle{plain}
\newtheorem{theorem}{Theorem}[section]
\newtheorem{proposition}[theorem]{Proposition}

\theoremstyle{definition}

\theoremstyle{remark}



\usepackage{algorithm}
\usepackage{algpseudocode}
\algrenewcommand\alglinenumber[1]{\scriptsize #1:}

\iclrfinalcopy

\title{Dynamic Generalized Gromov-Wasserstein Optimal Transport}

\author{Junda Ying\textsuperscript{1}\thanks{Equal contribution}, Zhiwei Zeng\textsuperscript{2}\footnotemark[1], Peijie Zhou\textsuperscript{2,3,4,5}\thanks{Corresponding authors: pjzhou@pku.edu.cn, zhangl@math.pku.edu.cn}, Lei Zhang\textsuperscript{1,2,3,6,7}\footnotemark[2]\\
  \textsuperscript{1}Beijing International Center for Mathematical Research, Peking University \\
  \textsuperscript{2}Center for Quantitative Biology, Peking University \\
  \textsuperscript{3}Center for Machine Learning Research, Peking University \\
  \textsuperscript{4}National Engineering Laboratory for Big Data Analysis and Applications, Beijing\\
  \textsuperscript{5}AI for Science Institute, Beijing\\
  \textsuperscript{6}School of Mathematical Sciences, Peking University\\
  \textsuperscript{7}Institute for Artificial Intelligence, Peking University\\
}
\begin{document}

\maketitle

\fancyhead{}
\renewcommand{\headrulewidth}{0pt}

\begin{abstract}
Gromov--Wasserstein optimal transport (GW-OT) extends classical optimal transport by introducing structure-aware transport cost. This is particularly relevant for spatial transcriptomics, where dynamical reconstruction should preserve tissue structure in addition to matching expression patterns. While static formulations have been widely used for such structure-aware alignment, a general dynamic formulation for reconstructing continuous trajectories is still missing. We introduce \textbf{T}ravelling \textbf{P}air \textbf{D}ynamical \textbf{A}lignment and \textbf{T}rajectory \textbf{E}stimation (TP-DATE), a theoretical and computational framework to generalize GW-OT dynamically in a simulation-free manner. We formulate a broad class of static and dynamic Quadratic-form OT (QOT) through path actions and prove the static dynamic equivalence. We further develop travelling-pair flow matching, which allows interacting conditional paths and marginalizes their interactions into a single vector field. On synthetic and real spatial transcriptomics data, TP-DATE better preserves spatial structure and improves continuous 3D dynamics reconstruction.
\end{abstract}

\section{Introduction}
\label{sec:introduction}

Optimal transport (OT) \citep{kantorovich1942translocation} provides a principled way to match probability distributions by minimizing transport cost, while its dynamic formulation \citep{benamou2000computational} lifts this endpoint matching into a continuous time evolution of measures. This perspective has made OT a natural tool for reconstructing dynamics from population snapshots, including single-cell systems where individual cells cannot be tracked longitudinally. Static OT aligns snapshots by infering couplings across time points, whereas dynamic OT enables continuous interpolation and trajectory inference between observed snapshots \citep{waddingot,trajectorynet,cfm_tong,TIGON,DeepRUOT,wfr_fm}. Extensions such as Schr\"odinger bridges \citep{schrodinger1932sur,SBsurvey} and Wasserstein--Fisher--Rao \citep{chizat2018unbalanced,chizat2018interpolating,liero2018optimal} further broaden this framework to stochastic and unbalanced dynamics.

However, pointwise transport cost alone may be insufficient when the data carry meaningful internal structure such as spatial structure. Gromov--Wasserstein OT (GW-OT) compares pairwise relations within distributions and are therefore well suited to structure-aware matching tasks \cite{gwot}. Fused GW-OT (FGW-OT) further combines feature and structural information \citep{FGWOT,moscot}. Such methods have become useful for structured data including graphs, heterogeneous domains, and spatial omics. However, most GW-like OTs still lack a dynamic formulation, therefore can only align snapshots instead of reconstructing a continuous time dynamics.

Though a dynamic GW-like formulation may not be meaningful when aligning snapshots from different modalities, it becomes important when the snapshots represent different states of the same structured system. Spatial transcriptomics is a canonical example. Snapshots collected at different times describe the same tissue evolving over time, and a meaningful interpolation should recover not only gene expression changes but also the continuous evolution of spatial organization \citep{MOSTA,ARTISTA}. The same principle applies to serial tissue sections, where the interpolation axis is tissue depth rather than time and unseen intermediate sections correspond to physically meaningful states \citep{3DTumor}.

Recent work has begun to generalize GW-OT from complementary directions. On the static side, \citep{QOT} defines a family of static Quadratic-form optimal transport (QOT) including GW-OT as a special case. On the dynamic side, inner-product GW-OT (IGW-OT) \citep{IGWOT} develops and studies a dynamic formulation for a particular GW-like OT based on gradient flow and Riemannian geometry. However, a general dynamic QOT formulation and its theory is still missing, as is an efficient simulation-free algorithm for solving them.

To address these limitations, We introduce \textbf{T}ravelling \textbf{P}air \textbf{D}ynamical \textbf{A}lignment and \textbf{T}rajectory \textbf{E}stimation (TP-DATE), a general framework for dynamic QOT together with a simulation-free travelling-pair flow matching solver. Our contributions are summarized as follows.

\begin{itemize}
    \item We developed a mathematical theory of dynamic QOT and included standard OT, GW-OT, and IGW-OT as special cases of our QOT framework.
    \item We developed a simulation-free travelling pair flow matching framework for learning dynamics with interacting conditional paths and solving the dynamic QOT problems.
    \item We proposed a dynamic fusion of OT and GW-OT, and demonstrate the ability of it in several biological tasks including spatiotemporal dynamics and 3D structure reconstruction.
\end{itemize}

\section{Related works}
\label{sec:related works}
\textbf{Optimal transport and extensions.} Optimal transport \citep{kantorovich1942translocation} and its dynamic formulations 
\citep{benamou2000computational} have been widely used to reconstruct dynamics from snapshot data. Stochastic counterparts \citep{schrodinger1932sur,SBsurvey} and unbalanced extensions \citep{chizat2018interpolating,liero2018optimal} have further been applied to model more complex dynamics. Another important line of extension is QOT, represented by formulations such as GW-OT and FGW-OT \citep{gwot,FGWOT,moscot}, which can model global structure preservation during transport. Recently, \citep{QOT} introduced a static QOT framework that unifies a broad class of such extensions, while \citep{IGWOT} developed IGW-OT as a dynamic formulation for a particular QOT problem. We propose TP-DATE, a unified dynamic formulation for a broad class of QOT problems which is compatible with these existing formulations.

\textbf{Flow matching based optimal transport solvers.}
Flow matching \citep{cfm_lipman} is an efficient simulation-free generative modeling framework that has been used to solve OT \citep{cfm_tong}, Schr\"odinger bridge \citep{tong2023simulationfree}, WFR \citep{wfr_fm}, and a variety of OT-based dynamics reconstruction problems \citep{UOT,rathod2026contextflowcontextawareflowmatching,klein2024genot,USB}. Inspired by the conditional path technique and the travelling Dirac in optimal transport \citep{chizat2018interpolating}, we formulate dynamic QOT from a path action viewpoint and develop a flow matching solver for this class of problems.

\textbf{OT-based spatial transcriptomics dynamics reconstruction.} Several recent methods have been developed for spatiotemporal dynamics inference from spatial transcriptomics. On the static alignment side, \citep{moscot,PASTE,PASTE2,SOCS} formulate cross-time or spatial alignment via GW-OT and FGW-OT or related structure-aware OT objectives, while \citep{DeSTOT} encodes temporal and spatial structural information into static transport objectives. On the dynamic side, \citep{rathod2026contextflowcontextawareflowmatching} incorporates structure-aware static couplings into flow matching. \citep{peng2026stvcr} introduces structure-preserving terms directly into the dynamic transport objective, and \citep{zhang2025cytobridge} explicitly models interactions within the learned dynamics, both by simulation-based NeuralODE \citep{NeuralODE}. This leaves a clear gap for TP-DATE, a simulation-free flow-matching framework that directly models structure-aware transport dynamics.

\section{Preliminaries}
\label{sec:preliminaries}
\textbf{Static optimal transport.}
Let $\mathcal{P}(\mathcal{X})$ be the set of all probability densities supported on $\mathcal{X}$ for some $\mathcal{X}\subset\mathbb{R}^d$. $\mu_0,\mu_1$ are probability densities in $\mathcal{P}(\mathcal{X})$. \(\Pi(\mu_0,\mu_1)\) denotes the set of all couplings.
\[
\Pi(\mu_0,\mu_1)=\left\{\gamma\in\mathcal{P}(\mathcal{X}^2)|\int_{\mathcal{X}}\gamma(\bm{x},\bm{y})\mathrm{d} \bm{y}=\mu_0(\bm{x}),\ \int_{\mathcal{X}}\gamma(\bm{x},\bm{y})\mathrm{d} \bm{x}=\mu_1(\bm{y})\right\}
\]
The static optimal transport (OT), also known as the Kantorovich form \citep{kantorovich1942translocation}, is defined as 
\begin{equation}
\label{eq:static OT}
\begin{aligned}
&\text{OT}(\mu_0,\mu_1) =\inf_{\gamma\in\Pi(\mu_0,\mu_1)} \int_{\mathcal{X}^2}  \, C(\bm{x},\bm{y})\gamma(\bm{x},\bm{y})\mathrm{d} \bm{x}\mathrm{d} \bm{y}\\
\end{aligned}
\end{equation}
where \(C(\bm{x},\bm{y})\) represents the cost of transporting unit mass from \(\bm{x}\) to \(\bm{y}\), and \(\gamma(\bm{x},\bm{y})\in \Pi(\mu_0,\mu_1)\) is called the coupling. From an optimization perspective, (\ref{eq:static OT}) can be viewed as a linear programming w.r.t the coupling \(\gamma\). As a well known example, when choosing \(C(\bm{x},\bm{y})=\frac{1}{2}\Vert\bm{x}-\bm{y}\Vert^2\), which is the square of the Euclidean distance, the infimum is called the square of 2-Wasserstein distance (\(\mathcal{W}_2\)).

\textbf{Static quadratic-form optimal transport.} 
Recently, a static quadratic-from optimal transport (QOT) problem is defined and studied mathematically \citep{QOT}. Given two spaces \(\mathcal{X},\mathcal{Y}\), and a cost function \(C:(\mathcal{X}\times\mathcal{Y})^2\to\mathbb{R}\), the static QOT is defined as
\begin{equation}
\label{eq:general static QOT}
\begin{aligned}
&\text{QOT}(\mu_0,\mu_1) =\inf_{\gamma\in\Pi(\mu_0,\mu_1)} \int_{(\mathcal{X}\times\mathcal{Y})^2}  \, C(\bm{x},\bm{y},\bm{x}',\bm{y}')\gamma(\bm{x},\bm{y})\gamma(\bm{x}',\bm{y}')\mathrm{d} \bm{x}\mathrm{d} \bm{y}\mathrm{d} \bm{x}'\mathrm{d} \bm{y}'\\
\end{aligned}
\end{equation}
If we choose \(C(\bm{x},\bm{y},\bm{x}',\bm{y}')=|d_{\mathcal{X}}(\bm{x},\bm{x}')-d_{\mathcal{Y}}(\bm{y},\bm{y}')|^2\) where \((\mathcal{X},d_{\mathcal{X}}),(\mathcal{Y},d_{\mathcal{Y}})\) are two metric spaces, (\ref{eq:general static QOT}) recovers the Gromov-Wasserstein OT (GW-OT). This type of QOT has been widely used in single-cell trajectory inference \citep{gwot,moscot}. Intuitively, GW-OT preserves the local structure after transport. If we choose  \(C(\bm{x},\bm{y},\bm{x}',\bm{y}')=f(\bm{x},\bm{y})+g(\bm{x}',\bm{y}')\), (\ref{eq:general static QOT}) just reduces to the standard static OT with cost \(f+g\). Different to static OT, static QOT is a quadratic programming w.r.t the coupling \(\gamma\).

\textbf{Dynamic optimal transport.}
For static OT, \citep{benamou2000computational} established a dynamic form for \(\mathcal{W}_2\) case, also known as the BB-form.
\begin{equation}
\label{eq:dynamic OT}
\begin{aligned}
&\text{OT}(\mu_0,\mu_1) =\inf_{\rho,\bm{u}} \int_0^1\int_{\mathcal{X}}  \, \frac{1}{2}\Vert \bm{u}(\bm{x},t)\Vert_2^2\rho_t(\bm{x})\mathrm{d} \bm{x}\mathrm{d}t \\
&\text{s.t.} \ \ \ \ \ \ \ \partial_t\rho+\nabla_{\bm{x}}\cdot(\rho \bm{u})=0,\ \rho_0=\mu_0,\ \rho_1=\mu_1
\end{aligned}
\end{equation}
They proved the equivalence between (\ref{eq:dynamic OT}) and (\ref{eq:static OT}) when \(C(\bm{x},\bm{y})=\frac{1}{2}\Vert\bm{x}-\bm{y}\Vert^2\). Intuitively, it aims to find a continuous probability flow connecting \(\mu_0,\mu_1\) which also minimizes the total kinetic energy. With nice fluid dynamics interpretation, it has also been widely used in single-cell trajectory inference \citep{trajectorynet,cfm_tong,klein2024genot}.

\textbf{Travelling Dirac.}
To solve dynamic OT, one can first consider the dynamic OT between two Dirac measures
\begin{equation}
\label{eq:Dirac OT}
\text{OT}(\delta_{\bm{x}_0},\delta_{\bm{x}_1}) =\inf_{\bm{x}_t} \frac{1}{2}\int_0^1\Vert \dot{\bm{x}}_t\Vert_2^2\mathrm{d}t \quad \text{s.t.}\quad \bm{x}_0=\bm{x}_0,\quad\bm{x}_1=\bm{x}_1
\end{equation}
which yields the displacement interpolation \(\bm{x}_t=(1-t)\bm{x}_0+t\bm{x}_1\), also known as the travelling Dirac. More complex travelling Dirac can also be derived for different type of OT, such as unbalanced OT \citep{chizat2018unbalanced,chizat2018interpolating}. For a specific dynamic OT, previous works obtained its solution by integrating these travelling Diracs over the static coupling \(\gamma\) \citep{cfm_tong,wfr_fm}. Therefore, the dynamic OT can decoupled to two parts: the travelling Dirac and the static coupling. Integration can be realized by flow matching.

\textbf{Flow matching.} Flow matching is a simulation-free generative framework for learning a continuous probability flow from data \citep{cfm_lipman}. Given \(\mu_0,\mu_1\), it aims to learn a marginal velocity field \(\bm{u}_t(\bm{x})\), such that \(\partial_t\rho+\nabla_{\bm{x}}\cdot(\rho \bm{u})=0\) and \(\rho_0=\mu_0,\rho_1=\mu_1\), which means \(\bm{u}_t(\bm{x})\) transports \(\mu_0\) to \(\mu_1\). They parameterized a neural network \(\bm{u}_{\bm{\theta}}\) to approximate the true \(\bm{u}\). The neural networks are trained to minimize the marginal regression loss.
\begin{equation}
\begin{aligned}
&\mathcal{L}_{\text{FM}}(\bm{\theta})=\int_{0}^{1}\int_\mathcal{X} \left\| \bm{\bm{u}_{\theta}}(\bm{x},t) - \bm{u}_t(\bm{x}) \right\|_2^2\rho_t(\bm{x})\mathrm{d} \bm{x}\mathrm{d}t
\label{eq:FM}
\end{aligned}
\end{equation}
Although the true \(\rho, \bm{u}\) are intractable, they proved that minimizing the loss above is equivalent to minimize the conditional regression loss
\begin{equation}
\begin{aligned}
\label{eq:CFM}
&\mathcal{L}_{\text{CFM}}(\bm{\theta})=
\mathbb{E}_{t\sim\mathcal{U}[0,1], \bm{z}\sim q(\bm{z}), \bm{x}\sim \rho_t(\bm{x}\vert \bm{z})}\left\| \bm{\bm{u}_{\theta}}(\bm{x},t) - \bm{u}_t(\bm{x}\vert\bm{z}) \right\|_2^2
\end{aligned}
\end{equation}
where \(\bm{z}\sim q(\bm{z})\) is some conditional variable and \(\partial_t\rho_t(\bm{x}\vert\bm{z})+\nabla_{\bm{x}}\cdot(\rho_t(\bm{x}\vert\bm{z}) \bm{u}_t(\bm{x}\vert\bm{z}))=0\) is called the conditional continuity equation, which the conditional probability flow \(\rho_t(\bm{x}\vert \bm{z})\) and the conditional velocity \(\bm{u}_t(\bm{x}\vert \bm{z})\) satisfy. The marginal probability flow satisfies \(\rho_t(\bm{x})=\int\rho_t(\bm{x}\vert \bm{z})q(\bm{z})\mathrm{d}\bm{z}\). The marginalization theorem states that the marginal velocity \(\bm{u}_t(\bm{x})=\int\bm{u}_t(\bm{x}\vert\bm{z})\frac{\rho_t(\bm{x}\vert \bm{z})q(\bm{z})}{\rho_t(\bm{x})}\mathrm{d}\bm{z}\) transports \(\mu_0\) to \(\mu_1\). Therefore, one can learn a admissible \(\bm{u}_t(\bm{x})\) by designing tractable \(q(\bm{z})\) and conditional path \(\rho_t(\bm{x}|\bm{z})\) with tractable conditional velocity \(\bm{u}_t(\bm{x}|\bm{z})\), and applying flow matching.

By careful design, flow matching can be used for solving dynamic OT. Following \citep{cfm_tong}, one can choose \(\bm{z}=(\bm{x}_0,\bm{x}_1)\sim\gamma(\bm{x}_0,\bm{x}_1)\) drawn from the static OT coupling, set the travelling Dirac as conditional path, and derive the conditional velocity from it. The resulting flow are proved to recover the dynamic OT flow. Under this framework, travelling Diracs are integrated over the static OT coupling independently, which means particles actually move independently. In this work, we develop a flow matching framework which allows conditional paths to interact.

\section{Dynamic QOT}
\label{sec:dynamic QOT}
In this section, we consider one metric space \(\mathcal{X}\subset\mathbb{R}^d\), i.e. \(\mathcal{Y}=\mathcal{X}\), and develop a dynamic formulation for QOT under this condition. Since the flow have to live in some space, the meaning of a dynamic flow connecting two totally different spaces needs further study. All the proofs are left to \ref{app:proofs}. 

\subsection{Static and Dynamic form}
\textbf{Path action.} To model conditional paths with interaction, we study how a pair is transported to a pair rather than travelling Dirac. We call the corresponding path \textbf{travelling pair}. Let \(\bm{z}=(\bm{x}_0,\bm{x}_1),\bm{z}'=(\bm{y}_0,\bm{y}_1)\), we define the action of a pair path for some Lagrangian \(\mathcal{L}\) with sufficient regularity
\begin{equation}
\label{eq:travelling pair}
\mathcal{A}(\bm{z},\bm{z}') =\inf_{\bm{x}_t,\bm{y}_t} \int_0^1\mathcal{L}(t,\bm{x}_t,\bm{y}_t,\dot{\bm{x}}_t,\dot{\bm{y}}_t)\mathrm{d}t
\end{equation}
The minimizer yields a conditional velocity pair \(\dot{\bm{x}}_t=\bm{u}_t^1(\bm{x}_t,\bm{y}_t|\bm{z},\bm{z}'),\dot{\bm{y}}_t=\bm{u}_t^2(\bm{x}_t,\bm{y}_t|\bm{z},\bm{z}')\). It can also be viewed as over the travelling Pair measure path \(\pi_t(\bm{x},\bm{y}|\bm{z},\bm{z}')\)
\begin{equation}
\begin{aligned}
&\mathcal{A}(\bm{z},\bm{z}') =\inf_{\pi,\bm{u}^1,\bm{u}^2} \int_0^1\int_{\mathcal{X}^2} \mathcal{L}(t,\bm{x},\bm{y},u^1_t(\bm{x},\bm{y}|\bm{z},\bm{z}'),u^2_t(\bm{x},\bm{y}|\bm{z},\bm{z}'))\pi_t(\bm{x},\bm{y}|\bm{z},\bm{z}')\mathrm{d} \bm{x}\mathrm{d} \bm{y}\mathrm{d}t\\
&\partial_t\pi_t(\bm{x},\bm{y}|\bm{z},\bm{z}')+\nabla_{\bm{x}}\cdot(\pi_t(\bm{x},\bm{y}|\bm{z},\bm{z}')\bm{u}_t^1(\bm{x},\bm{y}|\bm{z},\bm{z}'))+\nabla_{\bm{y}}\cdot(\pi_t(\bm{x},\bm{y}|\bm{z},\bm{z}')\bm{u}_t^2(\bm{x},\bm{y}|\bm{z},\bm{z}'))=0\
\end{aligned}
\end{equation}
with boundary condition \(\pi_0(\bm{x},\bm{y}|\bm{z},\bm{z}')=\delta_{(\bm{x}_0,\bm{y}_0)}(\bm{x},\bm{y}),\pi_1(\bm{x},\bm{y}|\bm{z},\bm{z}')=\delta_{(\bm{x}_1,\bm{y}_1)}(\bm{x},\bm{y})\).

\textbf{Static form.} Based on the path action \(\mathcal{A}\), we can define the corresponding static QOT as
\begin{equation}
\label{eq:static QOT}
\begin{aligned}
&\text{QOT}_S(\mu_0,\mu_1) =\inf_{\gamma\in\Pi(\mu_0,\mu_1)} \int_{\mathcal{X}^4} \mathcal{A}(\bm{z},\bm{z}')\gamma(\bm{z})\gamma(\bm{z}')\mathrm{d} \bm{z}\mathrm{d} \bm{z}'\\
\end{aligned}
\end{equation}
The optimal coupling can be solved by standard quadratic programming algorithms, such as Frank-Wolfe \citep{kerdoncuff2021sampled,pot}. We left the details to \ref{app:static QOT solver}.

\textbf{Dynamic form.} Inspired by the travelling Dirac technique in standard OT, we define the dynamic QOT on the travelling pair path level. The corresponding dynamic probability flow is \(\pi_t(\bm{x},\bm{y})=\int\pi_t(\bm{x},\bm{y}|\bm{z},\bm{z}')\gamma(\bm{z})\gamma(\bm{z}')\mathrm{d}\bm{z}\mathrm{d}\bm{z}'.\)
\begin{equation}
\label{eq:dynamic QOT}
\begin{aligned}
&\text{QOT}_{D}(\mu_0,\mu_1) = \\&\inf\int_0^1\int\int_{\mathcal{X}^2} \mathcal{L}(t,\bm{x},\bm{y},u^1_t(\bm{x},\bm{y}|\bm{z},\bm{z}'),u^2_t(\bm{x},\bm{y}|\bm{z},\bm{z}'))\pi_t(\bm{x},\bm{y}|\bm{z},\bm{z}')\gamma(\bm{z})\gamma(\bm{z}')\mathrm{d} \bm{x}\mathrm{d} \bm{y}\mathrm{d} \bm{z}\mathrm{d} \bm{z}'\mathrm{d}t
\end{aligned}
\end{equation}

The infimum is taken over the coupling \(\gamma\) and all travelling pairs. A fluid dynamics form, or say BB-form, analogous to OT \citep{benamou2000computational}, is also defined as
\begin{equation}
\label{eq:BB QOT}
\begin{aligned}
&\text{QOT}_{BB}(\mu_0,\mu_1) =\inf_{\pi,\bm{u}^1,\bm{u}^2} \int_0^1\int_{\mathcal{X}^2} \mathcal{L}(t,\bm{x},\bm{y},u^1_t(\bm{x},\bm{y}),u^2_t(\bm{x},\bm{y}))\pi_t(\bm{x},\bm{y})\mathrm{d} \bm{x}\mathrm{d} \bm{y}\mathrm{d}t\\
&\partial_t\pi_t(\bm{x},\bm{y})+\nabla_{\bm{x}}\cdot(\pi_t(\bm{x},\bm{y})u^1_t(\bm{x},\bm{y}))+\nabla_{\bm{y}}\cdot(\pi_t(\bm{x},\bm{y})u^2_t(\bm{x},\bm{y}))=0
\end{aligned}
\end{equation}
where \(\pi_t\) is a continuous probability flow on the two-particle space \(\mathcal{X}^2\), with boundary value \(\mu_0\otimes\mu_0,\mu_1\otimes\mu_1\). The marginal velocity pair \(\bm{u}^1_t,\bm{u}^2_t\) are particle velocities with interaction effect. The formulation intuitively seeks for a two-particle flow minimizing the total action. Our first result is the equivalence between the static and dynamic form. However, these forms are upper bound of the BB-form but not equal to it, hence a surrogate. We discuss the details and difficulties in \ref{app:BB-form}. Fortunately, since both the static and dynamic formulations provide upper bounds on the BB-form, solving dynamic QOT also implicitly minimizes the total energy in the sense of the BB-form.

\begin{theorem}
\label{thm:static dynamic equivalence}
If \(\mathcal{L}\) is convex w.r.t \((\dot{\bm{x}},\dot{\bm{y}})\), then \(\text{QOT}_S(\mu_0,\mu_1)=\text{QOT}_D(\mu_0,\mu_1)\ge \text{QOT}_{BB}(\mu_0,\mu_1)\).
\end{theorem}

\subsection{Lagrangian with tractable conditional velocity pair}
\label{sec:tractable lagrangian}
In the next section, we established a flow matching framework for solving dynamic QOT flows. In order to do that, we need to access to the conditional velocity pair. Therefore, we then consider what kind of \(\mathcal{L}\) leads to tractable conditional velocity. Note that TP-DATE is not restricted to the choices introduced below. As long as the conditional velocity pair can be obtained in some way, TP-DATE remains applicable. To take kinetic energy into account, we consider a broad class of \(\mathcal{L}\) with the form
\begin{equation}
\mathcal{L}(t,\bm{x},\bm{y},\dot{\bm{x}},\dot{\bm{y}})=\frac{1}{2}\Vert\dot{\bm{x}}\Vert^2+\frac{1}{2}\Vert\dot{\bm{y}}\Vert^2+\lambda\Phi(\bm{x}-\bm{y},\dot{\bm{x}}-\dot{\bm{y}})
\end{equation}
where \(\Phi\) represents a convex interaction term, making it different to standard OT. To simplify the pair dynamics, we change the variables to \(\bm{c}_t=\frac{\bm{x}_t+\bm{y}_t}{2},\bm{q}_t=\bm{x}_t-\bm{y}_t\), which are the barycenter and the relative displacement of the pair. The Lagrangian becomes
\begin{equation}
\mathcal{L}(t,\bm{x},\bm{y},\dot{\bm{x}},\dot{\bm{y}})=\Vert\dot{\bm{c}}_t\Vert^2+\frac{1}{4}\Vert\dot{\bm{q}}_t\Vert^2+\lambda\Phi(\bm{q},\dot{\bm{q}})
\end{equation}
Since the convexity preserves under affine transformation, it is sufficient to choose \(\Phi\) to be convex w.r.t \(\dot{\bm{q}}\). One natural choice is \(\Phi(\bm{q})\), which is independent of \(\dot{\bm{q}}\). If \(\Phi\) is further radial, then the travelling pair has a low dimensional structure, hence learnable without the curse of dimensionality.
\begin{proposition}
\label{prop:SUDO path}
The optimal \(\bm{c}_t\) is \(\bm{c}_t=(1-t)\bm{c}_0+t\bm{c}_1\). If \(\Phi=\Phi(\Vert\bm{q}\Vert)\), \(\bm{q}_0\nparallel\bm{q}_1\), then \(\bm{q}_t\in\text{span}\{\bm{q}_0,\bm{q}_1\}\).
\end{proposition}
\textbf{GW-OT as a limit case.} Another family of meaningful \(\Phi\) is \(\Phi=|\frac{\mathrm{d}}{\mathrm{d}t}\phi(\bm{q}_t)|^2\). It penalizes the variation of the relative displacement. The most explicit choice is \(\Phi=|\frac{\mathrm{d}}{\mathrm{d}t}\Vert\bm{q}_t\Vert|^2\). Intuitively, this seeks for a transport not only saving kinetic energy, but also trying to minimize the distortion. This \(\Phi\) leads to analytic conditional velocity.
\begin{proposition}
\label{prop:RK path}
If \(\Phi=|\frac{\mathrm{d}}{\mathrm{d}t}\phi(\bm{q}_t)|^2\), \(\Phi\) is convex w.r.t \(\dot{\bm{q}}_t\). If \(\Phi=|\frac{\mathrm{d}}{\mathrm{d}t}\Vert\bm{q}_t\Vert|^2\), \(\bm{q}_t\) has analytic solution.
\end{proposition}
A notable property of this choice of \(\Phi=|\frac{\mathrm{d}}{\mathrm{d}t}\Vert\bm{q}_t\Vert|^2\) is the following theorem. 
\begin{theorem}
\label{thm:GW-OT}
\(\Phi=|\frac{\mathrm{d}}{\mathrm{d}t}\Vert\bm{q}_t\Vert|^2\). When \(\lambda=0\) and \(d\ge2\), the corresponding static QOT reduces to standard OT, while when \(\lambda\to+\infty\), it reduces to standard GW-OT: \(\lambda^{-1}\text{QOT}_{S}(\mu_0,\mu_1)\to\text{GW-OT}(\mu_0,\mu_1)\).
\end{theorem}
We therefore refer to this choice of Lagrangian as a \textbf{dynamic fusion of OT and GW-OT}. Of note, this is different from the existing FGW-OT formulation \citep{FGWOT,moscot}. FGW-OT is obtained by taking a weighted combination directly at the static level, whereas dynamic fusion introduces the weighting at the level of the dynamic formulation. Consequently, its induced static formulation is not FGW-OT. Also, by choosing proper \(\Phi\), the recently considered static IGW-OT \citep{IGWOT} can also be realized as a special case of our dynamic QOT framework. We discuss the relation between these formulations in \ref{app:GW and FGW} \ref{app:IGW-OT}.

\textbf{Modality separation.} In some settings, the coordinates may consist of multiple modalities, and we may wish to impose different interaction terms on different modalities. We illustrate the formulation using the two modality case, the extension to multiple modalities follows naturally. Let \(\bm{x}=(\bm{x}^1,\bm{x}^2),\bm{y}=(\bm{y}^1,\bm{y}^2)\), and the Lagrangian is separated to the two modalities
\begin{equation}
\label{eq:modality separation}
\mathcal{L}(t,\bm{x},\bm{y},\dot{\bm{x}},\dot{\bm{y}})=\mathcal{L}_1(t,\bm{x}^1,\bm{y}^1,\dot{\bm{x}}^1,\dot{\bm{y}}^1)+\mathcal{L}_2(t,\bm{x}^2,\bm{y}^2,\dot{\bm{x}}^2,\dot{\bm{y}}^2).
\end{equation}
Then one can easily check that the conditional velocities of the two modalities can also be computed separately. The conditional velocity pair of \((\bm{x},\bm{y})\) is exactly the concatenation of the conditional velocity pairs induced by \(\mathcal{L}_i\).

\section{Travelling pair flow matching}
In this section, we develop a flow matching framework for solving the dynamic QOT problem. As OT-CFM \citep{cfm_tong} averages travelling Diracs over OT coupling, our framework allows to average travelling pairs over QOT coupling. All the proofs are left to \ref{app:proofs}.
\subsection{Marginalization theory}
\textbf{Pair marginalization.}
Given conditional probability paths \(\pi_t(\bm{x},\bm{y}|\bm{z},\bm{z}')\) and conditional velocity pairs \(\bm{u}_t^1(\bm{x},\bm{y}|\bm{z},\bm{z}'),\bm{u}_t^2(\bm{x},\bm{y}|\bm{z},\bm{z}')\) satisfying the conditional continuity equation, let the conditional variables \((\bm{z},\bm{z}')\sim q(\bm{z},\bm{z}')\), and define the marginal probability flow as \(\pi_t(\bm{x},\bm{y})=\int\pi_t(\bm{x},\bm{y}|\bm{z},\bm{z}')q(\bm{z},\bm{z}')\mathrm{d}\bm{z}\mathrm{d}\bm{z}'\), the marginalization theorem on the two-particle space holds.
\begin{theorem}
\label{thm:pair marginalization}
Define the marginal velocity as \(\bm{u}^i_t(\bm{x},\bm{y})=\int\bm{u}^i_t(\bm{x},\bm{y}|\bm{z},\bm{z}')\frac{\pi_t(\bm{x},\bm{y}|\bm{z},\bm{z}')q(\bm{z},\bm{z}')}{\pi_t(\bm{x},\bm{y})}\mathrm{d}\bm{z}\mathrm{d}\bm{z}'\), the marginals satisfy the marginal continuity equation
\begin{equation}
\label{eq:pair continuity equation}
\partial_t\pi_t+\nabla_{\bm{x}}\cdot(\pi_t\bm{u}^1_t)+\nabla_{\bm{y}}\cdot(\pi_t\bm{u}^2_t)=0
\end{equation}
If \(q=\gamma\otimes\gamma\) for some \(\gamma\in\Pi(\mu_0,\mu_1)\), the marginal velocity transports \(\mu_0\otimes\mu_0\) to \(\mu_1\otimes\mu_1\).
\end{theorem}

\textbf{Particle marginalization.} The marginal velocity \(\bm{u}^1_t(\bm{x},\bm{y})\) can be interpreted as how particle \(\bm{x}\) travels given the influence of \(\bm{y}\). From this perspective, one can obtain a marginal velocity \(\bm{v}_t(\bm{x})\) only depends on \(\bm{x}\) by averaging the influences of all possible \(\bm{y}\).
\begin{equation}
\label{eq:particle velocity}
\bm{v}_t(\bm{x})=\mathbb{E}_{\pi_t}[\bm{u}^1_t(\bm{X}_t,\bm{Y}_t)|\bm{X}_t=\bm{x}]=\int_{\mathcal{X}}\bm{u}^1_t(\bm{x},\bm{y})\pi_t(\bm{y}|\bm{x})\mathrm{d}\bm{y}
\end{equation}
This velocity only depends on the particle itself, but it already contains all the interaction information by taking conditional expectation. We further define the \(\bm{x}\) marginal \(\rho_t(\bm{x})=\int_{\mathcal{X}}\pi_t(\bm{x},\bm{y})\mathrm{d}\bm{y}\). For these single particle marginals, we also have a marginalization theorem.
\begin{theorem}
\label{thm:particle marginalization}
The single particle marginals satisfy the continuity equation
\begin{equation}
\label{eq:particle continuity equation}
\partial_t\rho_t+\nabla_{\bm{x}}\cdot(\rho_t\bm{v}_t)=0
\end{equation}
\end{theorem}
Equivalently, it means \(\bm{v}_t(\bm{x})\) transports \(\mu_0\) to \(\mu_1\). It can be viewed as the projection of the pair flow on the single particle space. However, it is worth clarifying that the two-particle dynamics generally cannot be recovered from the single-particle marginal dynamics. They are not equivalent.

\textbf{Mean-field approximation.} We further decouple the velocity \(\bm{u}^1_t\) into two parts \(\bm{u}^1_t(\bm{x},\bm{y})=\bm{b}_t(\bm{x})+\bm{f}_t(\bm{x},\bm{y})\). \(\bm{b}_t\) represents the self-driven term, and \(\bm{f}_t\) represents the interaction term. Under this decomposition, we have
\begin{equation}
\bm{v}_t(\bm{x})=\bm{b}_t(\bm{x})+\int_{\mathcal{X}}\bm{f}_t(\bm{x},\bm{y})\pi_t(\bm{y}|\bm{x})\mathrm{d}\bm{y}
\end{equation}
The particle continuity equation (\ref{eq:particle continuity equation}) becomes
\begin{equation}
\label{eq:interaction continuity equation}
\partial_t\rho_t(\bm{x})+\nabla_{\bm{x}}\cdot[\rho_t(\bm{x})\big(\bm{b}_t(\bm{x})+\int_{\mathcal{X}}\bm{f}_t(\bm{x},\bm{y})\pi_t(\bm{y}|\bm{x})\mathrm{d}\bm{y}\big)]=0
\end{equation}

Under independence approximation \(\pi_t(\bm{x},\bm{y})\approx\rho_t(\bm{x})\rho_t(\bm{y})\), (\ref{eq:interaction continuity equation}) reduces to the mean-field flow.
\begin{equation}
\label{eq:mean-field flow equation}
\partial_t\rho_t(\bm{x})+\nabla_{\bm{x}}\cdot[\rho_t(\bm{x})\big(\bm{b}_t(\bm{x})+\int_{\mathcal{X}}\bm{f}_t(\bm{x},\bm{y})\rho_t(\bm{y})\mathrm{d}\bm{y}\big)]=0
\end{equation}
Though the independence may not hold accurately due to the particle interactions, the difference between the continuous measure path generated by the mean-field approximation (\ref{eq:mean-field flow equation}) and the true dynamics (\ref{eq:interaction continuity equation}) can be controlled by the difference between \(\pi_t\) and \(\rho_t\otimes\rho_t\).
\begin{proposition}
\label{prop:error bound}
Let the true probability flow be \(\rho_t\) and the mean-field approximation be \(\hat{\rho}_t\). Assume \(\bm{b}_t,\bm{f}_t\) are both Lipschitz, and \(M=\underset{0<t\le1}{\max}\underset{\bm{x}}{\max}\mathcal{W}_1(\rho_t,\pi_t(\cdot|\bm{x}))\), then \(\forall 0<t\le1,\exists C_t>0\) such that \(\mathcal{W}_1(\rho_t,\hat{\rho}_t)\le C_tM\).
\end{proposition}

\subsection{Training loss design} 
\textbf{Marginal velocity pair.} Analog to conditional flow matching \citep{cfm_lipman}, when travelling pairs are tractable, we can parameterize two neural networks \(\bm{u}^i_{\bm{\theta}}(\bm{x},\bm{y},t),i=1,2\) and regress the true marginal velocity pair by minimizing the conditional flow matching loss. This further leads us to the solution to dynamic QOT problem (\ref{eq:dynamic QOT}). 
\begin{equation}
\label{eq:pair loss}
\mathcal{L}_{\text{pair}}(\bm{\theta})=
\mathbb{E}_{t\sim\mathcal{U}[0,1], (\bm{z},\bm{z}')\sim q(\bm{z},\bm{z}'), (\bm{x},\bm{y})\sim \pi_t(\bm{x},\bm{y}\vert \bm{z},\bm{z}')}\sum_{i=1}^2\left\| \bm{\bm{u}^i_{\theta}}(\bm{x},\bm{y},t) - \bm{u}^i_t(\bm{x},\bm{y}\vert\bm{z},\bm{z}') \right\|_2^2
\end{equation}
\begin{theorem}
\label{thm:solve dynamic QOT}
The minimizer is \(\bm{\bm{u}^i_{\theta}}(\bm{x},\bm{y},t)=\bm{u}^i_t(\bm{x},\bm{y})\). When \(q=\gamma\otimes\gamma\) where \(\gamma\) is the optimal QOT coupling of (\ref{eq:static QOT}), and the travelling pair used is the minimizer of (\ref{eq:travelling pair}), then the learned velocity pair generates the probability flow of the dynamic QOT problem (\ref{eq:dynamic QOT}).
\end{theorem}
\textbf{Particle velocity.} A new design is that we can also obtain the particle velocity \(\bm{v}_t(\bm{x})\) (\ref{eq:particle velocity}) by flow matching. We can parameterize a single neural network \(\bm{v}_{\bm{\theta}}(\bm{x},t)\) and try to minimize the intractable marginal regression loss \(\mathcal{L}_{M}(\bm{\theta)}=\mathbb{E}_{t\sim\mathcal{U}[0,1],\bm{x}\sim\rho_t(\bm{x})}\Vert\bm{v}_{\bm{\theta}}(\bm{x},t)-\bm{v}_t(\bm{x})\Vert^2\). The minimizer is \(\bm{v}_t(\bm{x})\). A tractable conditional version is
\begin{equation}
\label{eq:particle CFM}
\mathcal{L}_{C}(\bm{\theta)}=\mathbb{E}_{t\sim\mathcal{U}[0,1],(\bm{z},\bm{z}')\sim q(\bm{z},\bm{z}'),(\bm{x},\bm{y})\sim\pi_t(\bm{x},\bm{y}\vert\bm{z},\bm{z}')}\Vert\bm{v}_{\bm{\theta}}(\bm{x},t)-\bm{u}^1_t(\bm{x},\bm{y}\vert\bm{z},\bm{z}')\Vert^2
\end{equation}
\begin{theorem}
\label{thm:particle loss}
\(\nabla_{\bm{\theta}}\mathcal{L}_M(\bm{\theta})=\nabla_{\bm{\theta}}\mathcal{L}_C(\bm{\theta})\).
\end{theorem}
In practice, if the symmetry condition holds, i.e. \(q=\gamma\otimes\gamma\) and \(\mathcal{L}\) is symmetric to \((\bm{x},\dot{\bm{x}}),(\bm{y},\dot{\bm{y}})\), then it is easy to show that \(\bm{u}^1_t(\bm{y},\bm{x}\vert\bm{z},\bm{z}')=\bm{u}^2_t(\bm{x},\bm{y}\vert\bm{z},\bm{z}')\). In this case, the conditional loss (\ref{eq:particle CFM}) can be replaced by a convex combination of itself and \(\Vert\bm{v}_{\bm{\theta}}(\bm{y},t)-\bm{u}^2_t(\bm{x},\bm{y}\vert\bm{z},\bm{z}')\Vert^2\) to make the training more efficient. In most QOT settings, this condition obviously holds.

\textbf{Velocity decomposition.} If there are some gauges of velocity decomposition \(\bm{u}^1_t(\bm{x},\bm{y}|\bm{z},\bm{z}')=\bm{b}_t(\bm{x})+\bm{f}_t(\bm{x},\bm{y}|\bm{z},\bm{z}')\), it is also possible to regress the interaction part \(\bm{f}\) via flow matching. Let's parameterize a neural network \(\bm{f}_{\bm{\phi}}(\bm{x},\bm{y},t)\) and minimize the conditional loss below.
\begin{equation}
\label{eq:bf loss}
\begin{aligned}
\mathcal{L}_f(\bm{\phi})&=
\mathbb{E}_{t\sim\mathcal{U}[0,1], (\bm{z},\bm{z}')\sim q(\bm{z},\bm{z}'), (\bm{x},\bm{y})\sim \pi_t(\bm{x},\bm{y}\vert \bm{z},\bm{z}')}\left\| \bm{\bm{f}_{\phi}}(\bm{x},\bm{y},t) - \bm{f}_t(\bm{x},\bm{y}\vert\bm{z},\bm{z}') \right\|_2^2
\end{aligned}
\end{equation}
\begin{theorem}
\label{thm:bf loss}
The minimizer of (\ref{eq:bf loss}) is \(\bm{f}_t(\bm{x},\bm{y})\).
\end{theorem}
An interesting point of the decomposed velocity is that one can use it to simulate the mean-field interaction dynamics. Consider \(N\) particles \(X_0^i,i=1,2\cdots,N\) independently sampled from \(\rho_0=\mu_0\) following the ODE system \(\mathrm{d}X_t^i=\Big(\bm{b}_t(X_t^i)+\frac{1}{N-1}\sum_{j\neq i}\bm{f}_t(X_t^i,X_t^j)\Big)\mathrm{d}t,i=1,2,\cdots,N
\). When \(N\to\infty\), this dynamics approximates the mean-field dynamics (\ref{eq:mean-field flow equation}). Note that the velocity decomposition is not naturally unique, hence some gauges must be given. One possible decomposition is to first train a single-particle velocity field as \(\bm{b}\) using TP-DATE, OT-CFM, or other methods, or to specify \(\bm{b}\) directly based on domain specific prior knowledge. When only a conditional form of \(\bm{b}\) is available, additional conditions are required to make the regression well defined; we discuss this case in \ref{velocity decomposition gauges}. Based on domain specific decomposition gauges, users may further exploit such a decomposition to provide more interpretable meanings for the components of the learned velocity.

\section{Experiments}
\label{sec:experiment}
As a direct application of QOT, we demonstrate on both synthetic and real datasets that TP-DATE outperforms existing methods on spatiotemporal transcriptomics slice interpolation and continuous 3D reconstruction tasks. In this section, we use \(\Phi=|\frac{\mathrm{d}}{\mathrm{d}t}\Vert\bm{q}_t\Vert|^2\) for space and \(\Phi=0\) for expression. Details of this Lagrangian choice can be found in \ref{app:Lagrangian choice}. We train TP-DATE under single-particle version (\ref{eq:particle CFM}). We use (F)GW-OT coupling for flow matching training, i.e. (F)GW-CFM, as baselines to further demonstrate the necessity of dynamic QOT against static QOT. See \ref{app:GW-CFM} for details. Ablation and scaling study of TP-DATE are left to \ref{app:ablation}, \ref{app:scaling}.

\textbf{TP-DATE preserves spatial structure during transport.} We first use two synthetic datasets to verify that TP-DATE better preserves spatial structure during the transport. Each synthetic dataset consists of three cell types, and the ground truth spatial dynamics are defined by a \(90^\circ\) counterclockwise rigid rotation, so that the relative spatial structure of the data remains unchanged throughout the entire process. In the Rotation data, the ground truth expression dynamics keep gene expression unchanged, whereas in the Rotation + expression distractor setting (Distractor), a type dependent expression change is applied to each cell type. See \ref{app:toy details} for further details.

\begin{figure}[ht]
  \centering
  \includegraphics[width=0.9\linewidth]{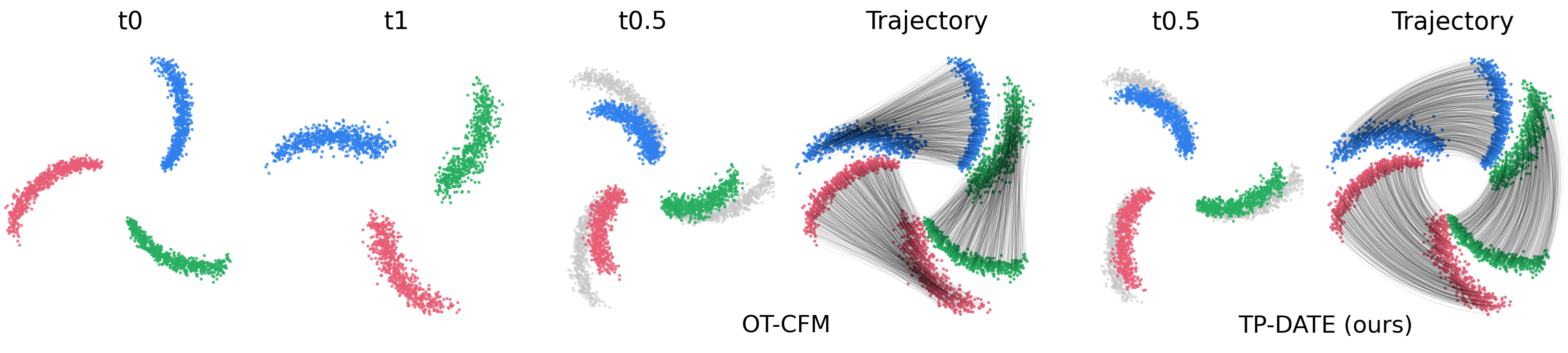}
  \caption{Trajectories on Rotation}
  \label{fig:toy-rotation-benchmarks}
\end{figure}

\begin{table}
  \centering
  \caption{Hold one out experiments on synthetic data. The best performing results are in bold. We report mean and standard deviation over 5 random seeds.}
  \label{tab:toy-rotation-results}
  \resizebox{\linewidth}{!}{%
\begin{tabular}{llccc}
\toprule
Dataset & Method & Spatial MSE ($\downarrow$) & Pair distortion ($\downarrow$) & Balanced fused $W_2$ ($\downarrow$) \\
\midrule
\multirow{7}{*}{Rotation} & OT-CFM \citep{cfm_tong} & $0.10601\!\pm\!0.01460$ & $0.33787\!\pm\!0.03021$ & $11.34260\!\pm\!0.81057$ \\
 & GW-CFM & $0.57642\!\pm\!0.07041$ & $0.62157\!\pm\!0.07289$ & $23.21205\!\pm\!1.14792$ \\
 & FGW-CFM & $0.10579\!\pm\!0.02591$ & $0.33472\!\pm\!0.04125$ & $11.30984\!\pm\!1.50067$ \\
 & stVCR \citep{peng2026stvcr}& $0.54023\!\pm\!0.01521$ & $0.17179\!\pm\!0.01119$ & $23.82269\!\pm\!0.23244$ \\
 & CytoBridge \citep{zhang2025cytobridge}& $0.09268\!\pm\!0.00318$ & $0.33237\!\pm\!0.01992$ & $10.62306\!\pm\!0.18430$ \\
 & ContextFlow \citep{rathod2026contextflowcontextawareflowmatching} & $0.23476\!\pm\!0.00584$ & $0.28606\!\pm\!0.01139$ & $17.01128\!\pm\!0.23455$ \\
 & TP-DATE (ours) & $\mathbf{0.02565\!\pm\!0.00480}$ & $\mathbf{0.17174\!\pm\!0.01483}$ & $\mathbf{5.63983\!\pm\!0.54717}$ \\
\midrule
\multirow{7}{*}{\shortstack{Rotation +\\expr. distractor}} & OT-CFM & $0.11015\!\pm\!0.00911$ & $0.32524\!\pm\!0.01367$ & $11.63383\!\pm\!0.50811$ \\
 & GW-CFM & $0.65267\!\pm\!0.06988$ & $0.64477\!\pm\!0.04815$ & $23.34253\!\pm\!0.74481$ \\
 & FGW-CFM & $0.10654\!\pm\!0.01524$ & $0.31233\!\pm\!0.01226$ & $11.43368\!\pm\!0.83093$ \\
 & stVCR & $0.52214\!\pm\!0.03327$ & $\mathbf{0.15765\!\pm\!0.01677}$ & $23.50474\!\pm\!0.58836$ \\
 & CytoBridge & $0.09085\!\pm\!0.00343$ & $0.31848\!\pm\!0.00244$ & $10.54004\!\pm\!0.19780$ \\
 & ContextFlow & $0.24025\!\pm\!0.01084$ & $0.29867\!\pm\!0.01964$ & $17.22919\!\pm\!0.38236$ \\
 & TP-DATE (ours) & $\mathbf{0.02269\!\pm\!0.00176}$ & $0.16102\!\pm\!0.00654$ & $\mathbf{5.35946\!\pm\!0.20573}$ \\
\bottomrule
\end{tabular}%
  }
\end{table}

We conducted the standard hold one out experiment commonly used in the field to evaluate the dynamics reconstruction performance of different methods. That is, we trained the models on the initial and final time points and compared the model predictions at the intermediate unseen time point with the ground truth. As shown in Figure~\ref{fig:toy-rotation-benchmarks}, on Rotation data, OT-CFM follows nearly straight paths to minimize the kinetic energy and shrinks the global structure, whereas TP-DATE recovers the more faithful curved rotation while better preserving the spatial structure.

Quantitatively, we use the Spatial MSE and pairwise distortion to assess how well the model recovers the intermediate spatial structure, and the fused Wasserstein distance to evaluate whether the model can recover the joint pattern of gene expression and spatial coordinates. See \ref{app:Evaluation metrics} for further details. As shown in Table~\ref{tab:toy-rotation-results}, TP-DATE gives the lowest spatial and fused errors, and low pair distortions on both datasets. These results show that TP-DATE can better preserve the spatial structure during transport. For these rotation datasets, to maintain the task non-trivial, we implemented stVCR \citep{peng2026stvcr} without performing rigid body transformation only on them.

\textbf{TP-DATE improves spatiotemporal dynamics reconstruction.} We further evaluate TP-DATE by hold one out experiments on two real spatiotemporal transcriptomics datasets, Mouse brain \citep{MOSTA} and ARTISTA \citep{ARTISTA}. We use the Wasserstein distances in spatial coordinate and expression space, respectively, to evaluate whether the model prediction can accurately recover the spatial structure and gene expression pattern of the intermediate slice. To jointly account for spatial and expression information, we use the fused Wasserstein distance introduced above. In addition, we introduce the spatially coupled expression MSE (SC-eMSE) to evaluate the recovery of spatial gene expression patterns in the intermediate slice. Specifically, SC-eMSE first computes a spatial optimal transport coupling between the predicted and ground truth slice using squared distance in spatial coordinate space, and then evaluates the transport cost under this coupling using squared distance in expression space. Intuitively, this metric measures how different the gene expression profiles are between spatially corresponding locations in the two slices. See \ref{app:real details} for dataset details, and \ref{app:Evaluation metrics} for metric details. 

\begin{table}
  \centering
  \caption{Hold one out experiments on spatiotemporal transcriptomics data. The best performing results are in bold. We report mean and standard deviation over 5 random seeds.}
  \label{tab:real-temporal-results}
  \resizebox{\linewidth}{!}{%
\begin{tabular}{llcccc}
\toprule
Dataset & Method & Spatial $W_2$ ($\downarrow$) & Expression $W_2$ ($\downarrow$) & SC-eMSE ($\downarrow$) & Balanced fused $W_2$ ($\downarrow$) \\
\midrule
\multirow{7}{*}{Mouse brain} & OT-CFM & $0.09190\!\pm\!0.00495$ & $6.52705\!\pm\!0.04751$ & $1.52747\!\pm\!0.00994$ & $9.91431\!\pm\!0.19724$ \\
 & GW-CFM & $0.10666\!\pm\!0.01233$ & $6.55846\!\pm\!0.04810$ & $1.55450\!\pm\!0.01966$ & $10.58390\!\pm\!0.51047$ \\
 & FGW-CFM & $0.09470\!\pm\!0.00504$ & $6.53585\!\pm\!0.05995$ & $1.53293\!\pm\!0.02363$ & $10.04189\!\pm\!0.19298$ \\
 & stVCR & $0.09053\!\pm\!0.00358$ & $7.55942\!\pm\!0.12372$ & $1.96876\!\pm\!0.03023$ & $10.76074\!\pm\!0.11473$ \\
 & CytoBridge & $0.09006\!\pm\!0.00519$ & $6.70320\!\pm\!0.05889$ & $1.61689\!\pm\!0.02489$ & $9.99210\!\pm\!0.17831$ \\
 & ContextFlow & $0.09141\!\pm\!0.00410$ & $7.25501\!\pm\!0.07919$ & $1.87498\!\pm\!0.03245$ & $10.56965\!\pm\!0.11675$ \\
 & TP-DATE (ours) & $\mathbf{0.08678\!\pm\!0.00291}$ & $\mathbf{6.40989\!\pm\!0.04644}$ & $\mathbf{1.47693\!\pm\!0.02312}$ & $\mathbf{9.59301\!\pm\!0.09255}$ \\
\midrule
\multirow{7}{*}{ARTISTA} & OT-CFM & $0.10137\!\pm\!0.00965$ & $5.94656\!\pm\!0.30203$ & $1.03487\!\pm\!0.08668$ & $8.80314\!\pm\!0.38687$ \\
 & GW-CFM & $0.10770\!\pm\!0.00734$ & $5.78681\!\pm\!0.39641$ & $0.98707\!\pm\!0.10308$ & $8.90625\!\pm\!0.24148$ \\
 & FGW-CFM & $0.10001\!\pm\!0.00739$ & $5.97909\!\pm\!0.31567$ & $1.04710\!\pm\!0.08661$ & $8.78789\!\pm\!0.16616$ \\
 & stVCR & $\mathbf{0.08582\!\pm\!0.00509}$ & $7.42768\!\pm\!0.49295$ & $1.46830\!\pm\!0.16836$ & $9.42630\!\pm\!0.39355$ \\
 & CytoBridge & $0.09833\!\pm\!0.00827$ & $7.00844\!\pm\!0.54146$ & $1.34401\!\pm\!0.17710$ & $9.49561\!\pm\!0.35947$ \\
 & ContextFlow & $0.09174\!\pm\!0.00663$ & $7.32022\!\pm\!0.49145$ & $1.44565\!\pm\!0.16502$ & $9.53524\!\pm\!0.38818$ \\
 & TP-DATE (ours) & $0.09196\!\pm\!0.00500$ & $\mathbf{5.39310\!\pm\!0.28347}$ & $\mathbf{0.87824\!\pm\!0.07587}$ & $\mathbf{8.01158\!\pm\!0.27289}$ \\
\bottomrule
\end{tabular}%
  }
\end{table}

On ARTISTA, although TP-DATE is less effective than the current state-of-the-art method stVCR \citep{peng2026stvcr} in recovering the overall spatial shape of the tissue and performs comparably to the recently proposed ContextFlow \citep{rathod2026contextflowcontextawareflowmatching}, it outperforms existing baselines on the other evaluation metrics and on the Mouse Brain dataset (Table \ref{tab:real-temporal-results}). We emphasize that spatiotemporal interpolation requires not only recovering the overall spatial morphology of the tissue, but also reconstructing the spatial patterns of gene expression. TP-DATE performs better on the latter aspect. Therefore, these results suggest that TP-DATE improves spatiotemporal dynamics reconstruction.

\textbf{TP-DATE allows continuous 3D reconstruction.} Recently, \cite{3DTumor} released a spatial transcriptomics dataset containing slices collected from the same tumor tissue at the same time point but at different depths. The slices are ordered along the depth axis rather than time. Since high-throughput volumetric spatial transcriptomics remains experimentally challenging and is not routinely available for many tissues and platforms, continuous reconstruction of 3D spatial structure from serial tissue sections remains an important computational task. If we assume that the spatial structure and gene expression patterns of the same tissue vary smoothly across adjacent depths, the QOT framework can likewise be used to interpolate unseen depths, thereby reconstructing the three dimensional tissue structure from a small number of sections sampled at different depths. TP-DATE achieves performance comparable to the best baseline in terms of spatial Wasserstein distance, while outperforming all baselines on the other evaluation metrics (Table \ref{tab:real-tumor-results}), demonstrating its superior performance on this new and important computational task.


\begin{table}
  \centering
  \caption{Hold one out experiments on 3D reconstruction. The best performing results are in bold. We report mean and standard deviation over 5 random seeds.}
  \label{tab:real-tumor-results}
  \resizebox{\linewidth}{!}{%
\begin{tabular}{lcccc}
\toprule
Method & Spatial $W_2$ ($\downarrow$) & Expression $W_2$ ($\downarrow$) & SC-eMSE ($\downarrow$) & Balanced fused $W_2$ ($\downarrow$) \\
\midrule
OT-CFM & $0.07042\!\pm\!0.00144$ & $5.23581\!\pm\!0.41570$ & $1.15154\!\pm\!0.08720$ & $7.74659\!\pm\!0.29841$ \\
GW-CFM & $0.07731\!\pm\!0.00789$ & $5.20556\!\pm\!0.27124$ & $1.15201\!\pm\!0.05978$ & $7.90406\!\pm\!0.35317$ \\
FGW-CFM & $0.07459\!\pm\!0.00802$ & $5.05869\!\pm\!0.14095$ & $1.10336\!\pm\!0.03139$ & $7.70859\!\pm\!0.27401$ \\
stVCR & $0.07059\!\pm\!0.00188$ & $5.24894\!\pm\!0.44683$ & $1.13438\!\pm\!0.09702$ & $7.63873\!\pm\!0.32859$ \\
CytoBridge & $0.08103\!\pm\!0.00191$ & $5.08726\!\pm\!0.39788$ & $1.08467\!\pm\!0.08165$ & $7.76704\!\pm\!0.22224$ \\
ContextFlow & $\mathbf{0.06768\!\pm\!0.00166}$ & $5.23844\!\pm\!0.43624$ & $1.13347\!\pm\!0.09420$ & $7.57907\!\pm\!0.35204$ \\
TP-DATE (ours) & $0.06959\!\pm\!0.00151$ & $\mathbf{4.81955\!\pm\!0.13874}$ & $\mathbf{1.04595\!\pm\!0.02881}$ & $\mathbf{7.37480\!\pm\!0.08766}$ \\
\bottomrule
\end{tabular}%
  }
\end{table}

\section{Conclusion and discussion}
We proposed TP-DATE, a dynamic QOT theory with a travelling pair flow matching framework for solving a family of dynamic QOT problems. Theoretically, we formulated dynamic QOT and included GW-OT, IGW-OT as special cases. Algorithmically, we developed a flow matching method that allows conditional paths to interact and used it to solve the dynamic QOT problem. Empirically, we quantitatively demonstrated the advantages of dynamic QOT on spatiotemporal dynamics reconstruction and continuous 3D reconstruction. We also theoretically offered a possibility to decompose the learned velocity into interpretable parts when specific domain knowledge is given.

QOT is a relatively new concept, yet some of its special cases, such as GW-OT, have already demonstrated substantial potential for applications \citep{gwot,FGWOT,moscot}. Beyond the dynamic formulation studied in this work, an important direction is to investigate whether other extensions of OT can also be introduced to the QOT setting, such as stochastic or unbalanced QOT formulations. Future work can also try to combine the velocity decomposition paradigm with specific field knowledge for more interpretable dynamics modeling. Another promising direction is to consider more general interaction terms, or even to learn these interactions directly from data.




\section*{AI Use Disclosure}
In this work, we used generative AI tools to assist with translation and language polishing, and to help verify the correctness and technical details of the mathematical arguments and algorithms. We did not use generative AI tools to develop the theoretical framework, formulate the main mathematical results, construct the proofs, design or implement the algorithms, or write the scientific content of the paper. All AI-assisted content was independently checked by the authors, and all final scientific and editorial decisions were made by the authors. We take full responsibility for the final content of this work.

\bibliography{iclr2027_conference}
\bibliographystyle{iclr2027_conference}

\newpage
\appendix
\section{Proofs}
\label{app:proofs}
Proofs of theorems and propositions.
\subsection{Proof of Theorem \ref{thm:static dynamic equivalence}}
\label{pf:static dynamic equivalence}
\textbf{Theorem \ref{thm:static dynamic equivalence}.}
\textit{If \(\mathcal{L}\) is convex w.r.t \((\dot{\bm{x}},\dot{\bm{y}})\), then \(\text{QOT}_S(\mu_0,\mu_1)=\text{QOT}_D(\mu_0,\mu_1)\ge \text{QOT}_{BB}(\mu_0,\mu_1)\).}

\begin{proof}
For convenience, we restate the three forms below. The static form is 
\begin{equation}
\label{eq:static QOT app}
\begin{aligned}
&\text{QOT}_S(\mu_0,\mu_1) =\inf_{\gamma\in\Pi(\mu_0,\mu_1)} \int_{\mathcal{X}^4} \mathcal{A}(\bm{z},\bm{z}')\gamma(\bm{z})\gamma(\bm{z}')\mathrm{d} \bm{z}\mathrm{d} \bm{z}'.\\
\end{aligned}
\end{equation}
The dynamic form is
\begin{equation}
\label{eq:dynamic QOT app}
\begin{aligned}
&\text{QOT}_{D}(\mu_0,\mu_1) = \\&\inf\int_0^1\int\int_{\mathcal{X}^2} \mathcal{L}(t,\bm{x},\bm{y},u^1_t(\bm{x},\bm{y}|\bm{z},\bm{z}'),u^2_t(\bm{x},\bm{y}|\bm{z},\bm{z}'))\pi_t(\bm{x},\bm{y}|\bm{z},\bm{z}')\gamma(\bm{z})\gamma(\bm{z}')\mathrm{d} \bm{x}\mathrm{d} \bm{y}\mathrm{d} \bm{z}\mathrm{d} \bm{z}'\mathrm{d}t.
\end{aligned}
\end{equation}
The BB-form is
\begin{equation}
\label{eq:BB QOT app}
\begin{aligned}
&\text{QOT}_{BB}(\mu_0,\mu_1) =\inf_{\pi,\bm{u}^1,\bm{u}^2} \int_0^1\int_{\mathcal{X}^2} \mathcal{L}(t,\bm{x},\bm{y},u^1_t(\bm{x},\bm{y}),u^2_t(\bm{x},\bm{y}))\pi_t(\bm{x},\bm{y})\mathrm{d} \bm{x}\mathrm{d} \bm{y}\mathrm{d}t\\
&\partial_t\pi_t(\bm{x},\bm{y})+\nabla_{\bm{x}}\cdot(\pi_t(\bm{x},\bm{y})u^1_t(\bm{x},\bm{y}))+\nabla_{\bm{y}}\cdot(\pi_t(\bm{x},\bm{y})u^2_t(\bm{x},\bm{y}))=0.\\
&\pi_0=\mu_0\otimes\mu_0,\qquad\pi_1=\mu_1\otimes\mu_1
\end{aligned}
\end{equation}

For any \((\bm{z},\bm{z}')\), the optimal path action \(\mathcal{A}(\bm{z},\bm{z}')\) is attained by the travelling pair \(\dot{\bm{x}}_t=\bm{u}_t^1(\bm{x}_t,\bm{y}_t|\bm{z},\bm{z}'),\dot{\bm{y}}_t=\bm{u}_t^2(\bm{x}_t,\bm{y}_t|\bm{z},\bm{z}')\). 
\begin{equation}
\begin{aligned}
\mathcal{A}(\bm{z},\bm{z}') &=\inf_{\bm{x}_t,\bm{y}_t} \int_0^1\mathcal{L}(t,\bm{x}_t,\bm{y}_t,\dot{\bm{x}}_t,\dot{\bm{y}}_t)\mathrm{d}t\\
&=\int_0^1\mathcal{L}(t,\bm{x}_t,\bm{y}_t,\bm{u}_t^1(\bm{x}_t,\bm{y}_t|\bm{z},\bm{z}'),\bm{u}_t^2(\bm{x}_t,\bm{y}_t|\bm{z},\bm{z}'))\mathrm{d}t\\
&=\int_0^1\int_{\mathcal{X}^2}\mathcal{L}(t,\bm{x},\bm{y},\bm{u}_t^1(\bm{x},\bm{y}|\bm{z},\bm{z}'),\bm{u}_t^2(\bm{x},\bm{y}|\bm{z},\bm{z}'))\pi_t(\bm{x},\bm{y}|\bm{z},\bm{z}')\mathrm{d}\bm{x}\mathrm{d}\bm{y}\mathrm{d}t
\end{aligned}
\end{equation}
where \(\pi_t(\bm{x},\bm{y}|\bm{z},\bm{z}')\) is the conditional Dirac probability path generated by the travelling pair velocity, which is
\begin{equation}
\begin{aligned}
&\partial_t\pi_t(\bm{x},\bm{y}|\bm{z},\bm{z}')+\nabla_{\bm{x}}\cdot(\pi_t(\bm{x},\bm{y}|\bm{z},\bm{z}')\bm{u}_t^1(\bm{x}_t,\bm{y}_t|\bm{z},\bm{z}'))+\nabla_{\bm{y}}\cdot(\pi_t(\bm{x},\bm{y}|\bm{z},\bm{z}')\bm{u}_t^2(\bm{x}_t,\bm{y}_t|\bm{z},\bm{z}'))=0\\
&\pi_0(\bm{x},\bm{y}|\bm{z},\bm{z}')=\delta_{(\bm{x}_0,\bm{y}_0)}(\bm{x},\bm{y}),\quad \pi_1(\bm{x},\bm{y}|\bm{z},\bm{z}')=\delta_{(\bm{x}_1,\bm{y}_1)}(\bm{x},\bm{y})
\end{aligned}
\end{equation}
Therefore, \(\text{QOT}_D(\mu_0,\mu_1)=\text{QOT}_S(\mu_0,\mu_1)\). 

For any admissible coupling \(\gamma\), 
\begin{equation}
\begin{aligned}
&\int_{\mathcal{X}^4} \mathcal{A}(\bm{z},\bm{z}')\gamma(\bm{z})\gamma(\bm{z}')\mathrm{d} \bm{z}\mathrm{d} \bm{z}'\\
&=\int\mathcal{L}(t,\bm{x},\bm{y},\bm{u}_t^1(\bm{x},\bm{y}|\bm{z},\bm{z}'),\bm{u}_t^2(\bm{x},\bm{y}|\bm{z},\bm{z}'))\pi_t(\bm{x},\bm{y}|\bm{z},\bm{z}')\gamma(\bm{z})\gamma(\bm{z}')\mathrm{d}\bm{x}\mathrm{d}\bm{y}\mathrm{d}t\mathrm{d} \bm{z}\mathrm{d} \bm{z}'\\
&=\int\Big(\int\mathcal{L}(t,\bm{x},\bm{y},\bm{u}_t^1(\bm{x},\bm{y}|\bm{z},\bm{z}'),\bm{u}_t^2(\bm{x},\bm{y}|\bm{z},\bm{z}'))\frac{\pi_t(\bm{x},\bm{y}|\bm{z},\bm{z}')\gamma(\bm{z})\gamma(\bm{z}')}{\pi_t(\bm{x},\bm{y})}\mathrm{d} \bm{z}\mathrm{d} \bm{z}'\Big)\pi_t(\bm{x},\bm{y})\mathrm{d}\bm{x}\mathrm{d}\bm{y}\mathrm{d}t\\
&\ge \int\mathcal{L}(t,\bm{x},\bm{y},\bm{u}_t^1(\bm{x},\bm{y}),\bm{u}_t^2(\bm{x},\bm{y}))\pi_t(\bm{x},\bm{y})\mathrm{d}\bm{x}\mathrm{d}\bm{y}\mathrm{d}t \qquad(\text{Jensen's inequality})\\
&\ge\text{QOT}_{BB}(\mu_0,\mu_1)\qquad(\text{Take infimum})
\end{aligned}
\end{equation}
By taking infimum w.r.t \(\gamma\), we have \(\text{QOT}_S(\mu_0,\mu_1)\ge\text{QOT}_{BB}(\mu_0,\mu_1)\).
\end{proof}

\subsection{Proof of Proposition \ref{prop:SUDO path}}
\label{pf:SUDO path}
\textbf{Proposition \ref{prop:SUDO path}.}
\textit{The optimal \(\bm{c}_t\) is \(\bm{c}_t=(1-t)\bm{c}_0+t\bm{c}_1\). If \(\Phi=\Phi(\Vert\bm{q}\Vert)\), \(\bm{q}_0\nparallel\bm{q}_1\), then \(\bm{q}_t\in\text{span}\{\bm{q}_0,\bm{q}_1\}\).}

\begin{proof}
Recall the Lagrangian
\begin{equation}
\mathcal{L}(t,\bm{x},\bm{y},\dot{\bm{x}},\dot{\bm{y}})=\Vert\dot{\bm{c}}_t\Vert^2+\frac{1}{4}\Vert\dot{\bm{q}}_t\Vert^2+\lambda\Phi(\bm{q},\dot{\bm{q}}).
\end{equation}
The Euler-Lagrange equation is
\begin{equation}
\left\{
\begin{aligned}
&2\ddot{\bm{c}}_t=\bm{0}\qquad&(\frac{\mathrm{d}}{\mathrm{d}t}\frac{\partial\mathcal{L}}{\partial\dot{\bm{c}}}=\frac{\partial\mathcal{L}}{\partial\bm{c}})\\
&\frac{1}{2}\ddot{\bm{q}}_t+\lambda\frac{\mathrm{d}}{\mathrm{d}t}\frac{\partial\Phi}{\partial\dot{\bm{q}}}=\lambda\frac{\partial\Phi}{\partial\bm{q}}\qquad&(\frac{\mathrm{d}}{\mathrm{d}t}\frac{\partial\mathcal{L}}{\partial\dot{\bm{q}}}=\frac{\partial\mathcal{L}}{\partial\bm{q}})
\end{aligned}
\right .
\end{equation}
The first equation yields \(\bm{c}_t=(1-t)\bm{c}_0+t\bm{c}_1\). When \(\Phi\) is a radial function of \(\bm{q}\), the second equation reduces to 
\begin{equation}
\frac{1}{2}\ddot{\bm{q}}_t=\lambda\frac{\partial\Phi}{\partial\bm{q}}=\lambda\Phi'(\Vert\bm{q}_t\Vert)\frac{\bm{q}_t}{\Vert\bm{q}_t\Vert}
\end{equation}
That means the acceleration \(\ddot{\bm{q}}_t\) is parallel to \(\bm{q}_t\), and hence the angular momentum is conserved. Therefore, \(\bm{q}_t\) is constrained on a plane. When \(\bm{q}_0\nparallel\bm{q}_1\), we have \(\bm{q}_t\in\text{span}\{\bm{q}_0,\bm{q}_1\}\).
\end{proof}

\subsection{Proof of Proposition \ref{prop:RK path}}
\label{pf:RK path}
\textbf{Proposition \ref{prop:RK path}.}
\textit{If \(\Phi=|\frac{\mathrm{d}}{\mathrm{d}t}\phi(\bm{q}_t)|^2\), \(\Phi\) is convex w.r.t \(\dot{\bm{q}}_t\). If \(\Phi=|\frac{\mathrm{d}}{\mathrm{d}t}\Vert\bm{q}_t\Vert|^2\), \(\bm{q}_t\) has analytic solution.}

\begin{proof}
By direct calculation,
\begin{equation}
\begin{aligned}
\Phi=|\frac{\mathrm{d}}{\mathrm{d}t}\phi(\bm{q}_t)|^2=|\nabla_{\bm{q}}\phi(\bm{q}_t)^{\mathrm{T}}\dot{\bm{q}}_t|^2=\dot{\bm{q}}_t^{\mathrm{T}}[\nabla_{\bm{q}}\phi(\bm{q}_t)\nabla_{\bm{q}}\phi(\bm{q}_t)^{\mathrm{T}}]\dot{\bm{q}}_t
\end{aligned}
\end{equation}
Therefore, it is convex w.r.t \(\dot{\bm{q}}_t\). When \(\Phi=|\frac{\mathrm{d}}{\mathrm{d}t}\Vert\bm{q}_t\Vert|^2\), the action of \((\bm{q},\dot{\bm{q}})\) is
\begin{equation}
\label{eq:q action}
\int_0^1\frac{1}{4}\Vert\dot{\bm{q}}_t\Vert^2+\lambda|\frac{\mathrm{d}}{\mathrm{d}t}\Vert\bm{q}_t\Vert|^2\mathrm{d}t.
\end{equation}
The corresponding Euler-Lagrange equation is
\begin{equation}
\frac{1}{2}\ddot{\bm{q}}_t+\lambda\frac{\mathrm{d}}{\mathrm{d}t}\frac{\partial\Phi}{\partial\dot{\bm{q}}}=\lambda\frac{\partial\Phi}{\partial\bm{q}}.
\end{equation}
By direct calculation, we have
\begin{equation}
\frac{\partial\Phi}{\partial\dot{\bm{q}}}=\frac{2\bm{q}_t^{\mathrm{T}}\dot{\bm{q}}_t}{\Vert\bm{q}_t\Vert^2}\bm{q}_t,\qquad\frac{\partial\Phi}{\partial\bm{q}}=\frac{2\bm{q}_t^{\mathrm{T}}\dot{\bm{q}}_t}{\Vert\bm{q}_t\Vert^2}\dot{\bm{q}}_t-\frac{2(\bm{q}_t^{\mathrm{T}}\dot{\bm{q}}_t)^2}{\Vert\bm{q}_t\Vert^4}\bm{q}_t.\
\end{equation}
Therefore, 
\begin{equation}
\frac{\mathrm{d}}{\mathrm{d}t}\frac{\partial\Phi}{\partial\dot{\bm{q}}}=2\Big(\frac{\Vert\dot{\bm{q}}_t\Vert^2+\bm{q}_t^{\mathrm{T}}\ddot{\bm{q}}_t}{\Vert\bm{q}_t\Vert^2}-2\frac{(\bm{q}_t^{\mathrm{T}}\dot{\bm{q}}_t)^2}{\Vert\bm{q}_t\Vert^4}\Big)\bm{q}_t+\frac{2\bm{q}_t^{\mathrm{T}}\dot{\bm{q}}_t}{\Vert\bm{q}_t\Vert^2}\dot{\bm{q}}_t.
\end{equation} 
The Euler-Lagrange equation reduces to
\begin{equation}
\frac{1}{2}\ddot{\bm{q}}_t+2\lambda \Big(\frac{\Vert\dot{\bm{q}}_t\Vert^2+\bm{q}_t^{\mathrm{T}}\ddot{\bm{q}}_t}{\Vert\bm{q}_t\Vert^2}-\frac{(\bm{q}_t^{\mathrm{T}}\dot{\bm{q}}_t)^2}{\Vert\bm{q}_t\Vert^4}\Big)\bm{q}_t=\bm{0}.
\end{equation}
This leads to the conservation of angular momentum, therefore \(\bm{q}_t\) is again constrained on the plane spanned by \(\bm{q}_0,\bm{q}_1\). For convenience, we define \(r_t=\Vert\bm{q}_t\Vert,\alpha=\arccos\frac{\langle\bm{q}_0,\bm{q}_1\rangle}{\Vert\bm{q}_0\Vert\Vert\bm{q}_1\Vert}\), and the polar coordinates on that plane.
\begin{equation}
\bm{q}_0=r_0(1,0),\qquad \bm{q}_1=r_1(\cos\alpha,\sin\alpha),\qquad\bm{q}_t=r_t(\cos\theta_t,\sin\theta_t)
\end{equation}
Under this representation, \(\dot{\bm{q}}_t=\dot{r}_t(\cos\theta_t,\sin\theta_t)+r_t\dot{\theta}_t(-\sin\theta_t,\cos\theta_t)\), \(|\dot{\bm{q}}_t|^2=\dot{r}_t^2+r_t^2\dot{\theta}_t^2\). The action (\ref{eq:q action}) becomes
\begin{equation}
\int_0^1\frac{1}{4}r_t^2\dot{\theta}_t^2+(\frac{1}{4}+\lambda)\dot{r}_t^2\mathrm{d}t.
\end{equation}
Write \(a=\frac{1}{4}+\lambda,k=\sqrt{\frac{1}{1+4\lambda}}\in(0,1]\), and define \(\bm{w}_t=\sqrt{a}r_t(\cos k\theta_t,\sin k\theta_t)\), we have
\begin{equation}
\Vert\dot{\bm{w}}_t\Vert^2=a(\dot{r}_t^2+k^2r_t^2\dot{\theta}_t^2)=\frac{1}{4}r_t^2\dot{\theta}_t^2+(\frac{1}{4}+\lambda)\dot{r}_t^2.
\end{equation}
Therefore, the minimization problem reduces to
\begin{equation}
\inf_{\bm{w}_t}\int_0^1\Vert\dot{\bm{w}}_t\Vert^2\mathrm{d}t\quad \text{s.t.}\quad \bm{w}_0=\sqrt{a}r_0(1,0),\bm{w}_1=\sqrt{a}r_1(\cos k\alpha,\sin k\alpha).
\end{equation}
Since \(\alpha\in(0,\pi)\) and \(k\in(0,1]\), the optimal \(\bm{w}_t\) is the displacement interpolation 
\begin{equation}
\label{eq:analytic path}
\bm{w}_t=(1-t)\bm{w}_0+t\bm{w}_1=\sqrt{a}((1-t)r_0+tr_1\cos k\alpha,tr_1\sin k\alpha).    
\end{equation}
By direct calculation, we have
\begin{equation}
\label{eq:analytic action}
\Vert\dot{\bm{w}}_t\Vert^2=a(r_0^2+r_1^2-2r_0r_1\cos k\alpha).
\end{equation}
Therefore, given the conditional variables \(\bm{z}=(\bm{x}_0,\bm{x}_1),\bm{z}'=(\bm{y}_0,\bm{y}_1)\), we can obtain the analytic solution of the travelling pair path and the path action via (\ref{eq:analytic path},\ref{eq:analytic action}). We first calculate \(\bm{c}_0=\frac{\bm{x}_0+\bm{y}_0}{2},\bm{c}_1=\frac{\bm{x}_1+\bm{y}_1}{2},\bm{q}_0=\bm{x}_0-\bm{y}_0,\bm{q}_1=\bm{x}_1-\bm{y}_1\) and \(r_0=\Vert\bm{q}_0\Vert,r_1=\Vert\bm{q}_1\Vert,\alpha=\arccos\frac{\langle\bm{q}_0,\bm{q}_1\rangle}{\Vert\bm{q}_0\Vert\Vert\bm{q}_1\Vert}\). Then we get the analytic path action
\begin{equation}
\label{eq:RK action}
\mathcal{A}(\bm{z},\bm{z}')=\Vert\bm{c}_1-\bm{c}_0\Vert^2+(\frac{1}{4}+\lambda)(r_0^2+r_1^2-2r_0r_1\cos k\alpha).
\end{equation}
Then, we calculate the analytic solution of the polar coordinates.
\begin{equation}
\left\{
\begin{aligned}
r_t&=\Vert\bm{w}_t\Vert/\sqrt{a}=\sqrt{(1-t)^2r_0^2+t^2r_1^2+2t(1-t)r_0r_1\cos k\alpha}\\
\theta_t&=\text{Arg}(\bm{w}_t)/k=\sqrt{1+4\lambda}\cdot\text{Arg}\Big((1-t)r_0+tr_1\cos k\alpha+i\cdot tr_1\sin k\alpha\Big)
\end{aligned}
\right .
\end{equation}
The analytic solution of C-Q decomposition can be further calculated.
\begin{equation}
\left\{
\begin{aligned}
\bm{c}_t&=(1-t)\bm{c}_0+t\bm{c}_1\\
\bm{q}_t&=r_t(\cos\theta_t,\sin\theta_t)
\end{aligned}
\right .
\end{equation}
The analytic solution of the travelling pair is
\begin{equation}
\left\{
\begin{aligned}
\bm{x}_t&=\bm{c}_t+\frac{1}{2}\bm{q}_t\\
\bm{y}_t&=\bm{c}_t-\frac{1}{2}\bm{q}_t
\end{aligned}
\right .,\qquad
\left\{
\begin{aligned}
\dot{\bm{x}}_t&=\bm{c}_1-\bm{c}_0+\frac{1}{2}\dot{\bm{q}}_t\\
\dot{\bm{y}}_t&=\bm{c}_1-\bm{c}_0-\frac{1}{2}\dot{\bm{q}}_t
\end{aligned}
\right .
\end{equation}
\end{proof}

\subsection{Proof of Theorem \ref{thm:GW-OT}}
\label{pf:GW-OT}
\textbf{Theorem \ref{thm:GW-OT}}
\textit{
\(\Phi=|\frac{\mathrm{d}}{\mathrm{d}t}\Vert\bm{q}_t\Vert|^2\). When \(\lambda=0\) and \(d\ge2\), the corresponding static QOT reduces to standard OT, while when \(\lambda\to+\infty\), it reduces to standard GW-OT: \(\lambda^{-1}\text{QOT}_{S}(\mu_0,\mu_1)\to\text{GW-OT}(\mu_0,\mu_1)\).}

\begin{proof}
See \ref{app:GW and FGW} for details.
\end{proof}

\subsection{Proof of Theorem \ref{thm:pair marginalization}}
\label{pf:pair marginalization}
\textbf{Theorem \ref{thm:pair marginalization}.}
\textit{Define the marginal velocity as \(\bm{u}^i_t(\bm{x},\bm{y})=\int\bm{u}^i_t(\bm{x},\bm{y}|\bm{z},\bm{z}')\frac{\pi_t(\bm{x},\bm{y}|\bm{z},\bm{z}')q(\bm{z},\bm{z}')}{\pi_t(\bm{x},\bm{y})}\mathrm{d}\bm{z}\mathrm{d}\bm{z}'\), the marginals satisfy the marginal continuity equation.}
\begin{equation}
\label{eq:pair continuity equation app}
\partial_t\pi_t+\nabla_{\bm{x}}\cdot(\pi_t\bm{u}^1_t)+\nabla_{\bm{y}}\cdot(\pi_t\bm{u}^2_t)=0
\end{equation}
\textit{If \(q=\gamma\otimes\gamma\) for some \(\gamma\in\Pi(\mu_0,\mu_1)\), the marginal velocity transports \(\mu_0\otimes\mu_0\) to \(\mu_1\otimes\mu_1\).}

\begin{proof}
Consider the probability path \(\pi_t(\bm{x},\bm{y})\) on the two-particle augmented space \(\mathcal{X}^2\), the corresponding velocity is \(\bm{u}_t(\bm{x},\bm{y})=\begin{pmatrix}
\bm{u}^1_t(\bm{x},\bm{y})\\
\bm{u}^2_t(\bm{x},\bm{y})
\end{pmatrix}\). Applying the marginalization theorem of standard conditional flow matching \citep{cfm_lipman,cfm_tong}, we have
\begin{equation}
\partial_t\pi_t+\nabla_{(\bm{x},\bm{y})}\cdot(\pi_t\bm{u}_t)=0
\end{equation}
By definition, \(\nabla_{(\bm{x},\bm{y})}\cdot(\pi_t\bm{u}_t)=\nabla_{\bm{x}}\cdot(\pi_t\bm{u}^1_t)+\nabla_{\bm{y}}\cdot(\pi_t\bm{u}^2_t)\). Therefore, identity (\ref{eq:pair continuity equation app}) holds. If we further have \(q=\gamma\otimes\gamma\) for some \(\gamma\in\Pi(\mu_0,\mu_1)\), then by definition, we have
\begin{equation}
\begin{aligned}
\pi_0(\bm{x},\bm{y})&=\int\pi_0(\bm{x},\bm{y}|\bm{z},\bm{z}')q(\bm{z},\bm{z}')\mathrm{d}\bm{z}\mathrm{d}\bm{z}'\\
&=\int\delta_{(\bm{x}_0,\bm{y}_0)}(\bm{x},\bm{y})\gamma(\bm{x}_0,\bm{x}_1)\gamma(\bm{y}_0,\bm{y}_1)\mathrm{d}\bm{x}_0\mathrm{d}\bm{x}_1\mathrm{d}\bm{y}_0\mathrm{d}\bm{y}_1\\
&=\int\delta_{\bm{x}_0}(\bm{x})\gamma(\bm{x}_0,\bm{x}_1)\mathrm{d}\bm{x}_0\mathrm{d}\bm{x}_1\int\delta_{\bm{y}_0}(\bm{y})\gamma(\bm{y}_0,\bm{y}_1)\mathrm{d}\bm{y}_0\mathrm{d}\bm{y}_1\\
&=\mu_0(\bm{x})\mu_0(\bm{y})
\end{aligned}
\end{equation}
where \(\bm{z}=(\bm{x}_0,\bm{x_1}),\bm{z}'=(\bm{y}_0,\bm{y}_1)\). Similarly, \(\pi_1(\bm{x},\bm{y})=\mu_1(\bm{x})\mu_1(\bm{y})\). Therefore, the marginal velocity transports \(\mu_0\otimes\mu_0\) to \(\mu_1\otimes\mu_1\).
\end{proof}

\subsection{Proof of Theorem \ref{thm:particle marginalization}}
\label{pf:particle marginalization}
\textbf{Theorem \ref{thm:particle marginalization}.}
\textit{The single particle marginals satisfy the continuity equation}
\begin{equation}
\partial_t\rho_t+\nabla_{\bm{x}}\cdot(\rho_t\bm{v}_t)=0
\end{equation}

\begin{proof}
Recall the definitions
\begin{equation}
\rho_t(\bm{x})=\int_{\mathcal{X}}\pi_t(\bm{x},\bm{y})\mathrm{d}\bm{y},\quad\bm{v}_t(\bm{x})=\int_{\mathcal{X}}\bm{u}^1_t(\bm{x},\bm{y})\pi_t(\bm{y}|\bm{x})\mathrm{d}\bm{y}=\int_{\mathcal{X}}\bm{u}^1_t(\bm{x},\bm{y})\frac{\pi_t(\bm{x},\bm{y})}{\rho_t(\bm{x})}\mathrm{d}\bm{y}.
\end{equation}
By direct calculation, we have
\begin{equation}
\begin{aligned}
\partial_t\rho_t(\bm{x})&=\partial_t\int_{\mathcal{X}}\pi_t(\bm{x},\bm{y})\mathrm{d}\bm{y}=\int_{\mathcal{X}}\partial_t\pi_t(\bm{x},\bm{y})\mathrm{d}\bm{y}\\
&=-\int_{\mathcal{X}}\nabla_{\bm{x}}\cdot(\pi_t(\bm{x},\bm{y})\bm{u}^1_t(\bm{x},\bm{y}))+\nabla_{\bm{y}}\cdot(\pi_t(\bm{x},\bm{y})\bm{u}^2_t(\bm{x},\bm{y}))\mathrm{d}\bm{y}\\
&=-\nabla_{\bm{x}}\cdot\int_{\mathcal{X}}\pi_t(\bm{x},\bm{y})\bm{u}^1_t(\bm{x},\bm{y})\mathrm{d}\bm{y}-\rho_t(\bm{x})\int_{\mathcal{X}}\nabla_{\bm{y}}\cdot\Big(\pi_t(\bm{y}|\bm{x})\bm{u}^2_t(\bm{x},\bm{y})\Big)\mathrm{d}\bm{y}\\
&=-\nabla_{\bm{x}}\cdot\int_{\mathcal{X}}\pi_t(\bm{x},\bm{y})\bm{u}^1_t(\bm{x},\bm{y})\mathrm{d}\bm{y}\\
&=-\nabla_{\bm{x}}\cdot\Big(\rho_t(\bm{x})\int_{\mathcal{X}}\frac{\pi_t(\bm{x},\bm{y})}{\rho_t(\bm{x})}\bm{u}^1_t(\bm{x},\bm{y})\mathrm{d}\bm{y}\Big)\\
&=-\nabla_{\bm{x}}\cdot(\rho_t(\bm{x})\bm{v}_t(\bm{x}))
\end{aligned}
\end{equation}
\end{proof}
\textbf{Remark.} If we further have \(\int q(\bm{z},\bm{z}')\mathrm{d}\bm{z}'=\gamma(\bm{z})\) for some \(\gamma\in\Pi(\mu_0,\mu_1)\), then it is easy to check \(\rho_0=\mu_0,\rho_1=\mu_1\). In this case, \(\bm{v}_t\) transports \(\mu_0\) to \(\mu_1\). We point out that this condition is slightly weaker than \(q=\gamma\otimes\gamma\) for some \(\gamma\in\Pi(\mu_0,\mu_1)\), which we used in theorem \ref{thm:pair marginalization}.

\subsection{Proof of Proposition \ref{prop:error bound}}
\label{pf:error bound}
\textbf{Proposition \ref{prop:error bound}.}
\textit{Let the true probability flow be \(\rho_t\) and the mean-field approximation be \(\hat{\rho}_t\). Assume \(\bm{b}_t,\bm{f}_t\) are both Lipschitz, and \(M=\underset{0<t\le1}{\max}\underset{\bm{x}}{\max}\mathcal{W}_1(\rho_t,\pi_t(\cdot|\bm{x}))\), then \(\forall 0<t\le1,\exists C_t>0\) such that \(\mathcal{W}_1(\rho_t,\hat{\rho}_t)\le C_tM\).}

\begin{proof}
Let \(L_b,L_f\) be the Lipschitz constant of \(\bm{b}_t,\bm{f}_t\), and \(L\ge L_b,L_f\). The true dynamics is
\begin{equation}
\partial_t\rho_t(\bm{x})+\nabla_{\bm{x}}\cdot[\rho_t(\bm{x})\big(\bm{b}_t(\bm{x})+\int_{\mathcal{X}}\bm{f}_t(\bm{x},\bm{y})\pi_t(\bm{y}|\bm{x})\mathrm{d}\bm{y}\big)]=0.
\end{equation}
The mean-field approximation is
\begin{equation}
\partial_t\hat{\rho}_t(\bm{x})+\nabla_{\bm{x}}\cdot[\hat{\rho}_t(\bm{x})\big(\bm{b}_t(\bm{x})+\int_{\mathcal{X}}\bm{f}_t(\bm{x},\bm{y})\hat{\rho}_t(\bm{y})\mathrm{d}\bm{y}\big)]=0.
\end{equation}
Let \(R_t(\bm{x})=\int_{\mathcal{X}}\bm{f}_t(\bm{x},\bm{y})[\pi_t(\bm{y}|\bm{x})-\hat{\rho}_t(\bm{y})]\mathrm{d}\bm{y}\) represent the difference between the two velocity fields. By the Lipschitz condition, we have
\begin{equation}
\begin{aligned}
\Vert R_t(\bm{x})\Vert&=\Vert\int_{\mathcal{X}}\bm{f}_t(\bm{x},\bm{y})[\pi_t(\bm{y}|\bm{x})-\hat{\rho}_t(\bm{y})]\mathrm{d}\bm{y}\Vert\\
&\le\Vert\int_{\mathcal{X}}\bm{f}_t(\bm{x},\bm{y})[\pi_t(\bm{y}|\bm{x})-\rho_t(\bm{y})]\mathrm{d}\bm{y}\Vert+\Vert\int_{\mathcal{X}}\bm{f}_t(\bm{x},\bm{y})[\rho_t(\bm{y})-\hat{\rho}_t(\bm{y})]\mathrm{d}\bm{y}\Vert\\
&\le L\Big(\mathcal{W}_1(\pi_t(\cdot|\bm{x}),\rho_t)+\mathcal{W}_1(\rho_t,\hat{\rho}_t)\Big)\\
&\le LM+L\mathcal{W}_1(\rho_t,\hat{\rho}_t).
\end{aligned}
\end{equation}
Since \(\hat{\rho}_0=\rho_0\), the difference between \(\rho_t,\hat{\rho}_t\) can be fully described by the difference between the velocity fields. For any starting point \(X_0=Y_0=\bm{x}\), the true dynamics is 
\begin{equation}
\frac{\mathrm{d}}{\mathrm{d}t}X_t=\bm{b}_t(X_t)+\int_{\mathcal{X}}\bm{f}_t(X_t,\bm{y})\pi_t(\bm{y}|X_t)\mathrm{d}\bm{y}\triangleq A(X_t).
\end{equation}
The mean-field dynamics is
\begin{equation}
\frac{\mathrm{d}}{\mathrm{d}t}Y_t=\bm{b}_t(Y_t)+\int_{\mathcal{X}}\bm{f}_t(Y_t,\bm{y})\hat{\rho}_t(\bm{y})\mathrm{d}\bm{y}\triangleq B(Y_t).
\end{equation}
Let \(Z_t=\Vert X_t-Y_t\Vert\), we have
\begin{equation}
\begin{aligned}
\frac{\mathrm{d}}{\mathrm{d}t}(X_t-Y_t)&=A(X_t)-B(Y_t)=A(X_t)-B(X_t)+B(X_t)-B(Y_t)\\
&=R_t(X_t)+[\bm{b}_t(X_t)-\bm{b}_t(Y_t)+\int_{\mathcal{X}}\Big(\bm{f}_t(X_t,\bm{y})-\bm{f}_t(Y_t,\bm{y})\Big)\hat{\rho}_t(\bm{y})\mathrm{d}\bm{y}]
\end{aligned}
\end{equation}
Therefore,
\begin{equation}
\frac{\mathrm{d}}{\mathrm{d}t}Z_t\le\Vert\frac{\mathrm{d}}{\mathrm{d}t}(X_t-Y_t)\Vert\le LM+L\mathcal{W}_1(\rho_t,\hat{\rho}_t)+2LZ_t.
\end{equation}
By coupling \(X_t,Y_t\) generated by the same starting point \(\bm{x}\) together, we obtain an inequality of  \(\mathcal{W}_1(\rho_t,\hat{\rho}_t)\).
\begin{equation}
\begin{aligned}
\mathcal{W}_1(\rho_t,\hat{\rho}_t)&=\inf_{\gamma}\int_{\mathcal{X}^2}\Vert\bm{x}-\bm{y}\Vert\gamma(\bm{x},\bm{y})\mathrm{d}\bm{x}\mathrm{d}\bm{y}\\
&\le\int_{\mathcal{X}}Z_t(\bm{x})\rho_0(\bm{x})\mathrm{d}\bm{x}.
\end{aligned}
\end{equation}
Therefore, it is sufficient to estimate the upper bound of \(Z_t\). By Gronwall's inequality, we have
\begin{equation}
\begin{aligned}
Z_t&\le e^{2Lt}L\int_0^te^{-2Ls}(M+\mathcal{W}_1(\rho_s,\hat{\rho}_s))\mathrm{d}s\\
&=\frac{e^{2Lt}-1}{2}M+e^{2Lt}L\int_0^te^{-2Ls}\mathcal{W}_1(\rho_s,\hat{\rho}_s)\mathrm{d}s.
\end{aligned}
\end{equation}
Therefore, we get another inequality of \(\mathcal{W}_1(\rho_t,\hat{\rho}_t)\).
\begin{equation}
\begin{aligned}
\mathcal{W}_1(\rho_t,\hat{\rho}_t)&\le\int_{\mathcal{X}}\Big(\frac{e^{2Lt}-1}{2}M+e^{2Lt}L\int_0^te^{-2Ls}\mathcal{W}_1(\rho_s,\hat{\rho}_s)\mathrm{d}s\Big)\rho_0(\bm{x})\mathrm{d}\bm{x}\\
&=\frac{e^{2Lt}-1}{2}M+e^{2Lt}L\int_0^te^{-2Ls}\mathcal{W}_1(\rho_s,\hat{\rho}_s)\mathrm{d}s.
\end{aligned}
\end{equation}
Or equivalently, we have
\begin{equation}
e^{-2Lt}\mathcal{W}_1(\rho_t,\hat{\rho}_t)\le\frac{1-e^{-2Lt}}{2}M+L\int_0^te^{-2Ls}\mathcal{W}_1(\rho_s,\hat{\rho}_s)\mathrm{d}s.   
\end{equation}
Let \(F(t)=\int_0^te^{-2Ls}\mathcal{W}_1(\rho_s,\hat{\rho}_s)\mathrm{d}s\), the inequality can be reformulated as 
\begin{equation}
\label{eq:lemma}
F'(t)\le\frac{1-e^{-2Lt}}{2}M+LF(t).
\end{equation}
By Gronwall's inequality again, we have
\begin{equation}
F(t)\le\frac{M}{6L}(2e^{Lt}-3+e^{-2Lt}).
\end{equation}
Plug in (\ref{eq:lemma}), we have
\begin{equation}
\mathcal{W}_1(\rho_t,\hat{\rho}_t)\le\frac{e^{3Lt}-1}{3}M.
\end{equation}

\end{proof}

\subsection{Proof of Theorem \ref{thm:solve dynamic QOT}}
\label{pf:solve dynamic QOT}
\textbf{Theorem \ref{thm:solve dynamic QOT}.}
\textit{The minimizer is \(\bm{\bm{u}^i_{\theta}}(\bm{x},\bm{y},t)=\bm{u}^i_t(\bm{x},\bm{y})\). When \(q=\gamma\otimes\gamma\) where \(\gamma\) is the optimal QOT coupling of (\ref{eq:static QOT}), and the travelling pair used is the minimizer of (\ref{eq:travelling pair}), then the learned velocity pair generates the probability flow of the dynamic QOT problem (\ref{eq:dynamic QOT}).}

\begin{proof}
Recall the conditional loss
\begin{equation}
\label{eq:pair loss app}
\mathcal{L}_{\text{pair}}(\bm{\theta})=
\mathbb{E}_{t\sim\mathcal{U}[0,1], (\bm{z},\bm{z}')\sim q(\bm{z},\bm{z}'), (\bm{x},\bm{y})\sim \pi_t(\bm{x},\bm{y}\vert \bm{z},\bm{z}')}\sum_{i=1}^2\left\| \bm{\bm{u}^i_{\theta}}(\bm{x},\bm{y},t) - \bm{u}^i_t(\bm{x},\bm{y}\vert\bm{z},\bm{z}') \right\|_2^2
\end{equation}
Similarly to what we have done in the proof of theorem \ref{thm:pair marginalization}, on the two-particle augmented space, the standard equivalence between the marginal flow matching loss and the conditional flow matching loss \citep{cfm_lipman,cfm_tong} tells us that minimizing the conditional loss (\ref{eq:pair loss app}) is equivalent to minimizing the marginal loss
\begin{equation}
\label{eq:marginal pair loss}
\mathcal{L}_{\text{pair-marginal}}(\bm{\theta})=
\mathbb{E}_{t\sim\mathcal{U}[0,1], (\bm{x},\bm{y})\sim \pi_t(\bm{x},\bm{y})}\sum_{i=1}^2\left\| \bm{\bm{u}^i_{\theta}}(\bm{x},\bm{y},t) - \bm{u}^i_t(\bm{x},\bm{y}) \right\|_2^2
\end{equation}
Therefore, the minimizer is \(\bm{\bm{u}^i_{\theta}}(\bm{x},\bm{y},t)=\bm{u}^i_t(\bm{x},\bm{y})\). By the pair marginalization theorem \ref{thm:pair marginalization}, we know that the minimizer generates the probability flow of the dynamic QOT.
\end{proof}

\textbf{Remark.} Although the learned velocity induces the same probability flow as the dynamic QOT construction, the mechanism by which it generates this flow is not exactly the same as in the definition of dynamic QOT. In dynamic QOT, the coupling first determines where each particle is transported, and the travelling-pair path then specifies how each pairwise transport is carried out. The resulting process is therefore generally non-Markovian, since its evolution depends on the prescribed endpoint. In realistic dynamics, however, future states are typically not known in advance. Instead, the evolution is determined by the current state, possibly together with its history. A natural approximation is therefore to replace the original process by a Markov process. The velocity pair learned by flow matching defines a pair of time-dependent velocity fields on the two-particle space, so particles evolve according to their current positions and time, making the resulting dynamics Markovian. The two processes share the same two-particle marginal distribution \(\pi_t(\bm{x},\bm{y})\) at every time \(t\). In this sense, the learned velocity pair can be viewed as providing a Markovian projection of the original endpoint-conditioned process.

\subsection{Proof of Theorem \ref{thm:particle loss}}
\label{pf:particle loss}
\textbf{Theorem \ref{thm:particle loss}.}
\textit{\(\nabla_{\bm{\theta}}\mathcal{L}_M(\bm{\theta})=\nabla_{\bm{\theta}}\mathcal{L}_C(\bm{\theta})\).}

\begin{proof}
We first recall the two losses.
\begin{equation}
\begin{aligned}
&\mathcal{L}_{M}(\bm{\theta)}=\mathbb{E}_{t\sim\mathcal{U}[0,1],\bm{x}\sim\rho_t(\bm{x})}\Vert\bm{v}_{\bm{\theta}}(\bm{x},t)-\bm{v}_t(\bm{x})\Vert^2\\
&\mathcal{L}_{C}(\bm{\theta)}=\mathbb{E}_{t\sim\mathcal{U}[0,1],(\bm{z},\bm{z}')\sim q(\bm{z},\bm{z}'),(\bm{x},\bm{y})\sim\pi_t(\bm{x},\bm{y}\vert\bm{z},\bm{z}')}\Vert\bm{v}_{\bm{\theta}}(\bm{x},t)-\bm{u}^1_t(\bm{x},\bm{y}\vert\bm{z},\bm{z}')\Vert^2
\end{aligned}
\end{equation}
By direct calculation, we have
\begin{equation}
\begin{aligned}
\mathcal{L}_{C}(\bm{\theta)}&=\mathbb{E}_{t\sim\mathcal{U}[0,1],(\bm{z},\bm{z}')\sim q(\bm{z},\bm{z}'),(\bm{x},\bm{y})\sim\pi_t(\bm{x},\bm{y}\vert\bm{z},\bm{z}')}\Vert\bm{v}_{\bm{\theta}}(\bm{x},t)-\bm{u}^1_t(\bm{x},\bm{y}\vert\bm{z},\bm{z}')\Vert^2\\
&=\int_0^1\int\int_{\mathcal{X}^2}\Vert\bm{v}_{\bm{\theta}}(\bm{x},t)-\bm{u}^1_t(\bm{x},\bm{y}\vert\bm{z},\bm{z}')\Vert^2\pi_t(\bm{x},\bm{y}|\bm{z},\bm{z}')q(\bm{z},\bm{z}')\mathrm{d}\bm{x}\mathrm{d}\bm{y}
\mathrm{d}\bm{z}\mathrm{d}\bm{z}'\mathrm{d}t\\
&=\int_0^1\int\int_{\mathcal{X}^2}\Vert\bm{v}_{\bm{\theta}}(\bm{x},t)\Vert^2\pi_t(\bm{x},\bm{y}|\bm{z},\bm{z}')q(\bm{z},\bm{z}')\mathrm{d}\bm{x}\mathrm{d}\bm{y}\mathrm{d}\bm{z}\mathrm{d}\bm{z}'\mathrm{d}t\\
&\quad+\int_0^1\int\int_{\mathcal{X}^2}\Vert\bm{u}^1_t(\bm{x},\bm{y}\vert\bm{z},\bm{z}')\Vert^2\pi_t(\bm{x},\bm{y}|\bm{z},\bm{z}')q(\bm{z},\bm{z}')\mathrm{d}\bm{x}\mathrm{d}\bm{y}\mathrm{d}\bm{z}\mathrm{d}\bm{z}'\mathrm{d}t\\
&\quad-2\int_0^1\int\int_{\mathcal{X}^2}\langle\bm{v}_{\bm{\theta}}(\bm{x},t),\bm{u}^1_t(\bm{x},\bm{y}\vert\bm{z},\bm{z}')\rangle\pi_t(\bm{x},\bm{y}|\bm{z},\bm{z}')q(\bm{z},\bm{z}')\mathrm{d}\bm{x}\mathrm{d}\bm{y}\mathrm{d}\bm{z}\mathrm{d}\bm{z}'\mathrm{d}t\\
&=\int_0^1\int_{\mathcal{X}^2}\Vert\bm{v}_{\bm{\theta}}(\bm{x},t)\Vert^2\pi_t(\bm{x},\bm{y})\mathrm{d}\bm{x}\mathrm{d}\bm{y}\mathrm{d}t+C\\
&\quad-2\int_0^1\int_{\mathcal{X}^2}\langle\bm{v}_{\bm{\theta}}(\bm{x},t),\int\bm{u}^1_t(\bm{x},\bm{y}\vert\bm{z},\bm{z}')\frac{\pi_t(\bm{x},\bm{y}|\bm{z},\bm{z}')q(\bm{z},\bm{z}')}{\pi_t(\bm{x},\bm{y})}\mathrm{d}\bm{z}\mathrm{d}\bm{z}'\rangle\pi_t(\bm{x},\bm{y})\mathrm{d}\bm{x}\mathrm{d}\bm{y}\mathrm{d}t\\
&=\int_0^1\int_{\mathcal{X}^2}\Vert\bm{v}_{\bm{\theta}}(\bm{x},t)\Vert^2\pi_t(\bm{x},\bm{y})\mathrm{d}\bm{x}\mathrm{d}\bm{y}\mathrm{d}t+C\\
&\quad-2\int_0^1\int_{\mathcal{X}^2}\langle\bm{v}_{\bm{\theta}}(\bm{x},t),\bm{u}^1_t(\bm{x},\bm{y})\rangle\pi_t(\bm{x},\bm{y})\mathrm{d}\bm{x}\mathrm{d}\bm{y}\mathrm{d}t\\
&=\int_0^1\int_{\mathcal{X}^2}\Vert\bm{v}_{\bm{\theta}}(\bm{x},t)\Vert^2\pi_t(\bm{x},\bm{y})\mathrm{d}\bm{x}\mathrm{d}\bm{y}\mathrm{d}t+C\\
&\quad-2\int_0^1\int_{\mathcal{X}}\langle\bm{v}_{\bm{\theta}}(\bm{x},t),\int_{\mathcal{X}}\bm{u}^1_t(\bm{x},\bm{y})\frac{\pi_t(\bm{x},\bm{y})}{\rho_t(\bm{x})}\mathrm{d}\bm{y}\rangle\rho_t(\bm{x})\mathrm{d}\bm{x}\mathrm{d}t\\
&=\int_0^1\int_{\mathcal{X}^2}\Vert\bm{v}_{\bm{\theta}}(\bm{x},t)\Vert^2\pi_t(\bm{x},\bm{y})\mathrm{d}\bm{x}\mathrm{d}\bm{y}\mathrm{d}t+C-2\int_0^1\int_{\mathcal{X}}\langle\bm{v}_{\bm{\theta}}(\bm{x},t),\bm{v}_t(\bm{x})\rangle\rho_t(\bm{x})\mathrm{d}\bm{x}\mathrm{d}t\\
&=\int_0^1\int_{\mathcal{X}}\Vert\bm{v}_{\bm{\theta}}(\bm{x},t)\Vert^2\rho_t(\bm{x})\mathrm{d}\bm{x}\mathrm{d}t-2\int_0^1\int_{\mathcal{X}}\langle\bm{v}_{\bm{\theta}}(\bm{x},t),\bm{v}_t(\bm{x})\rangle\rho_t(\bm{x})\mathrm{d}\bm{x}\mathrm{d}t\\
&\quad+\int_0^1\int_{\mathcal{X}}\Vert\bm{v}_t(\bm{x})\Vert^2\rho_t(\bm{x})\mathrm{d}\bm{x}\mathrm{d}t+C\\
&=\int_0^1\int_{\mathcal{X}}\Vert\bm{v}_{\bm{\theta}}(\bm{x},t)-\bm{v}_t(\bm{x})\Vert^2\rho_t(\bm{x})\mathrm{d}\bm{x}\mathrm{d}t+C\\
&=\mathcal{L}_M(\bm{\theta})+C
\end{aligned}
\end{equation}
where \(C\) is some constant independent of \(\bm{\theta}\). Therefore, \(\nabla_{\bm{\theta}}\mathcal{L}_M(\bm{\theta})=\nabla_{\bm{\theta}}\mathcal{L}_C(\bm{\theta})\).
\end{proof}

\subsection{Proof of Theorem \ref{thm:bf loss}}
\label{pf:bf loss}
\textbf{Theorem \ref{thm:bf loss}.}
\textit{The minimizer of (\ref{eq:bf loss}) is \(\bm{f}_t(\bm{x},\bm{y})\).}

\begin{proof}
We first recall the loss.
\begin{equation}
\label{eq:bf loss app}
\begin{aligned}
\mathcal{L}_f(\bm{\phi})&=
\mathbb{E}_{t\sim\mathcal{U}[0,1], (\bm{z},\bm{z}')\sim q(\bm{z},\bm{z}'), (\bm{x},\bm{y})\sim \pi_t(\bm{x},\bm{y}\vert \bm{z},\bm{z}')}\left\| \bm{\bm{f}_{\phi}}(\bm{x},\bm{y},t) - \bm{f}_t(\bm{x},\bm{y}\vert\bm{z},\bm{z}') \right\|_2^2
\end{aligned}
\end{equation}
Similarly to theorem \ref{thm:solve dynamic QOT}, we have
\begin{equation}
\begin{aligned}
\mathcal{L}_f(\bm{\phi})&=
\mathbb{E}_{t\sim\mathcal{U}[0,1], (\bm{x},\bm{y})\sim \pi_t(\bm{x},\bm{y})}\left\| \bm{\bm{f}_{\phi}}(\bm{x},\bm{y},t) - \bm{f}_t(\bm{x},\bm{y}) \right\|_2^2+C
\end{aligned}
\end{equation}
where \(C\) is a constant independent of \(\bm{\phi}\). Therefore, the minimizer is \(\bm{f}_t(\bm{x},\bm{y})\).
\end{proof}

\section{Implementation details}
\label{app:implementation details}
\subsection{Computation resources}
\label{app:Computation resources}
All experiments were performed on a personal computer with NVIDIA 5070 Ti GPU.

\subsection{Network architecture}
\label{app:Network architecture}
The velocity net \(\bm{v}_{\bm{\theta}}\) was parameterized as a MLP with 128 hidden channels, 5 hidden layers and SiLU activations \cite{SiLU}. It receives the scalar time \(t\) concatenated with expression and spatial state \((\bm{x},\bm{s})\), and outputs the velocity on both expression and spatial space. Across all experiments, the networks were optimized by Adam \citep{adam} with learning rate \(2\times10^{-3}\), weight decay \(10^{-5}\), and gradient norm clipping at 5.

\subsection{Lagrangian choice for experiments}
\label{app:Lagrangian choice}
As introduced in \ref{sec:tractable lagrangian}, we used a modality separated Lagrangian for numerical experiments. On both synthetic data and real data, we only added interaction term on spatial coordinates, and used standard kinetic energy on expression spaces. That is, let two cells represented by \((\bm{x},\bm{s}),(\bm{y},\bm{h})\) where \(\bm{x},\bm{y}\) are expressions, \(\bm{s},\bm{h}\) are spatial coordinates, we choose the Lagrangian as

\begin{equation}
\begin{aligned}
\mathcal{L}=&\eta_{\text{gene}}\left(\frac{1}{2}\Vert\dot{\bm{x}}\Vert^2+\frac{1}{2}\Vert\dot{\bm{y}}\Vert^2\right)+\lambda_{\text{gene}}|\frac{\mathrm{d}}{\mathrm{d}t}\Vert\bm{x}-\bm{y}\Vert|^2\\+&\eta_{\text{spatial}}\left(\frac{1}{2}\Vert\dot{\bm{s}}\Vert^2+\frac{1}{2}\Vert\dot{\bm{h}}\Vert^2\right)+\lambda_{\text{spatial}}|\frac{\mathrm{d}}{\mathrm{d}t}\Vert\bm{s}-\bm{h}\Vert|^2
\end{aligned}
\end{equation}

Since the corresponding QOT problem remains the same under global scaling, we set \(\eta_{\text{gene}}=1\). The interaction term is only added to spatial Lagrangian, hence \(\lambda_{\text{gene}}=0\). The Lagrangian reduces to

\begin{equation}
\begin{aligned}
\mathcal{L}=&\frac{1}{2}\Vert\dot{\bm{x}}\Vert^2+\frac{1}{2}\Vert\dot{\bm{y}}\Vert^2+\eta_{\text{spatial}}\left(\frac{1}{2}\Vert\dot{\bm{s}}\Vert^2+\frac{1}{2}\Vert\dot{\bm{h}}\Vert^2\right)+\lambda_{\text{spatial}}|\frac{\mathrm{d}}{\mathrm{d}t}\Vert\bm{s}-\bm{h}\Vert|^2
\end{aligned}
\end{equation}

There are only two hyperparameters \(\eta_{\text{spatial}},\lambda_{\text{spatial}}\). We conducted an ablation study in \ref{app:ablation}.

\subsection{Synthetic rotation data details}
\label{app:toy details}
In \ref{sec:experiment}, we evaluated TP-DATE on two synthetic rotation datasets. Here, we introduce the generation details. Each synthetic rotation dataset contains 2,000 paired cells from three cell types, with two spatial coordinates and 50 expression features. The target time point is a $90^\circ$ rotation of the source time point. Therefore, the true midpoint is the corresponding $45^\circ$ rotation. 



The source spatial coordinates were sampled from three noisy curved arms. For a cell of type
\(c\), we sampled
\[
r\sim \mathcal{U}(0.18,1),\qquad
\epsilon_\theta\sim\mathcal{N}(0,0.055^2),\qquad
\epsilon_r,\epsilon_q\sim\mathcal{N}(0,0.035^2),
\]
and set
\[
\theta=\theta_c+1.05r+\epsilon_\theta,\qquad
\bm{s}_0=(0.35+2.25r+\epsilon_r)
\begin{bmatrix}\cos\theta\\ \sin\theta\end{bmatrix}
+\epsilon_q
\begin{bmatrix}-\sin\theta\\ \cos\theta\end{bmatrix},
\]
where \((\theta_0,\theta_1,\theta_2)=(15^\circ,140^\circ,260^\circ)\).

After concatenating the three cell types, we randomly permuted the cells, subtracted the global spatial centroid, and divided all coordinates by \(\sqrt{n^{-1}\sum_i\|\bm{s}_{0,i}\|_2^2}\). The spatial trajectory was a rigid counterclockwise rotation,
\[
\bm{s}_{t,i}=R(90^\circ t)\bm{s}_{0,i},\qquad t\in[0,1],
\]
where \(R(\cdot)\) is the two-dimensional rotation matrix. Consequently, the observed source, held out midpoint, and target correspond to rotations of \(0^\circ\), \(45^\circ\), and \(90^\circ\), respectively, and all pairwise spatial distances are exactly preserved.

Source expression was initialized independently sampled from \(\mathcal{N}(0,0.18^2)\). To generate cell-type markers, we added \(2.2\) to genes \(1\!-\!12\), \(13\!-\!24\), and \(25\!-\!36\) for cell types 1, 2, and 3, respectively. Genes \(37\!-\!50\) additionally received \(0.75\) times the following 14 dimensional spatial signals, evaluated using the normalized source coordinates \((x_i,y_i)\), radius \(\rho_i=\sqrt{x_i^2+y_i^2}\), and angle \(\phi_i=\text{atan2}\frac{y_i}{x_i}\):
\[
\begin{split}
(&x_i,y_i,\rho_i,\sin\phi_i,\cos\phi_i,x_iy_i,x_i^2-y_i^2,
\sin 2\phi_i,\cos 2\phi_i,\\
&\rho_i\sin(\phi_i+0.7),\rho_i\cos(\phi_i-0.4),
\exp(-0.7\rho_i^2),\tanh(1.5x_i),\tanh(1.5y_i)).
\end{split}
\]
Each expression feature was then centered and divided by its population standard deviation. For Rotation, expression remained unchanged over time. For Distractor, we constructed a drift matrix \(\bm{\Delta}\) by adding \(0.45\) to genes \(37\!-\!41\), \(42\!-\!46\), and \(47\!-\!50\) for cell types 1, 2, and 3, respectively, followed by independent noise sampled from \(\mathcal{N}(0,0.16^2)\) added for each cell, all 50 features. Denoted the expression matrix at time \(t\) as \(\bm{e}_t\), we set
\[
\bm{e}_{0.5}=\bm{e}_0+\tfrac{1}{2}\bm{\Delta},\qquad
\bm{e}_1=\bm{e}_0+\bm{\Delta}.
\]
Finally, each feature was standardized using its mean and population standard deviation computed jointly over the concatenated source, midpoint, and target expression matrices. Therefore, the Distractor dataset contains a cell type dependent temporal expression shift, while its spatial dynamics remain the same rigid rotation as in Rotation.

For the TP-DATE experiments in the main text, the hyperparameter was set to \((\eta_{\mathrm{spatial}},\lambda_{\mathrm{spatial}})=(1,16).\)

\subsection{Real data details}
\label{app:real details}
\subsubsection{Developing Mouse brain}
The developing mouse brain data was adopted from \citep{MOSTA}. This dataset contains spatial transcriptomics sections from mouse embryonic development at eight different time points. We selected brain cells according to the cell annotations provided with the dataset and performed Leiden clustering \citep{Leiden} on gene expression space for visualization purpose only. We use the snapshots from day 14.5, 15.5, and 16.5, denoted E14.5, E15.5, and E16.5 with 17591, 17031, 17296 cells, respectively. We reduced the dimension by first selecting 1500 highly variable genes, then PCA to 50D. The spatial coordinates were linearly scaled to \([-1,1]\). The hold one out experiment was trained on E14.5, E16.5, and evaluated on E15.5. The hyperparameters are set to \((\eta_{\mathrm{spatial}},\lambda_{\mathrm{spatial}})=(100,0.01)\). An overview of the data is shown in figure \ref{fig:Mouse Brain overview}. Different colors stand for different Leiden clusters.

\begin{figure}[ht]
    \centering
    \includegraphics[width=0.7\linewidth]{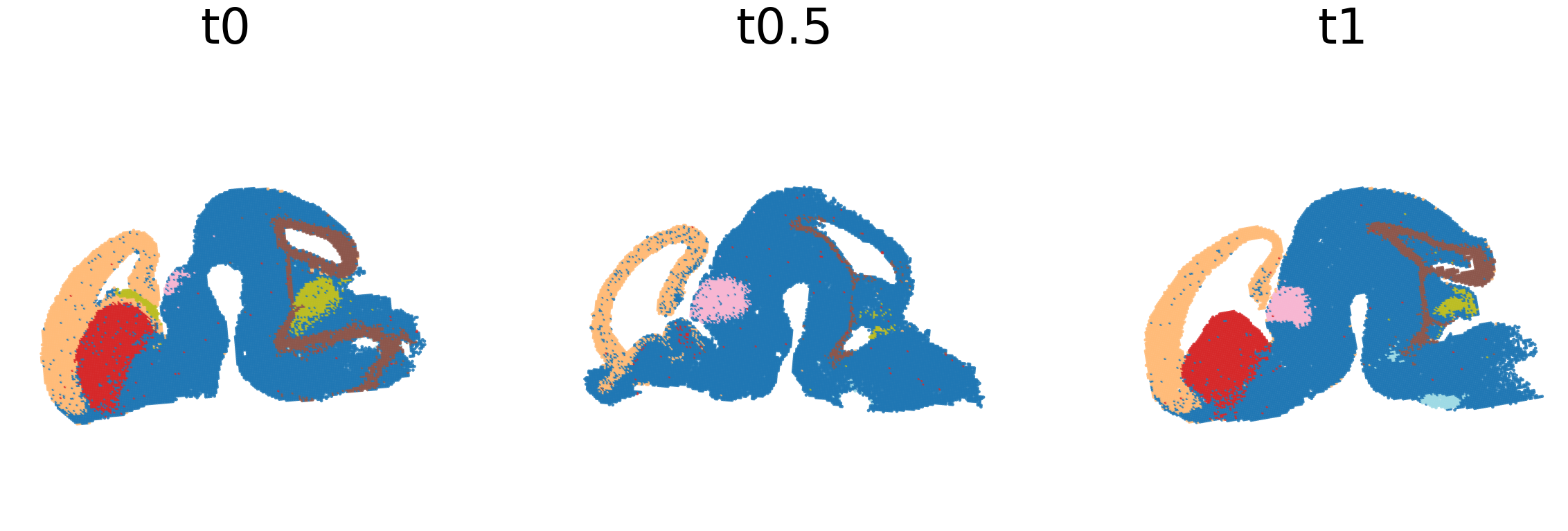}
    \caption{Mouse Brain overview}
    \label{fig:Mouse Brain overview}
\end{figure}

\subsubsection{ARTISTA}
The ARTISTA data was adopted from \citep{ARTISTA}. This dataset describes the brain regeneration process in axolotl following injury and contains spatial transcriptomics sections collected at seven different time points from Day 2 to Day 60 after injury. In the main text, we use snapshots at Day 2, Day 5, and Day 10, and subsample 7,500 cells per snapshot. The expression vectors were also dimension reduced to 50D by PCA. The spatial coordinates were linearly scaled to \([-1,1]\). For hold one out experiment, we trained on Day 2 and Day 10, and evaluated on Day 5. The hyperparameters are set to \((\eta_{\mathrm{spatial}},\lambda_{\mathrm{spatial}})=(100,0.01)\). The normalized evaluation time is therefore \(0.375\) instead of \(0.5\). An overview of the data is shown in figure \ref{fig:ARTISTA overview}. Different colors stand for Niche annotations provided with the dataset.

\begin{figure}[ht]
    \centering
    \includegraphics[width=0.7\linewidth]{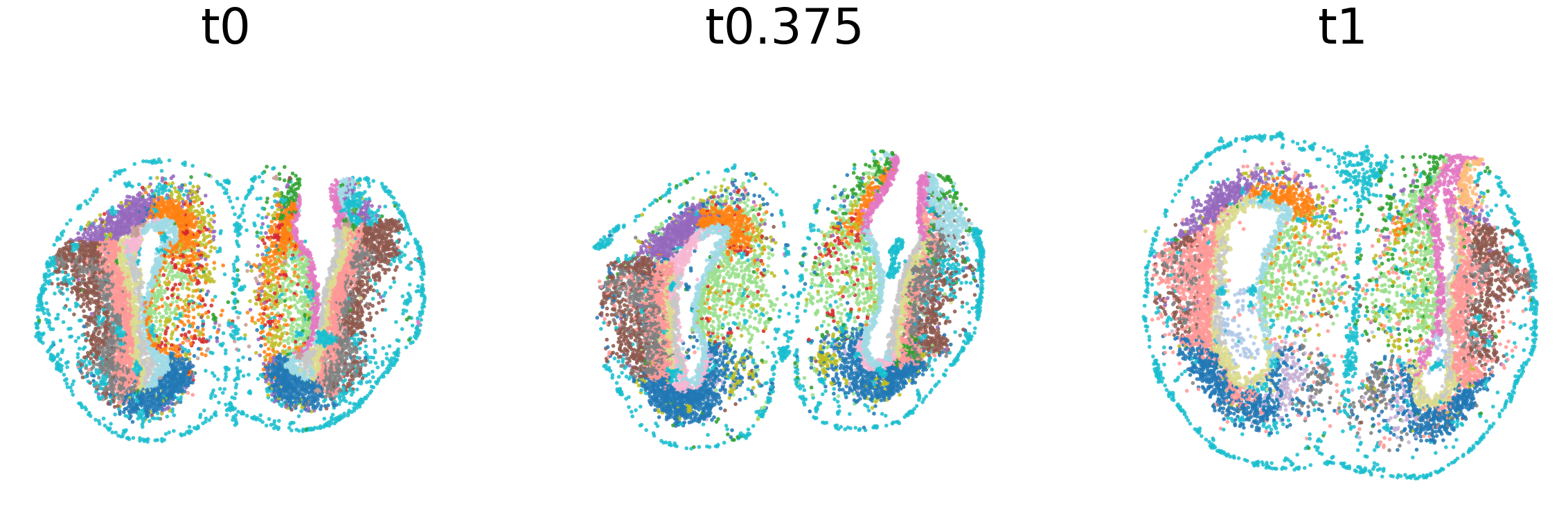}
    \caption{ARTISTA overview}
    \label{fig:ARTISTA overview}
\end{figure}

\subsubsection{Tumor}
The Tumor data was adopted from \citep{3DTumor}. This dataset contains five spatial transcriptomics sections of the same MC38 tumor collected at different depths (subQ-1 to subQ-5). We use subQ-3, subQ-4, and subQ-5 with 10,000 spatially sampled cells per snapshot. The expression vectors are also reduced to 50D via PCA. The spatial coordinates were linearly scaled to \([-1,1]\). For hold one out experiments, models are trained on subQ-3 and subQ-5, and evaluated on subQ-4. The hyperparameters are set to \((\eta_{\mathrm{spatial}},\lambda_{\mathrm{spatial}})=(0.01,1)\). An overview of the data is shown in figure \ref{fig:Tumor overview}. Different colors stand for Niche annotations provided with the dataset.

\begin{figure}[ht]
    \centering
    \includegraphics[width=0.7\linewidth]{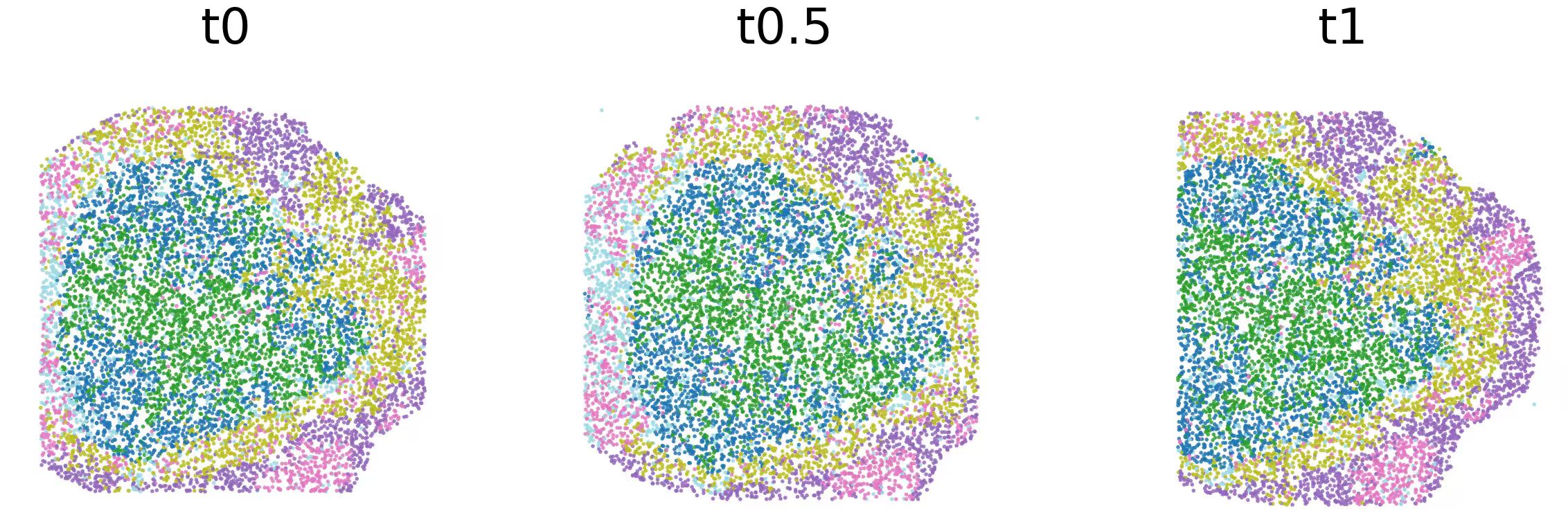}
    \caption{Tumor overview}
    \label{fig:Tumor overview}
\end{figure}

\subsection{Evaluation metrics}
\label{app:Evaluation metrics}
In the hold one out experiments on real data, we train the model on the two endpoint time points, infer the distribution at the intermediate time point, and compare the prediction with the ground truth. In the spatial transcriptomics setting, each cell is jointly characterized by its gene expression \(\bm{x}\) and spatial coordinate \(\bm{s}\), so that each snapshot can be regarded as a joint distribution \(p(\bm{x},\bm{s})\). At the intermediate time point, we evaluate the discrepancy between the ground truth distribution \(p\) and the predicted distribution \(\hat{p}\).

\subsubsection{Marginal Wassersteins}
To measure the discrepancy between two joint distribution, one can first measure the discrepancy between the marginal distributions. Let \(p_x(\bm{x})=\int p(\bm{x},\bm{s})\mathrm{d}\bm{s},p_s(\bm{s})=\int p(\bm{x},\bm{s})\mathrm{d}\bm{x}\) be the marginals on expression space and spatial space, the spatial Wasserstein distance and expression Wasserstein distance are defined as

\begin{equation}
\begin{aligned}
\text{Spatial }W_2^2(p,\hat{p})=\inf_{\gamma\in\Pi(p_s,\hat{p}_s)}\int \Vert\bm{s}-\bm{s}'\Vert^2\gamma(\bm{s},\bm{s}')\mathrm{d}\bm{s}\mathrm{d}\bm{s}'\\
\text{Expression }W_2^2(p,\hat{p})=\inf_{\gamma\in\Pi(p_x,\hat{p}_x)}\int \Vert\bm{x}-\bm{x}'\Vert^2\gamma(\bm{x},\bm{x}')\mathrm{d}\bm{x}\mathrm{d}\bm{x}'
\end{aligned}
\end{equation}
Intuitively, these marginal discrepancies measure how well can the model reconstruct the spatial or expression information solely. To further measure how well can the model reconstruct the spatial and expression information jointly, we introduce fused Wasserstein distance and spatially coupled expression MSE (SC-eMSE).

\subsubsection{Fused Wasserstein}
Fused \(W_2^2\) is defined as
\begin{equation}
\text{Fused }W_2^2(p,\hat{p})=\inf_{\gamma\in\Pi(p,\hat{p})}\int \Vert\bm{z}-\bm{z}'\Vert^2\gamma(\bm{z},\bm{z}')\mathrm{d}\bm{z}\mathrm{d}\bm{z}'
\end{equation}
where \(\bm{z}=(\bm{x},\alpha\bm{s})\). The expression vector \(\bm{x}\) was concatenated with a scaled spatial vector \(\alpha\bm{s}\), and the standard Wasserstein distance is calculated based on the concatenated vector. \(\alpha\) was introduced to balance the importance of \(\bm{x}\) and \(\bm{s}\), since they may have different dimension and scale. For all datasets, we choose \(\alpha\) with a common rule
\begin{equation}
\alpha=
\frac{\sum_{d=1}^{D_x}\operatorname{Std}(\bm{x}^d)}
{\sum_{d=1}^{D_s}\operatorname{Std}(\bm{s}^d)}.    
\end{equation}
where \(D_x,D_s\) are the dimensions of expression space and spatial space, \(\bm{x}^d,\bm{s}^d\) are the \(d\)-th components of \(\bm{x},\bm{s}\). Intuitively, this selection of \(\alpha\) balances the total standard deviation of expression and space. In hold one out experiments, the standard deviations are calculated only on the training time points.

\subsubsection{SC-eMSE}
To test whether expression is reconstructed at the correct spatial location, we additionally report spatially coupled expression MSE (SC-eMSE). Let \(\gamma\) be the optimal coupling between \(\rho,\hat{\rho}\), obtained using spatial coordinates \(\bm{s}\) alone with the spatial cost \(\Vert\bm{s}-\bm{s}'\Vert^2\), we define
\[
\operatorname{SC\text{-}eMSE}(\rho,\hat{\rho})=\frac{1}{D_x}\int\Vert\bm{x}-\bm{x}'\Vert^2\gamma(\bm{z},\bm{z}')\mathrm{d}\bm{z}\mathrm{d}\bm{z}'
\]
where \(\bm{z}=(\bm{x},\bm{s})\). The spatial coupling \(\gamma\) is fitted without expression information, while the transport cost \(\Vert\bm{x}-\bm{x}'\Vert^2\) is calculated without spatial information. Intuitively, SC-eMSE measures whether the predicted snapshot shares similar spatial expression patterns with the ground truth.

\subsubsection{Toydata metrics}
On rotation toys, the ground truth dynamics are known rigid rotations. Therefore, we can directly calculate the spatial MSE to measure the spatial level accuracy, instead of calculating distribution-level Wasserstein distance. At the unseen time point \(t\), let the ground truth spatial states be \(\bm{s}^\star\), the predicted spatial states be \(\hat{\bm{s}}\). The spatial MSE is defined as
\begin{equation}
\operatorname{MSE}_{s}=\frac{1}{n}\sum_{i=1}^n
\Vert\hat{\bm{s}}_i-\bm{s}_i^\star\Vert^2.   
\end{equation}

We also calculated the pair distortion to measure how well can the predicted dynamics preserve the spatial structure. We uniformly sampled \(M=6000\) pairs
\(\mathcal P=\{(i_r,k_r)\}_{r=1}^M\) where \(i_t\neq k_r\). The pair distortion is defined as 
\begin{equation}
\operatorname{PD}=\frac{1}{M}\sum_{r=1}^M
\frac{\left|\Vert\hat s_{i_r}-\hat s_{k_r}\Vert_2
-\Vert s_{i_r}^\star-s_{k_r}^\star\Vert_2\right|}
{\Vert s_{i_r}^\star-s_{k_r}^\star\Vert_2}.    
\end{equation}
Intuitively, it measures the average relative error of pairwise spatial distance. A lower PD means the predicted dynamics preserves the spatial structure better. In practice, \(\Vert s_{i_r}^\star-s_{k_r}^\star\Vert_2\) can be easily calculated at time \(0\)
\[
\Vert s_{i_r}^\star-s_{k_r}^\star\Vert_2=\Vert s_{i_r}^0-s_{k_r}^0\Vert_2
\]
since the true dynamics are rigid transformations. Note that, the ground truth pair distortions are zero on toy datasets, so lower values are meaningful here. But it is not the case of real datasets where nonzero biological deformation may
be correct, and the held out cells have no source-cell correspondence. Therefore, we only report pair distortions on toy datasets.

\subsection{GW-CFM and FGW-CFM baselines}
\label{app:GW-CFM}
In the main text, we compared TP-DATE with GW-CFM and FGW-CFM. These baselines are introduced as a naive dynamic extension of static GW-OT and FGW-OT. We used GW-OT coupling or FGW-OT coupling, and the displacement conditional path \(\bm{x}_t=(1-t)\bm{x}_0+t\bm{x}_1\) for conditional flow matching training. That is, for two probability densities \(\mu_0,\mu_1\), we first calculate the GW-OT coupling or FGW-OT coupling \(\gamma(\bm{x}_0,\bm{x}_1)\). Next, we parameterize a velocity network \(\bm{v}_{\bm{\theta}}(\bm{x},t)\), and train it by minimizing the conditional flow matching loss
\begin{equation}
\begin{aligned}
\label{eq:CFM app}
&\mathcal{L}_{\text{CFM}}(\bm{\theta})=
\mathbb{E}_{t\sim\mathcal{U}[0,1], (\bm{x}_0,\bm{x}_1)\sim \gamma(\bm{x}_0,\bm{x}_1), \bm{x}=(1-t)\bm{x}_0+t\bm{x}_1}\left\| \bm{\bm{v}_{\theta}}(\bm{x},t) - (\bm{x}_1-\bm{x}_0) \right\|_2^2.
\end{aligned}
\end{equation}
It can be viewed as replacing the OT coupling in OT-CFM by GW/FGW-OT coupling. Intuitively, it is a kind of linear interpolation between \(\mu_0,\mu_1\) induced by the corresponding coupling, therefore a natural choice for dynamic extension. For FGW-OT, we fused OT and GW-OT with weight \(7:3\).

\subsection{Static QOT solver}
\label{app:static QOT solver}
The static QOT problem is
\begin{equation}
\label{eq:static QOT app2}
\begin{aligned}
&\text{QOT}_S(\mu_0,\mu_1) =\inf_{\gamma\in\Pi(\mu_0,\mu_1)} \int_{\mathcal{X}^4} \mathcal{A}(\bm{z},\bm{z}')\gamma(\bm{z})\gamma(\bm{z}')\mathrm{d} \bm{z}\mathrm{d} \bm{z}'.\\
\end{aligned}
\end{equation}
In practice, \(\mu_0,\mu_1\) are represented by point cloud data \(\{X_i\}_{i=1:M},\{Y_j\}_{j=1:N}\). The conditional variables \(\bm{z}=(X_i,Y_j),\bm{z}'=(X_k,Y_l)\) are two end point pairs. The static cost in (\ref{eq:static QOT app2}) can be written as
\begin{equation}
\label{eq:discrete static QOT}
Q(\Gamma)=\sum A_{ij,kl}\Gamma_{ij}\Gamma_{kl}
\end{equation}
where \(\Gamma\in\{\Gamma\in\mathbb{R}_{\ge0}^{M\times N}|\Gamma\bm{1}=\frac{1}{M}\bm{1},\bm{1}^{\mathrm{T}}\Gamma=\frac{1}{N}\bm{1}^{\mathrm{T}}\}\) is the empirical coupling, \(A_{ij,kl}=A(X_i,Y_j,X_k,Y_l)\) is the path action. The static QOT problem is then formulated as
\begin{equation}
\min_{\Gamma}\sum A_{ij,kl}\Gamma_{ij}\Gamma_{kl}
\end{equation}
which is a quadratic programming w.r.t \(\Gamma\).
\subsubsection{Frank-Wolfe algorithm}
The Frank-Wolfe algorithm \citep{kerdoncuff2021sampled} is an iterative first-order optimization algorithm for convex constraint optimization. Though QOT problem may not be convex, it is standard and widely used in solving Gromov-Wasserstein type static QOT problem \citep{pot}. Consider a convex set \(D\) and an optimization problem
\begin{equation}
\min_{\bm{x}\in D}f(\bm{x})
\end{equation}
The Frank-Wolfe algorithm solves it by iteratively doing
\begin{equation}
\label{eq:Frank-Wolfe}
\begin{aligned}
&\text{Initialization}\quad \bm{x}_0\in D,\quad k=0\\
&\text{Step1.}\quad 
\bm{s}_k=\mathop{\arg \min}_{\bm{s}\in D} \bm{s}^{\mathrm{T}}\nabla f(\bm{x}_k)\\
&\text{Step2.}\quad \bm{x}_{k+1}=\bm{x}_k+\frac{2}{k+2}(\bm{s}_k-\bm{x}_k),\quad k=k+1\\
\end{aligned}
\end{equation}
In our case, with symmetric conditions such as \(\mathcal{L}\) is symmetric to \((\bm{x},\dot{\bm{x}}),(\bm{y},\dot{\bm{y}})\), it is easy to check \(A_{ij,kl}=A_{kl,ij}\). Therefore, \(\nabla Q(\Gamma)=2\sum A_{ij,kl}\Gamma_{kl}\). The Frank-Wolfe algorithm (\ref{eq:Frank-Wolfe}) can be realized as
\begin{equation}
\label{eq:Frank-Wolfe QOT}
\begin{aligned}
&\text{Initialization}\quad \Gamma^0,\quad n=0\\
&\text{Step1.}\quad G_{ij}^n=2\sum A_{ij,kl}\Gamma_{kl}^n\\
&\text{Step2.}\quad 
\Pi^{n}=\mathop{\arg \min}_{\Pi} \langle\Pi,G^n\rangle\\
&\text{Step3.}\quad \Gamma^{n+1}=\Gamma^n+\frac{2}{n+2}(\Pi^n-\Gamma^n),\quad n=n+1\\
\end{aligned}
\end{equation}
In step 1, we estimate the gradient \(G\). In step 2, we solve a standard optimal transport subproblem. In step 3, we update \(\Gamma\). The OT subproblem is relaxed and solved by Sinkhorn.

\subsubsection{Gradient estimation}
The full computation of step 1 is \(\mathcal{O}(M^2N^2)\) which is expensive. To reduce the computational cost, we estimate the gradient via Monte Carlo.
\begin{equation}
\hat{G}_{ij}=\frac{2}{R}\sum_{h=1}^RA_{ij,k_hl_h},\quad (k_h,l_h)\overset{i.i.d.}{\sim}\Gamma
\end{equation}
The computational complexity is \(\mathcal{O}(RMN)\) where \(R\) is the number of Monte Carlo samples. In the main text, we use \(R=64\).

\subsubsection{Sparse KNN}
If \(\mathcal{O}(RMN)\) is still too expensive, we can further construct a bidirectional KNN graph between \(\{X_i\}_{i=1:M},\{Y_j\}_{j=1:N}\) and restrict \(\Gamma\) on the graph. Intuitively, \(X_i\) is more likely to transport to some \(Y_j\) near to itself rather than far away from itself. Therefore, we construct two KNN graphs on gene expression space. Each \(X_i\) is linked to its \(K\) nearest \(Y_j\), and each \(Y_j\) is also linked to its \(K\) nearest \(X_i\). We take the union of the edges, denoted \(E_0\). That is, \((i,j)\in E_0\) means \(X_i\) is one of the \(K\) nearest neighbors of \(Y_j\), or \(Y_j\) is one of the \(K\) nearest neighbors of \(X_i\). It is easy to check \(|E_0|\le K(M+N)\). We use bidirectional KNN rather than one way KNN so that all points have neighbors. We then want to restrict the coupling \(\Gamma\) on \(E_0\), that is
\begin{equation}
\label{eq:sparse gamma 0}
\Gamma\in\mathbb{R}_{\ge0}^{M\times N},\quad\Gamma\bm{1}=\frac{1}{M}\bm{1},\quad \quad\bm{1}^{\mathrm{T}}\Gamma=\frac{1}{N}\bm{1}^{\mathrm{T}},\quad \Gamma_{ij}=0,(i,j)\notin E_0.
\end{equation}
Although the bidirectional construction guarantees that every row and column has at least one admissible edge, this alone does not necessarily guarantee the existence of a coupling with the prescribed marginals. We therefore augment \(E_0\) by a small number of additional edges to guarantee feasibility. Specifically, initialize residual masses
\[
r_i=\frac{1}{M},\qquad s_j=\frac{1}{N},
\]
and an auxiliary coupling \(\bar{\Gamma}=0\). We first scan all edges \((i,j)\in E_0\) in dictionary order. For each edge, we assign
\[
\delta_{ij}=\min\{r_i,s_j\},\qquad
\bar{\Gamma}_{ij}\leftarrow \bar{\Gamma}_{ij}+\delta_{ij},
\]
and update \(r_i\leftarrow r_i-\delta_{ij}\), \(s_j\leftarrow s_j-\delta_{ij}\). After all KNN edges have been scanned, if residual masses remain, we repeatedly select a pair \(i,j\) with \(r_i>0\) and \(s_j>0\), add the edge \((i,j)\), and assign
\[
\delta=\min\{r_i,s_j\}.
\]
Each added edge exhausts at least one residual row or column. Therefore, at most \(M+N-1\) additional edges are required. Denoting these repair edges by \(E_{\mathrm{rep}}\), we finally set
\[
E=E_0\cup E_{\mathrm{rep}}.
\]
By construction, \(\bar{\Gamma}\) satisfies
\[
\bar{\Gamma}\bm{1}=\frac{1}{M}\bm{1},
\qquad
\bm{1}^{\mathrm{T}}\bar{\Gamma}
=\frac{1}{N}\bm{1}^{\mathrm{T}},
\qquad
\operatorname{supp}(\bar{\Gamma})\subseteq E,
\]
which explicitly guarantees that the sparse coupling constraint is feasible. Moreover,
\[
|E|
\le K(M+N)+M+N-1,
\]
and the feasibility repair requires only
\(\mathcal{O}(K(M+N))\) additional computation. We then restrict the coupling \(\Gamma\) on the augmented sparse support \(E\), that is
\begin{equation}
\label{eq:sparse gamma}
\Gamma\in\mathbb{R}_{\ge0}^{M\times N},\quad
\Gamma\bm{1}=\frac{1}{M}\bm{1},\quad
\bm{1}^{\mathrm{T}}\Gamma=\frac{1}{N}\bm{1}^{\mathrm{T}},\quad
\Gamma_{ij}=0,\ (i,j)\notin E.
\end{equation}
Mathematically, solving the OT subproblem under constraint (\ref{eq:sparse gamma}) is equivalent to solving a standard OT problem with a masked cost
\begin{equation}
\left\{
\begin{aligned}
&\tilde{G}_{ij}=G_{ij},\quad (i,j)\in E\\
&\tilde{G}_{ij}=+\infty,\quad (i,j)\notin E
\end{aligned}
\right .
\end{equation}
In Sinkhorn iteration, infinite cost leads to zero element in Gibbs kernel
\begin{equation}
\left\{
\begin{aligned}
&K_{ij}=\exp(-\tilde{G}_{ij}/\tau)=\exp(-G_{ij}/\tau),\quad(i,j)\in E\\
&K_{ij}=\exp(-\tilde{G}_{ij}/\tau)=0,\quad(i,j)\notin E\\
\end{aligned}
\right .
\end{equation}
where \(\tau\) is the coefficient of entropy relaxation. Further, a zero element in Gibbs kernel leads to zero element in updated coupling
\begin{equation}
\Gamma^{\text{new}}_{ij}=u_iK_{ij}v_j=0,\quad(i,j)\notin E
\end{equation}
where \(u,v\) are dual variables. The analysis above shows that we can only estimate \(G_{ij}\) and update \(\Gamma_{ij}\) for \((i,j)\in E\). For \((i,j)\notin E\), \(K_{ij},\Gamma_{ij}\) are automatically zero. The complexity of gradient estimation is therefore further reduced to \(\mathcal{O}(RK(M+N))\). In the main text, we use \(\tau=10^{-4}, K=64\). We conducted an ablation study for \(K\) in \ref{app:ablation}.

\section{Additional results}
\subsection{Ablation studies}
\label{app:ablation}
\subsubsection{Loss weights}
To demonstrate the robustness of TP-DATE, we conducted an ablation study for hyperparameters \((\eta_{\text{spatial}},\lambda_{\text{spatial}})\) \ref{app:Lagrangian choice} on ARTISTA. As shown in Table \ref{tab:artista-sensitivity}, TP-DATE remains robust under different hyperparameter settings. We did not run different random seeds for this ablation experiment. All results are reported in one run.
\begin{table}[ht]
  \centering
  \caption{TP-DATE hyperparameter sensitivity on ARTISTA.}
  \label{tab:artista-sensitivity}
  \resizebox{0.8\linewidth}{!}{%
\begin{tabular}{cccccc}
\toprule
$\eta_{\mathrm{spatial}}$ & $\lambda_{\mathrm{spatial}}$ & Spatial $W_2$ ($\downarrow$) & Expression $W_2$ ($\downarrow$) & SC-eMSE ($\downarrow$) & Balanced fused $W_2$ ($\downarrow$) \\
\midrule
$10^{-2}$ & $10^{-2}$ & 0.10021 & 5.39525 & 0.89497 & 8.34454 \\
$10^{-2}$ & $10^{-1}$ & 0.10339 & 5.40739 & 0.90109 & 8.48391 \\
$10^{-2}$ & $1$       & 0.09922 & 5.41446 & 0.90192 & 8.35143 \\
$10^{-2}$ & $10$      & 0.10366 & 5.42960 & 0.90873 & 8.49660 \\
$10^{-2}$ & $10^{2}$  & 0.11070 & 5.35272 & 0.89337 & 8.75079 \\
\midrule
$10^{-1}$ & $10^{-2}$ & 0.10039 & 5.38527 & 0.89438 & 8.34612 \\
$10^{-1}$ & $10^{-1}$ & 0.10099 & 5.36024 & 0.88781 & 8.35834 \\
$10^{-1}$ & $1$       & 0.10213 & 5.35141 & 0.88634 & 8.38445 \\
$10^{-1}$ & $10$      & 0.10323 & 5.40274 & 0.90401 & 8.48153 \\
$10^{-1}$ & $10^{2}$  & 0.10297 & 5.40772 & 0.91005 & 8.49803 \\
\midrule
$1$ & $10^{-2}$ & 0.09734 & 5.40196 & 0.89723 & 8.24340 \\
$1$ & $10^{-1}$ & 0.09875 & 5.42559 & 0.90137 & 8.29199 \\
$1$ & $1$       & 0.09985 & 5.50784 & 0.92271 & 8.41832 \\
$1$ & $10$      & 0.09786 & 5.46790 & 0.91053 & 8.30708 \\
$1$ & $10^{2}$  & 0.10537 & 5.24299 & 0.86934 & 8.50025 \\
\midrule
$10$ & $10^{-2}$ & 0.09331 & 5.36099 & 0.87991 & 8.06390 \\
$10$ & $10^{-1}$ & 0.09105 & 5.34367 & 0.87604 & 7.96595 \\
$10$ & $1$       & 0.09381 & 5.37037 & 0.87955 & 8.09346 \\
$10$ & $10$      & 0.09014 & 5.36759 & 0.87766 & 7.94178 \\
$10$ & $10^{2}$  & 0.09154 & 5.43260 & 0.89303 & 8.03001 \\
\midrule
$10^{2}$ & $10^{-2}$ & 0.08922 & 5.23544 & 0.83938 & 7.79900 \\
$10^{2}$ & $10^{-1}$ & 0.08859 & 5.33771 & 0.86599 & 7.86204 \\
$10^{2}$ & $1$       & 0.09080 & 5.38808 & 0.87992 & 8.00044 \\
$10^{2}$ & $10$      & 0.08824 & 5.28488 & 0.85379 & 7.79917 \\
$10^{2}$ & $10^{2}$  & 0.08751 & 5.31516 & 0.86013 & 7.77876 \\
\bottomrule
\end{tabular}%
  }
\end{table}

\subsubsection{Sparse KNN}
We also conducted an ablation study for \(K\) in sparse KNN construction on the Tumor dataset. We tested \(K=64,128,256,512,1024,\) on the Tumor dataset. The corresponding model performance are approximately the same, demonstrating the robustness of sparse KNN technique (Table \ref{tab:tumor-knn-sensitivity}).

\begin{table}[ht]
  \centering
  \caption{$K$ ablation study on Tumor dataset. The dagger
  marks the value used in the main experiments.}
  \label{tab:tumor-knn-sensitivity}
  \resizebox{\linewidth}{!}{%
\begin{tabular}{rcccc}
\toprule
$k$ & Spatial $W_2$ ($\downarrow$) & Expression $W_2$ ($\downarrow$) & SC-eMSE ($\downarrow$) & Balanced fused $W_2$ ($\downarrow$)\\
\midrule
$64^{\dagger}$ & $0.069590\!\pm\!0.001515$ & $4.819554\!\pm\!0.138743$ & $1.045954\!\pm\!0.028812$ & $7.374801\!\pm\!0.087657$\\
$128$ & $0.070472\!\pm\!0.002885$ & $4.802107\!\pm\!0.196712$ & $1.037978\!\pm\!0.044555$ & $7.399844\!\pm\!0.127601$\\
$256$ & $0.069701\!\pm\!0.001579$ & $4.766193\!\pm\!0.215132$ & $1.029943\!\pm\!0.052658$ & $7.348703\!\pm\!0.145340$\\
$512$ & $0.071247\!\pm\!0.002133$ & $4.816275\!\pm\!0.175593$ & $1.037625\!\pm\!0.039205$ & $7.395463\!\pm\!0.142730$\\
$1024$ & $0.070168\!\pm\!0.003771$ & $4.830717\!\pm\!0.111594$ & $1.042600\!\pm\!0.017763$ & $7.395537\!\pm\!0.042100$\\
\bottomrule
\end{tabular}%
  }
\end{table}

We also recorded the training time and memory usage under different values of \(K\). As shown in Figs.\ref{fig:KNN Tumor}, on a log--log scale, the slopes of the linear regressions for both training time and GPU memory usage with respect to \(K\) are close to \(1\), indicating that, as expected, the computational cost of TP-DATE scales approximately linearly with \(K\). CPU memory usage does not exhibit the same linear dependence, mainly because CPU memory is substantially affected by data loading, preprocessing, and other fixed overheads, which dominate the algorithmic memory cost when \(K\) is relatively small.


\begin{figure}[ht]
    \centering
    \includegraphics[width=0.8\linewidth]{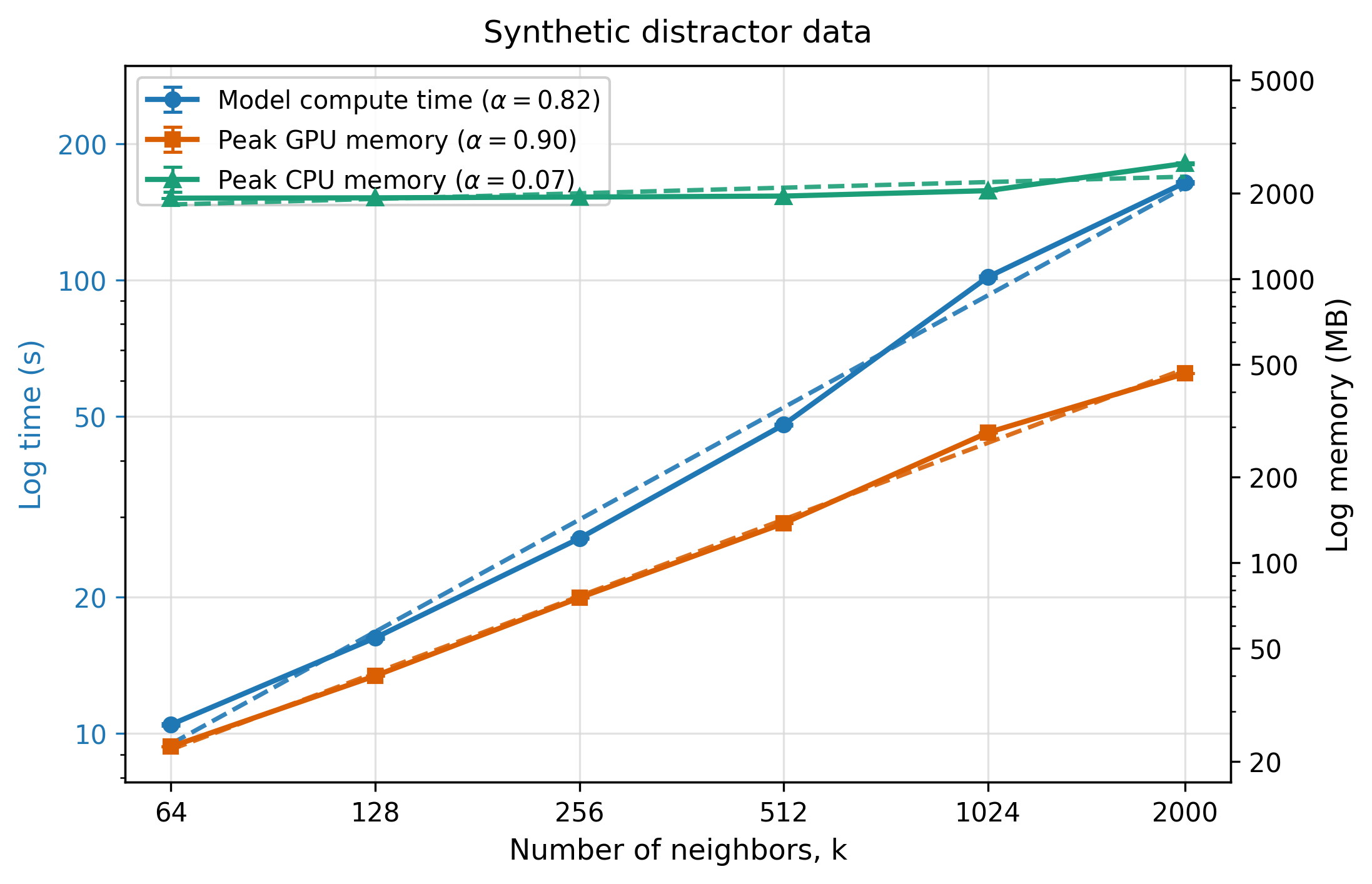}
    \caption{Training time and memory usage for different \(K\) on Tumor.}
    \label{fig:KNN Tumor}
\end{figure}

\subsubsection{Conditional path}
In the Mouse Brain and ARTISTA experiments, we used relatively small weights for the interaction term. To isolate the contribution of the interaction-induced conditional path from that of the static QOT coupling, we trained an additional baseline using exactly the same QOT coupling as TP-DATE but replacing the travelling-pair conditional path with standard displacement interpolation. All other training settings and random seeds were kept unchanged. As shown in Table~\ref{tab:path ablation}, TP-DATE achieves lower mean errors across all evaluation metrics on both datasets, with particularly clear improvements in expression reconstruction, SC-eMSE, and fused \(W_2\). These results indicate that the performance gain cannot be attributed solely to the static QOT coupling. Even with a relatively small interaction weight, the interaction induced dynamic conditional path provides an additional and consistent benefit. We do not additionally consider the zero-interaction coupling as a separate ablation, since removing the interaction term reduces the induced static QOT to the corresponding weighted standard OT formulation, whose flow-matching counterpart is already represented by OT-CFM.

\begin{table}[ht]
    \centering
    \caption{Conditional path ablation on spatial transcriptomics hold one out experiment. We report mean and
    standard deviation over 5 random seeds.}
    \label{tab:path ablation}
    \resizebox{\linewidth}{!}{%
  \begin{tabular}{llcccc}
  \toprule
  Dataset & Method & Spatial $W_2$ ($\downarrow$) & Expression $W_2$ ($\downarrow$) & SC-eMSE ($\downarrow$) & Balanced fused $W_2$ ($
  \downarrow$) \\
  \midrule
  \multirow{2}{*}{Mouse brain}
   & TP-DATE (ours)
   & $\mathbf{0.08678\!\pm\!0.00291}$
   & $\mathbf{6.40989\!\pm\!0.04644}$
   & $\mathbf{1.47693\!\pm\!0.02312}$
   & $\mathbf{9.59301\!\pm\!0.09255}$ \\
   & QOT + straight path
   & $0.09615\!\pm\!0.00500$
   & $6.45808\!\pm\!0.05616$
   & $1.52589\!\pm\!0.03583$
   & $10.05324\!\pm\!0.22572$ \\
  \midrule
  \multirow{2}{*}{ARTISTA}
   & TP-DATE (ours)
   & $\mathbf{0.09196\!\pm\!0.00500}$
   & $\mathbf{5.39310\!\pm\!0.28347}$
   & $\mathbf{0.87824\!\pm\!0.07587}$
   & $\mathbf{8.01158\!\pm\!0.27289}$ \\
   & QOT + straight path
   & $0.09210\!\pm\!0.00345$
   & $5.94608\!\pm\!0.39138$
   & $1.02462\!\pm\!0.10252$
   & $8.42871\!\pm\!0.30357$ \\
  \bottomrule
  \end{tabular}%
    }
  \end{table}
  
\subsection{Scalability}
\label{app:scaling}
We trained TP-DATE and simulation-based baselines (stVCR \citep{peng2026stvcr}, CytoBridge \citep{zhang2025cytobridge}) on different cell numbers. We downsampled the 
subQ-3, subQ-5 slices of the Tumor dataset \citep{3DTumor} to 10,000, 30,000, 50,000, 70,000, and 90,000 cells, respectively, and recorded the training time and memory usage of the three methods on datasets of the corresponding sizes. As shown in Figure \ref{fig:Memory}, CytoBridge ran out of memory on 90000, stVCR and TP-DATE both require moderate memory. As shown in Figure \ref{fig:Time}, TP-DATE exhibits computational cost that scales linearly with dataset size and is more efficient than the two simulation-based baselines. This highlights the computational efficiency of simulation-free approaches such as flow matching.

\begin{figure}[ht]
    \centering
    \includegraphics[width=0.8\linewidth]{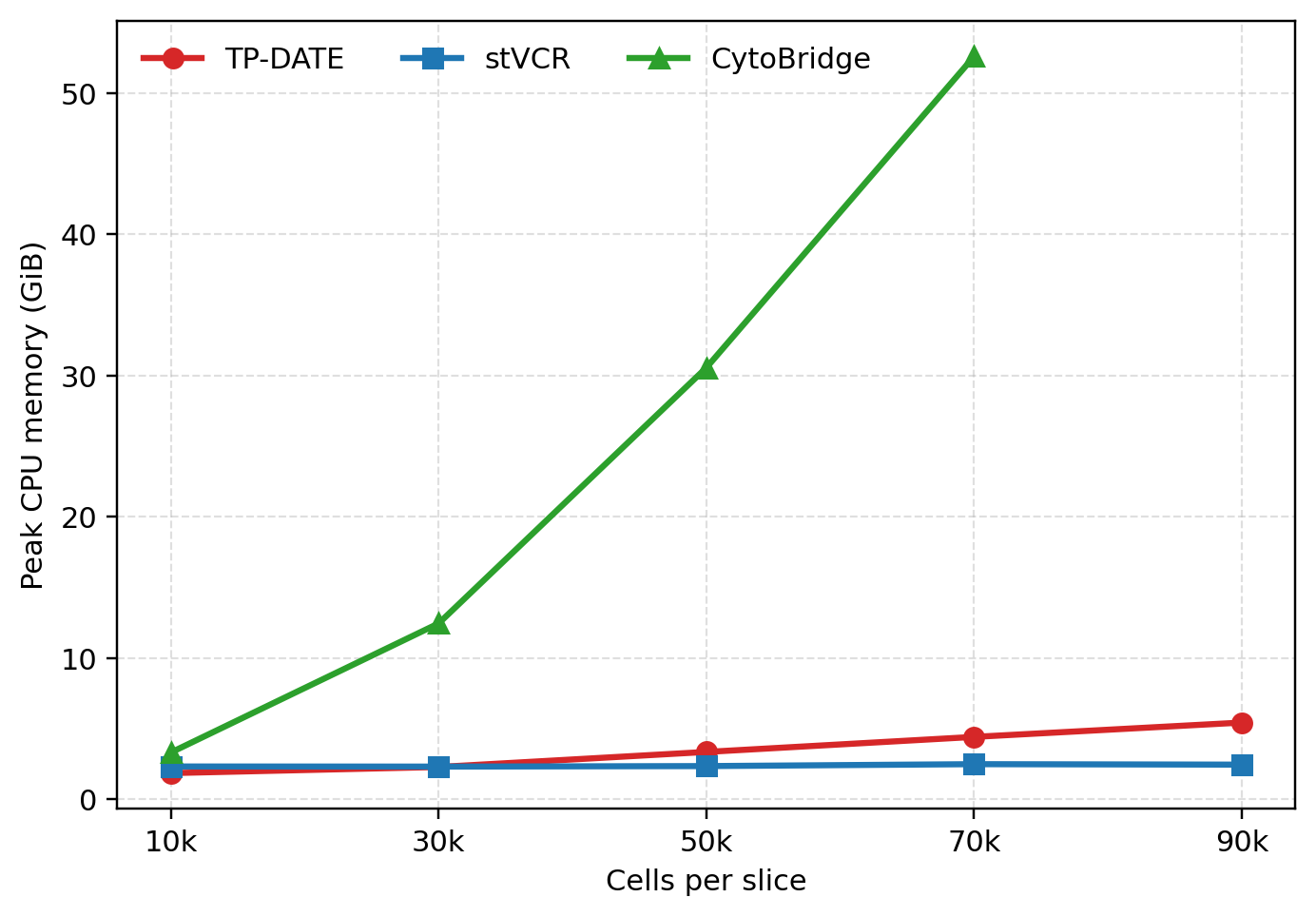}
    \includegraphics[width=0.8\linewidth]{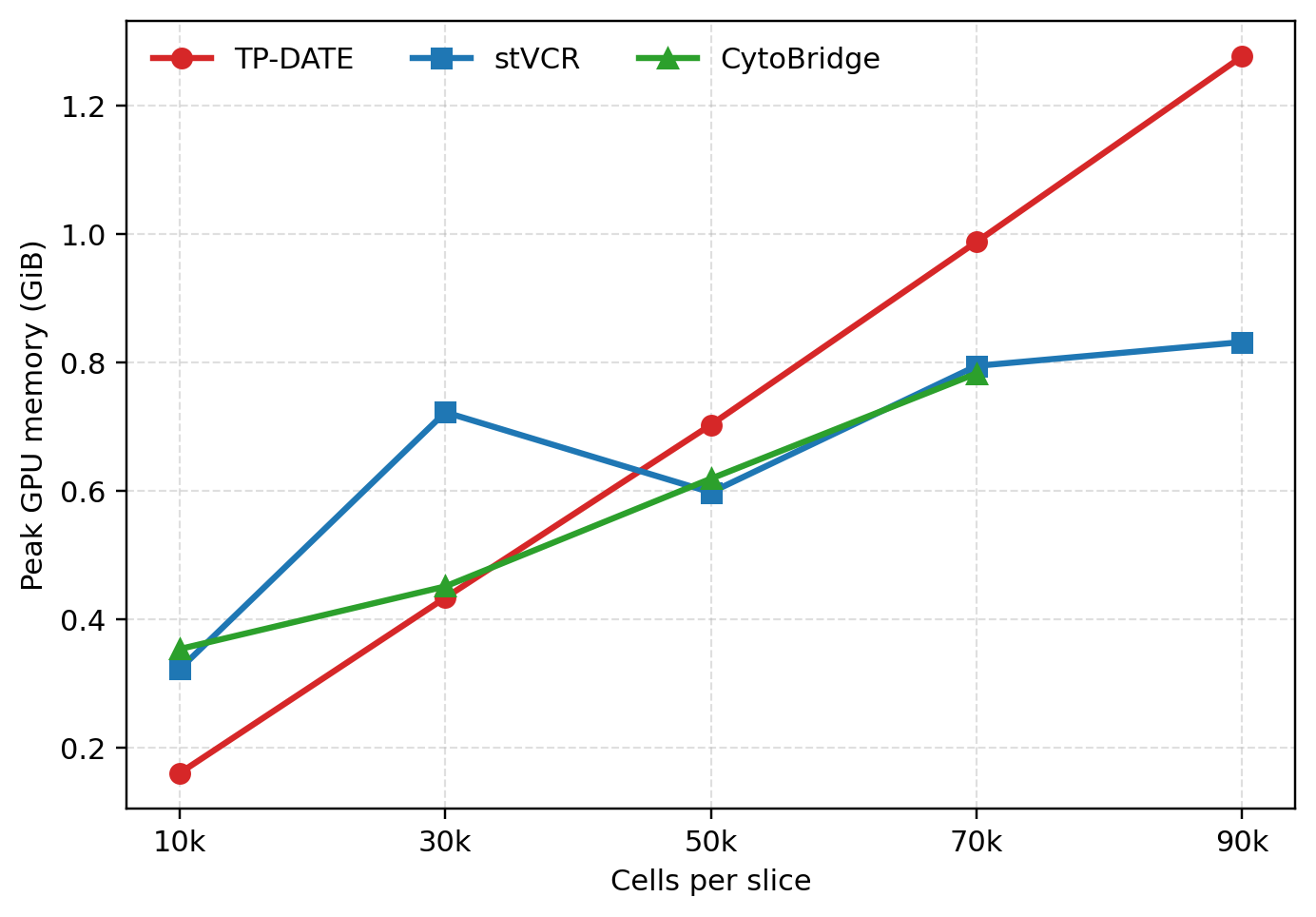}
    \caption{Training memory on different size data}
    \label{fig:Memory}
\end{figure}
\begin{figure}[!ht]
    \centering
    \includegraphics[width=0.8\linewidth]{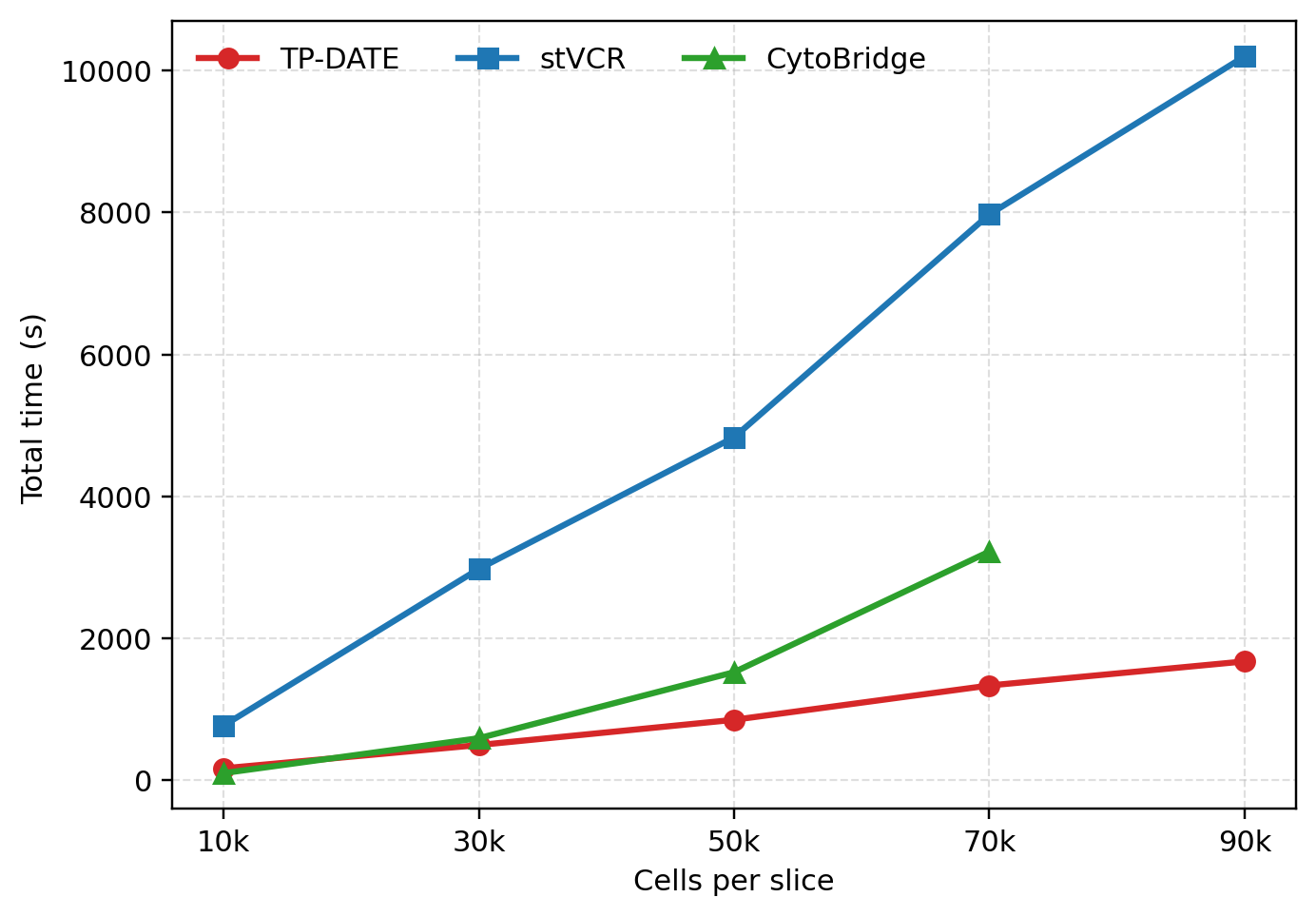}
    \caption{Training time on different size data}
    \label{fig:Time}
\end{figure}

\clearpage
\section{Relations to other works}
\subsection{GW-OT and FGW-OT}
\label{app:GW and FGW}
Consider the Lagrangian
\begin{equation}
\label{eq:Lagrangian}
\mathcal{L}(t,\bm{x},\bm{y},\dot{\bm{x}},\dot{\bm{y}})=\frac{1}{2}\Vert\dot{\bm{x}}\Vert^2+\frac{1}{2}\Vert\dot{\bm{y}}\Vert^2+\lambda|\frac{\mathrm{d}}{\mathrm{d}t}\Vert\bm{x}_t-\bm{y}_t\Vert|^2.
\end{equation}
\subsubsection{Standard static OT}
When \(\lambda=0\), the static cost is
\begin{equation}
\mathcal{A}(\bm{x}_0,\bm{x}_1,\bm{y}_0,\bm{y}_1) =\inf_{\bm{x}_t,\bm{y}_t} \int_0^1\frac{1}{2}\Vert\dot{\bm{x}}_t\Vert^2+\frac{1}{2}\Vert\dot{\bm{y}}_t\Vert^2\mathrm{d}t=\frac{1}{2}\Big(\Vert\bm{x}_1-\bm{x}_0\Vert^2+\Vert\bm{y}_1-\bm{y}_0\Vert^2\Big).
\end{equation}
The corresponding static form is
\begin{equation}
\begin{aligned}
\text{QOT}_S(\mu_0,\mu_1) &=\inf_{\gamma\in\Pi(\mu_0,\mu_1)} \int_{\mathcal{X}^4} \mathcal{A}(\bm{x}_0,\bm{x}_1,\bm{y}_0,\bm{y}_1)\gamma(\bm{x}_0,\bm{x}_1)\gamma(\bm{y}_0,\bm{y}_1)\mathrm{d} \bm{x}_0\mathrm{d}\bm{x}_1\mathrm{d}\bm{y}_0\mathrm{d}\bm{y}_1\\
&=\inf_{\gamma\in\Pi(\mu_0,\mu_1)} \frac{1}{2}\Big(\int_{\mathcal{X}^2}\Vert\bm{x}_1-\bm{x}_0\Vert^2\gamma(\bm{x}_0,\bm{x}_1)\mathrm{d} \bm{x}_0\mathrm{d}\bm{x}_1+\int_{\mathcal{X}^2}\Vert\bm{y}_1-\bm{y}_0\Vert^2\gamma(\bm{y}_0,\bm{y}_1)\mathrm{d}\bm{y}_0\mathrm{d}\bm{y}_1\Big)\\
&=\inf_{\gamma\in\Pi(\mu_0,\mu_1)}\int_{\mathcal{X}^2}\Vert\bm{x}_1-\bm{x}_0\Vert^2\gamma(\bm{x}_0,\bm{x}_1)\mathrm{d} \bm{x}_0\mathrm{d}\bm{x}_1\\
&=\text{OT}(\mu_0,\mu_1).
\end{aligned}
\end{equation}
It is equivalent to the standard static OT.
\subsubsection{Standard GW-OT}
When there are no kinetic terms, which is \(\mathcal{L}=|\frac{\mathrm{d}}{\mathrm{d}t}\Vert\bm{x}_t-\bm{y}_t\Vert|^2\), let \(r_t=\Vert\bm{x}_t-\bm{y}_t\Vert\), the static cost is
\begin{equation}
\label{eq:GW path action}
\mathcal{A}(\bm{x}_0,\bm{x}_1,\bm{y}_0,\bm{y}_1) =\inf_{r_t} \int_0^1|\dot{r}_t|^2\mathrm{d}t\quad\text{s.t.}\quad r_0=\Vert\bm{x}_0-\bm{y}_0\Vert,\quad r_1=\Vert\bm{x}_1-\bm{y}_1\Vert.
\end{equation}
Using the Euler-Lagrange equation, it is easy to show that when \(d\ge 2\), the optimal \(r_t\) is attained by \(r_t=(1-t)r_0+tr_1\). Therefore, the path action is
\begin{equation}
\mathcal{A}(\bm{x}_0,\bm{x}_1,\bm{y}_0,\bm{y}_1)=(r_1-r_0)^2= (\Vert\bm{x}_1-\bm{y}_1\Vert-\Vert\bm{x}_0-\bm{y}_0\Vert)^2.
\end{equation}
The corresponding static form is
\begin{equation}
\begin{aligned}
\text{QOT}_S(\mu_0,\mu_1) &=\inf_{\gamma\in\Pi(\mu_0,\mu_1)} \int_{\mathcal{X}^4} \mathcal{A}(\bm{x}_0,\bm{x}_1,\bm{y}_0,\bm{y}_1)\gamma(\bm{x}_0,\bm{x}_1)\gamma(\bm{y}_0,\bm{y}_1)\mathrm{d} \bm{x}_0\mathrm{d}\bm{x}_1\mathrm{d}\bm{y}_0\mathrm{d}\bm{y}_1\\
&=\inf_{\gamma\in\Pi(\mu_0,\mu_1)} \int_{\mathcal{X}^4} (\Vert\bm{x}_1-\bm{y}_1\Vert-\Vert\bm{x}_0-\bm{y}_0\Vert)^2\gamma(\bm{x}_0,\bm{x}_1)\gamma(\bm{y}_0,\bm{y}_1)\mathrm{d} \bm{x}_0\mathrm{d}\bm{x}_1\mathrm{d}\bm{y}_0\mathrm{d}\bm{y}_1\\
&=\text{GWOT}(\mu_0,\mu_1)
\end{aligned}
\end{equation}
which is equivalent to standard GW-OT \citep{gwot,moscot}. We point out that different to static OT which has a dynamic form, GW-OT does not naturally have an unique dynamic form. Because the static cost (\ref{eq:GW path action}) is defined by only minimizing w.r.t \(r_t\), which can not uniquely determine \(\bm{q}_t\). Intuitively, the travelling pair path action only penalizes the variation of the length of \(\bm{q}_t\), but doesn't penalize the rotation.
\subsubsection{Dynamic analog to FGW-OT}
Previous studies have combined the standard static OT and GW-OT together to obtain a fused GW-OT (FGW-OT) \citep{FGWOT,moscot}, which simultaneously minimizes the kinetic energy and the preserves the local structure. Its cost is defined as a weighted sum of static OT cost and static GW-OT cost.
\begin{equation}
\label{eq:FGW-OT}
C_\text{FGW}(\mu_0,\mu_1)=\alpha C_\text{GW}(\mu_0,\mu_1)+(1-\alpha)C_\text{OT}(\mu_0,\mu_1)
\end{equation}
As discussed above, the Lagrangian (\ref{eq:Lagrangian}) is also a weighted sum of OT and GW-OT up to a scalar, but in a dynamic sense. Hence, intuitively, the corresponding dynamic form can be viewed as a dynamic analog to FGW-OT.

As discussed in \ref{pf:RK path}, the Lagrangian (\ref{eq:Lagrangian}) can be written as \(\Vert\dot{\bm{c}}_t\Vert^2+\frac{1}{4}r_t^2\dot{\theta}_t^2+(\frac{1}{4}+\lambda)\dot{r}_t^2\) under the C-Q decomposition and the polar coordinates on the plane spanned by \(\bm{q}_0,\bm{q}_1\). The first term is the kinetic energy of the barycenter, which is an analog to the standard kinetic term \(\frac{1}{2}\Vert\dot{\bm{x}}\Vert^2+\frac{1}{2}\Vert\dot{\bm{y}}\Vert^2\). The third term exactly leads to the static path action of GW-OT (\ref{eq:GW path action}). The second term further penalizes the variation of \(\theta_t\), i.e. the rotation of \(\bm{q}_t\). Therefore, the Lagrangian (\ref{eq:Lagrangian}) actually leads to a more strict preservation of local structure than FGW-OT.

We also derived the analytic form of the path action in \ref{pf:RK path}.
\begin{equation}
\mathcal{A}(\bm{x}_0,\bm{x}_1,\bm{y}_0,\bm{y}_1)=\Vert\bm{c}_1-\bm{c}_0\Vert^2+(\frac{1}{4}+\lambda)(r_0^2+r_1^2-2r_0r_1\cos k\alpha)
\end{equation}
When \(\lambda=0\), 
\begin{equation}
\begin{aligned}
\mathcal{A}(\bm{x}_0,\bm{x}_1,\bm{y}_0,\bm{y}_1)&=\Vert\bm{c}_1-\bm{c}_0\Vert^2+\frac{1}{4}(r_0^2+r_1^2-2r_0r_1\cos\alpha)\\&=\Vert\bm{c}_1-\bm{c}_0\Vert^2+\frac{1}{4}\Vert\bm{q}_1-\bm{q}_0\Vert^2=\frac{1}{2}\Vert\bm{x}_1-\bm{x}_0\Vert^2+\frac{1}{2}\Vert\bm{y}_1-\bm{y}_0\Vert^2.
\end{aligned}
\end{equation}
It reduces to OT. When \(\lambda\to\infty\),
\begin{equation}
\begin{aligned}
\lambda^{-1}\mathcal{A}(\bm{x}_0,\bm{x}_1,\bm{y}_0,\bm{y}_1)&=\mathcal{O}(\lambda^{-1})+(r_0^2+r_1^2-2r_0r_1\cos\Big(\sqrt{\frac{1}{1+4\lambda}}\alpha\Big))\\
&=\mathcal{O}(\lambda^{-1})+(r_0^2+r_1^2-2r_0r_1(1+\mathcal{O}(\lambda^{-1})))\\
&=\mathcal{O}(\lambda^{-1})+(r_0^2+r_1^2-2r_0r_1)\\
&=(r_1-r_0)^2+\mathcal{O}(\lambda^{-1})\to(\Vert\bm{x}_1-\bm{y}_1\Vert-\Vert\bm{x}_0-\bm{y}_0\Vert)^2
\end{aligned}
\end{equation}
It reduces to GW-OT. Therefore, for general \(\lambda\in(0,+\infty)\), it can be viewed as a dynamic fusion of OT and GW-OT rather than FGW-OT, which is just a static fusion.

\subsubsection{The 1D case}
We need \(d\ge2\) (\(d\) is the dimension of the space) to have (\ref{eq:GW path action}) be equivalent to the static GW-OT cost. To be self-contained, we briefly discuss the 1D case \((d=1)\) here. When \(d=1\), the Lagrangian \(\mathcal{L}=(\frac{\mathrm{d}}{\mathrm{d}t}|x_t-y_t|)^2=\Big(\frac{\mathrm{d}}{\mathrm{d}t}(x_t-y_t)\Big)^2\). Let \(q_1=x_1-y_1,q_0=x_0-y_0\), the path action is
\begin{equation}
\mathcal{A}(x_0,x_1,y_0,y_1) =\inf_{q_t} \int_0^1|\dot{q}_t|^2\mathrm{d}t\quad\text{s.t.}\quad q_0=x_0-y_0,\quad q_1=x_1-y_1.
\end{equation}
The optimal path is \(q_t=(1-t)q_0+tq_1\), instead of \(r_t=(1-t)r_0+tr_1\). The path action is therefore
\begin{equation}
\mathcal{A}(x_0,x_1,y_0,y_1) = (q_1-q_0)^2=(x_1-y_1-x_0+y_0)^2
\end{equation}
instead of \((|x_1-y_1|-|x_0-y_0|)^2\). Thus, when \(d=1\), our path action does not reduce to the static cost of GW-OT. In fact, the two cost are equal if and only if \((x_1-y_1)(x_0-y_0)\ge0\). Geometrically, in a high dimensional space (\(d\ge2\)), one can continuously deform the configuration and reverse the relative positions of two points while keeping the magnitude of their relative displacement equal to the linear interpolation between its endpoint values. But it is impossible in one dimension. Reversing the relative positions of two points necessarily forces their relative displacement to vanish at some intermediate time, so its magnitude cannot follow the linear interpolation between the initial and final values.

\subsection{IGW-OT}
\label{app:IGW-OT}
Developing dynamic formulations of GW-like OT has long attracted considerable attention. A recent successful attempt is the inner product GW-OT (IGW-OT), proposed by \citep{IGWOT}, which replaces the cost based on discrepancies between pairwise distances in GW-OT with one based on discrepancies between inner products. 

\begin{equation}
\label{eq:IGW-OT}
\text{IGW-OT}^2_S(\mu_0,\mu_1) 
=\inf_{\gamma\in\Pi(\mu_0,\mu_1)} \int_{\mathcal{X}^4} (\langle\bm{x}_1,\bm{y}_1\rangle-\langle\bm{x}_0,\bm{y}_0\rangle)^2\gamma(\bm{x}_0,\bm{x}_1)\gamma(\bm{y}_0,\bm{y}_1)\mathrm{d} \bm{x}_0\mathrm{d}\bm{x}_1\mathrm{d}\bm{y}_0\mathrm{d}\bm{y}_1.
\end{equation}

They developed a dynamic formulation for the above static QOT problem, denoted as IGW-OT. Intuitively, IGW-OT likewise aims to preserve spatial structure as much as possible throughout the transport process. Different to TP-DATE, they focused on the gradient flow and Riemannian structure of IGW-OT.

We point out that, with an appropriate choice of Lagrangian within the TP-DATE framework, IGW-OT can also be recovered as a special case of the dynamic QOT formulation introduced in our work. Take
\begin{equation}
\mathcal{L}(t,\bm{x}_t,\bm{y}_t,\dot{\bm{x}}_t,\dot{\bm{y}}_t)=|\frac{\mathrm{d}}{\mathrm{d}t}\langle\bm{x}_t,\bm{y}_t\rangle|^2=|\langle\dot{\bm{x}}_t,\bm{y}_t\rangle+\langle\bm{x}_t,\dot{\bm{y}}_t\rangle|^2
\end{equation}
The path action is
\begin{equation}
\label{eq:IGW path action}
\mathcal{A}(\bm{x}_0,\bm{x}_1,\bm{y}_0,\bm{y}_1) =\inf_{r_t} \int_0^1|\dot{r}_t|^2\mathrm{d}t\quad\text{s.t.}\quad r_0=\langle\bm{x}_0,\bm{y}_0\rangle,\quad r_1=\langle\bm{x}_1,\bm{y}_1\rangle.
\end{equation}
Using the Euler-Lagrange equation, it is easy to show that when \(d\ge 2\), the optimal \(r_t\) is attained by \(r_t=(1-t)r_0+tr_1\). Therefore, the path action is
\begin{equation}
\mathcal{A}(\bm{x}_0,\bm{x}_1,\bm{y}_0,\bm{y}_1)=(r_1-r_0)^2=(\langle\bm{x}_1,\bm{y}_1\rangle-\langle\bm{x}_0,\bm{y}_0\rangle)^2.
\end{equation}
The corresponding static form is
\begin{equation}
\begin{aligned}
\text{QOT}_S(\mu_0,\mu_1) &=\inf_{\gamma\in\Pi(\mu_0,\mu_1)} \int_{\mathcal{X}^4} \mathcal{A}(\bm{x}_0,\bm{x}_1,\bm{y}_0,\bm{y}_1)\gamma(\bm{x}_0,\bm{x}_1)\gamma(\bm{y}_0,\bm{y}_1)\mathrm{d} \bm{x}_0\mathrm{d}\bm{x}_1\mathrm{d}\bm{y}_0\mathrm{d}\bm{y}_1\\
&=\inf_{\gamma\in\Pi(\mu_0,\mu_1)} \int_{\mathcal{X}^4} (\langle\bm{x}_1,\bm{y}_1\rangle-\langle\bm{x}_0,\bm{y}_0\rangle)^2\gamma(\bm{x}_0,\bm{x}_1)\gamma(\bm{y}_0,\bm{y}_1)\mathrm{d} \bm{x}_0\mathrm{d}\bm{x}_1\mathrm{d}\bm{y}_0\mathrm{d}\bm{y}_1\\
&=\text{IGW-OT}^2_S(\mu_0,\mu_1)
\end{aligned}
\end{equation}
The convexity of Lagrangian is also satisfied. To show that, let 
\begin{equation}
\bm{v}_t=\begin{pmatrix}
\dot{\bm{x}}_t\\\dot{\bm{y}}_t
\end{pmatrix}
\qquad
\bm{z}_t=\begin{pmatrix}
\bm{y}_t\\\bm{x}_t
\end{pmatrix}
\end{equation}
we have
\begin{equation}
|\langle\dot{\bm{x}}_t,\bm{y}_t\rangle+\langle\bm{x}_t,\dot{\bm{y}}_t\rangle|^2=|\bm{z}_t\cdot\bm{v}_t|^2=\bm{v}_t^{\mathrm{T}}(\bm{z}_t\bm{z}_t^{\mathrm{T}})\bm{v}_t.
\end{equation}
Therefore, the Lagrangian is convex w.r.t \(\bm{v}_t\), hence \((\dot{\bm{x}}_t,\dot{\bm{y}}_t)\). Thus, the static IGW objective is recovered as a special case of our path action construction, while the induced dynamic flow need not coincide with their intrinsic IGW flow (for mathematical details, please refer to their paper). From this perspective, our TP-DATE framework can be viewed as a more general dynamic QOT framework.

\subsubsection{The 1D case}
In the IGW-OT case, we also need \(d\ge2\) to make \(r_t=(1-t)r_0+tr_1\) attainable. We give an explicit construction here. Let \(\bm{p}_t=\frac{\bm{x}_t+\bm{y}_t}{\sqrt{2}},\bm{q}_t=\frac{\bm{x}_t-\bm{y}_t}{\sqrt{2}}\), we have \(\langle\bm{x}_t,\bm{y}_t\rangle=\Vert\bm{p}_t\Vert^2-\Vert\bm{q}_t\Vert^2
\). When \(d\ge2\), it is easy to construct \(\bm{p}_t,\bm{q}_t\) such that \(\Vert\bm{p}_t\Vert^2=(1-t)\Vert\bm{p}_0\Vert^2+t\Vert\bm{p}_1\Vert^2,\Vert\bm{q}_t\Vert^2=(1-t)\Vert\bm{q}_0\Vert^2+t\Vert\bm{q}_1\Vert^2\), hence \(\langle\bm{x}_t,\bm{y}_t\rangle=(1-t)\langle\bm{x}_0,\bm{y}_0\rangle+t\langle\bm{x}_1,\bm{y}_1\rangle\).

When \(d=1\), consider \(x_0=1,y_0=-1,x_1=-1,y_1=1\). By direct calculation, we have \(x_0y_0=x_1y_1=-1\). But, it is impossible to construct a travelling pair \((x_t,y_t)\) such that \(x_ty_t\equiv-1\), since there must exist some \(t\) such that \(x_t=0\), hence \(x_ty_t=0\neq-1\). The reason is similar to the GW-OT case. A reflection can be continuously realized by a rigid body rotation in high dimensional space \(d\ge2\), but can never be realized by a continuous rigid transformation in 1D spaces.

When \(d=1\), consider \(\mathcal{L}=\Big(\frac{\mathrm{d}}{\mathrm{d}t}(x_ty_t)\Big)^2\), by direct calculation, one can easily check that the corresponding static cost is
\begin{equation}
\mathcal{A}(x_0,x_1,y_0,y_1)=\left\{
\begin{aligned}
&(|x_0y_0|+|x_1y_1|)^2,\quad x_0x_1<0,y_0y_1<0\\
&(x_1y_1-x_0y_0)^2,\quad \text{otherwise}
\end{aligned}
\right .
\end{equation}
The infimum is attained when \(x_0y_0x_1y_1\neq0\) or \(x_0y_0=x_1y_1=0\), but may not be attainable when \(x_0y_0=0,x_1y_1\neq0\) or \(x_0y_0\neq0,x_1y_1=0\).

\subsection{OT-CFM}
\citep{cfm_tong} has proposed OT-CFM, a flow matching framework for solving standard dynamic OT problem. It uses the optimal OT coupling under squared Euclidean distance cost, and the displacement conditional path for flow matching training. The learned flow can be proved to solve the standard dynamic OT problem with kinetic energy cost.
\begin{equation}
\begin{aligned}
\label{eq:ot-cfm action}
&\text{OT}(\mu_0,\mu_1) =\inf_{\rho,\bm{u}} \int_0^1\int_{\mathcal{X}}  \, \frac{1}{2}\Vert \bm{u}(\bm{x},t)\Vert_2^2\rho_t(\bm{x})\mathrm{d} \bm{x}\mathrm{d}t \\
&\text{s.t.} \ \ \ \ \ \ \ \partial_t\rho+\nabla_{\bm{x}}\cdot(\rho \bm{u})=0,\ \rho_0=\mu_0,\ \rho_1=\mu_1
\end{aligned}
\end{equation}

Inspired by this pioneering work, we proposed a new travelling pair flow matching framework which can solve the dynamic QOT problem. When choosing the Lagrangian in our dynamic QOT framework as
\begin{equation}
\mathcal{L}(t,\bm{x},\bm{y},\dot{\bm{x}},\dot{\bm{y}})=\frac{1}{2}\Vert\dot{\bm{x}}\Vert^2+\frac{1}{2}\Vert\dot{\bm{y}}\Vert^2
\end{equation}
the corresponding dynamic QOT reduces to (\ref{eq:ot-cfm action}). Therefore, TP-DATE reduces to OT-CFM under this setting.

\subsection{Static structured transport in spatial transcriptomics}
PASTE \citep{PASTE} and PASTE2 \citep{PASTE2} formulate spatial transcriptomic slice alignment using FGW and partial-FGW, respectively, and primarily target spatial alignment, integration, and 3D reconstruction. MOSCOT \citep{moscot} provides a unified OT framework for mapping cells across time and space, including GW- and FGW-type formulations for spatial and spatiotemporal correspondence. SOCS \citep{SOCS} further targets time-series spatial transcriptomics by estimating an ancestor--descendant transport matrix between two observed time points while preserving gene-expression, geometric, and contiguous-structure information. Despite their different application settings, these methods operate at the level of couplings between observed datasets and do not model intermediate time measure paths or dynamical vector fields. In this sense, they perform static endpoint alignment. In contrast, TP-DATE goes beyond endpoint alignment by reconstructing continuous-time dynamics, and thus constitutes a dynamic trajectory inference framework.

\subsection{GENOT}
\citep{klein2024genot} also combines flow matching with GW-OT, but with a fundamentally different purpose. GENOT is a neural OT solver that uses flow matching to learn couplings associated with static entropic OT, GW-OT, or FGW-OT. In contrast, TP-DATE does not use flow matching to solve the QOT coupling. It uses flow matching to find the dynamic QOT probability flow. The flow in GENOT serves as a parameterization of a static transport plan, while the flow in TP-DATE represents the reconstructed dynamics between snapshots. The two approaches therefore address different problems.

\subsection{NicheFlow}
NicheFlow \citep{nicheflow} advances spatial trajectory inference by modeling local cellular microenvironments as structured point clouds, jointly capturing changes in spatial coordinates and gene expression states. It combines OT-based source--target matching with Variational Flow Matching (VFM) to generate future microenvironments conditioned on observed source niches. However, its flow evolves from noise to the target state conditioned on the source, so the internal flow time represents a generative process rather than the continuous time biological dynamics between two tissue states. In contrast, TP-DATE directly models source-to-target continuous time dynamics and imposes structure preservation at the dynamical level through dynamic QOT.

\subsection{ContextFlow}
ContextFlow \citep{rathod2026contextflowcontextawareflowmatching}
is a recently proposed context-aware generative modeling approach for transcriptomic snapshot interpolation. It incorporates spatial information and cell--cell communication signals into a static OT cost, thereby obtaining couplings that account for these contextual factors. These couplings are then used in conditional flow matching to reconstruct the dynamics. From this perspective, ContextFlow can be viewed as introducing a class of context-aware static OT formulations. In contrast, TP-DATE incorporates spatial structure directly at the dynamical level rather than indirectly through a static coupling. Viewed this way, TP-DATE extends the context-aware principle to dynamic settings through the framework of dynamic QOT.

\subsection{stVCR}
stVCR \citep{peng2026stvcr} models spatiotemporal dynamics of  
single cells with Wasserstein-Fisher-Rao dynamics (WFR). WFR can be viewed as an unbalanced extension of standard OT. It extends OT from transporting probability densities to transporting unbalanced measures. 
\begin{equation}
\begin{aligned}
&\text{WFR}_{\delta}^2(\mu_0,\mu_1) =\inf_{\rho,\bm{u}} \int_0^1\int_{\mathcal{X}}  \, \frac{1}{2}[\Vert \bm{u}(\bm{x},t)\Vert_2^2+\delta^2g_t^2(\bm{x})]\rho_t(\bm{x})\mathrm{d} \bm{x}\mathrm{d}t \\
&\text{s.t.} \ \ \ \ \ \ \ \ \ \ \ \ \ \partial_t\rho+\nabla_{\bm{x}}\cdot(\rho \bm{u})=g\rho,\ \rho_0=\mu_0,\ \rho_1=\mu_1
\end{aligned}
\end{equation}
It models the cell proliferation and apoptosis by a growth rate term \(g_t(\bm{x})\). When applied on scRNA-seq data without spatial information, this WFR problem can be solved by unbalanced flow matching \citep{wfr_fm}. On spatial transcriptomics data, to preserve the spatial structure, an optional spatial structure preserving loss is added. To account for possible slight misalignment between the coordinate systems of spatial transcriptomics slices, stVCR further introduces a rigid body transformation invariant OT formulation. Before each OT training step, it first solves an optimal rigid body transformation to align the coordinate systems. Both techniques make it unsolvable for flow matching based methods. Therefore, stVCR used NeuralODE method to solve the corresponding OT problem. 

In contrast, TP-DATE naturally preserves the spatial structure and addresses slight misalignment by adding interaction term \(\Phi\), and is simulation-free. Although TP-DATE is more efficient than stVCR in modeling spatial structure preservation, it does not currently account for cell proliferation and apoptosis as stVCR does. This limitation points to an important direction for future work: extending the dynamic QOT framework to the unbalanced setting.

\subsection{CytoBridge}
CytoBridge \citep{zhang2025cytobridge} is a NeuralODE framework for learning interactions and mean-field interactive dynamics from snapshots. They focuses on the unbalanced mean-field Schr\"odinger bridge problem (UMFSB).
\begin{equation}
\begin{aligned}
&\text{UMFSB}(\mu_0,\mu_1) =\inf_{\rho,\bm{u}} \int_0^1\int_{\mathcal{X}}  \, \frac{1}{2}[\Vert \bm{u}(\bm{x},t)\Vert_2^2+g_t^2(\bm{x})]\rho_t(\bm{x})\mathrm{d} \bm{x}\mathrm{d}t \\
&\text{s.t.} \  \ \ \partial_t\rho(\bm{x})+\nabla_{\bm{x}}\cdot[\rho(\bm{x}) (\bm{u}_t(\bm{x})-\int k(\bm{x},\bm{y})\nabla_{\bm{x}}\Phi(\bm{x}-\bm{y})\rho_t(\bm{y})\mathrm{d}\bm{y})]=g\rho\\
&\ \ \ \ \ \ \ \ \ \ \ \ \ \ \ \ \ \ \ \ \ \ \ \ \ \ \ \ \ \ \ \ \ \ \ \ \ \ \ \ \rho_0=\mu_0,\ \rho_1=\mu_1
\end{aligned}
\end{equation}
where \(\Phi\) is some interaction potential for modeling cell-cell interactions, and \(k\) is some gate function to control the intensity of the interaction. The key distinction between CytoBridge and TP-DATE is where the interaction term is introduced. CytoBridge incorporates the interaction term directly into the dynamical equation as part of the constraint, whereas TP-DATE retains the standard continuity equation constraint on the two-particle space and instead places the interaction term in the action functional.

One advantage of CytoBridge is that, because the interaction term is part of the dynamics to be learned, the interaction itself can be inferred from data through a NeuralODE. This is feasible for NeuralODE based methods. In principle, essentially any interaction term can be parameterized by a neural network and learned by minimizing the loss. For flow matching based methods, however, this is generally much more challenging, since the conditional paths determined by the interaction term must typically be computed in advance. Whether techniques from inverse problems can be used to overcome this limitation and enable flow matching methods to learn such interactions from the data  is therefore a very interesting direction for future research.

Another interesting question is whether CytoBridge can be made simulation-free. From TrajectoryNet \citep{trajectorynet} to OT-CFM \citep{cfm_tong}, Diffusion Schr\"odinger bridge \citep{bortoli2021diffusion,shi2024diffusion} to \(\text{SF}^2\text{M}\) \citep{tong2023simulationfree}, stVCR \citep{peng2026stvcr} to WFR-FM \cite{wfr_fm}, previous works have already built simulation-free flow matching based algorithms for OT, Schr\"odinger bridge, and WFR. There is a trend of converting simulation-based or iteration-based Neural OT solvers to simulation-free solvers by flow matching. However, the success of these algorithms rely on an important mathematical property of the underlying OT problems: the dynamical constraint is linear, or equivalently, local, with respect to the measure flow \(\rho_t\). This property is precisely what allows the marginalization theorem to hold in these settings. Once a CytoBridge style interaction term is introduced directly into the dynamical constraint, this linearity or locality is lost. One can readily verify that, for the CytoBridge constraint equation, the usual marginalization argument no longer applies. This limitation is also reflected in the independence of conditional paths in standard flow matching. Once the initial and terminal states are fixed, each conditional path is determined independently of the others, with no coupling between different paths. As a result, standard conditional path constructions are not naturally equipped to represent interactions between particles. From this perspective, extending CytoBridge to a fully simulation-free flow matching formulation appears to be fundamentally challenging.

TP-DATE's travelling-pair flow matching may provide a possible starting point for addressing this challenge. By considering the problem on the two-particle space, we simultaneously enable interactions between conditional paths while preserving the linearity of the dynamical constraint. Moreover, through velocity decomposition, the resulting single-particle dynamics can be interpreted as an approximation to mean-field dynamics, in which interactions at the population level emerge from underlying pairwise interactions. Although TP-DATE was not originally designed for this purpose, a promising direction for future work is to build on travelling-pair flow matching and investigate whether formulating the problem on the two-particle space can enable efficient flow matching based modeling of nonlocal dynamics.

\section{Velocity decomposition}
\label{velocity decomposition gauges}
In general cases, one may want to only have a velocity decomposition at conditional velocity level \(\bm{u}^1_t(\bm{x},\bm{y}|\bm{z},\bm{z}')=\bm{b}_t(\bm{x}|\bm{z},\bm{z}')+\bm{f}_t(\bm{x},\bm{y}|\bm{z},\bm{z}').\) In this case, we want to regress the marginal velocities via flow matching.
\begin{equation}
\begin{aligned}
\mathcal{L}_b(\bm{\theta})&=
\mathbb{E}_{t\sim\mathcal{U}[0,1], (\bm{z},\bm{z}')\sim q(\bm{z},\bm{z}'), (\bm{x},\bm{y})\sim \pi_t(\bm{x},\bm{y}\vert \bm{z},\bm{z}')}\left\| \bm{\bm{b}_{\theta}}(\bm{x},t) - \bm{b}_t(\bm{x}\vert\bm{z},\bm{z}') \right\|_2^2\\
\mathcal{L}_f(\bm{\phi})&=
\mathbb{E}_{t\sim\mathcal{U}[0,1], (\bm{z},\bm{z}')\sim q(\bm{z},\bm{z}'), (\bm{x},\bm{y})\sim \pi_t(\bm{x},\bm{y}\vert \bm{z},\bm{z}')}\left\| \bm{\bm{f}_{\phi}}(\bm{x},\bm{y},t) - \bm{f}_t(\bm{x},\bm{y}\vert\bm{z},\bm{z}') \right\|_2^2
\end{aligned}
\end{equation}
Similarly to theorem \ref{thm:bf loss}, we still have
\begin{equation}
\begin{aligned}
\mathcal{L}_b(\bm{\theta})&=
\mathbb{E}_{t\sim\mathcal{U}[0,1], (\bm{x},\bm{y})\sim \pi_t(\bm{x},\bm{y})}\left\| \bm{\bm{b}_{\theta}}(\bm{x},t) - \bm{b}_t(\bm{x}) \right\|_2^2+C_1\\&=\mathbb{E}_{t\sim\mathcal{U}[0,1], \bm{x}\sim \rho_t(\bm{x})}\left\| \bm{\bm{b}_{\theta}}(\bm{x},t) - \bm{b}_t(\bm{x}) \right\|_2^2+C_1\\
\mathcal{L}_f(\bm{\phi})&=
\mathbb{E}_{t\sim\mathcal{U}[0,1], (\bm{x},\bm{y})\sim \pi_t(\bm{x},\bm{y})}\left\| \bm{\bm{f}_{\phi}}(\bm{x},\bm{y},t) - \bm{f}_t(\bm{x},\bm{y}) \right\|_2^2+C_2
\end{aligned}
\end{equation}
where \(C_1,C_2\) are constants independent of \(\bm{\theta},\bm{\phi}\). Therefore, the minimizer is \(\bm{b}_t(\bm{x})\) and \(\bm{f}_t(\bm{x},\bm{y})\). The problem here is that the conditional velocity decomposition may not lead to a marginal velocity decomposition, i.e. \(\bm{u}^1_t(\bm{x},\bm{y})=\bm{b}_t(\bm{x})+\bm{f}_t(\bm{x},\bm{y})\) does not naturally hold. By marginalization theorems, we have
\begin{equation}
\left\{
\begin{aligned}
&\bm{u}^1_t(\bm{x},\bm{y})=\int\bm{u}^1_t(\bm{x},\bm{y}|\bm{z},\bm{z}')\frac{\pi_t(\bm{x},\bm{y}|\bm{z},\bm{z}')q(\bm{z},\bm{z}')}{\pi_t(\bm{x},\bm{y})}\mathrm{d}\bm{z}\mathrm{d}\bm{z}'\\
&\bm{b}_t(\bm{x})=\int\Big(\int\bm{b}_t(\bm{x}|\bm{z},\bm{z}')\frac{\pi_t(\bm{x},\bm{y}|\bm{z},\bm{z}')q(\bm{z},\bm{z}')}{\pi_t(\bm{x},\bm{y})}\mathrm{d}\bm{z}\mathrm{d}\bm{z}'\Big)\frac{\pi_t(\bm{x},\bm{y})}{\rho_t(\bm{x})}\mathrm{d}\bm{y}\\
&\bm{f}_t(\bm{x},\bm{y})=\int\bm{f}_t(\bm{x},\bm{y}|\bm{z},\bm{z}')\frac{\pi_t(\bm{x},\bm{y}|\bm{z},\bm{z}')q(\bm{z},\bm{z}')}{\pi_t(\bm{x},\bm{y})}\mathrm{d}\bm{z}\mathrm{d}\bm{z}'
\end{aligned}
\right .
\end{equation}
The conditional velocity decomposition yields 
\begin{equation}
\begin{aligned}
&\mathbb{E}[\bm{b}_t(\bm{X}|\bm{z},\bm{z}')|\bm{X}=\bm{x},\bm{Y}=\bm{y}]\\=&\int\bm{b}_t(\bm{x}|\bm{z},\bm{z}')\frac{\pi_t(\bm{x},\bm{y}|\bm{z},\bm{z}')q(\bm{z},\bm{z}')}{\pi_t(\bm{x},\bm{y})}\mathrm{d}\bm{z}\mathrm{d}\bm{z}'\\=&\int[\bm{u}^1_t(\bm{x},\bm{y}|\bm{z},\bm{z}')-\bm{f}_t(\bm{x},\bm{y}|\bm{z},\bm{z}')]\frac{\pi_t(\bm{x},\bm{y}|\bm{z},\bm{z}')q(\bm{z},\bm{z}')}{\pi_t(\bm{x},\bm{y})}\mathrm{d}\bm{z}\mathrm{d}\bm{z}'\\
=&\bm{u}^1_t(\bm{x},\bm{y})-\bm{f}_t(\bm{x},\bm{y})
\end{aligned}
\end{equation}
By definition, 
\begin{equation}
\bm{b}_t(\bm{x})=\mathbb{E}[\bm{b}_t(\bm{X}|\bm{z},\bm{z}')|\bm{X}=\bm{x}]
\end{equation}
Therefore, the necessary and sufficient condition of marginal velocity decomposition is
\begin{equation}
\mathbb{E}[\bm{b}_t(\bm{X}|\bm{z},\bm{z}')|\bm{X}=\bm{x},\bm{Y}=\bm{y}]=\mathbb{E}[\bm{b}_t(\bm{X}|\bm{z},\bm{z}')|\bm{X}=\bm{x}]
\end{equation}
Intuitively, the condition requires that, once the current state of the particle is given, the state of the interacting particle provides no additional information about the conditional mean of the self-driven velocity. It is a conditional mean-independence condition, which is weaker than conditional independence. The simplest example is what we introduced in the main text \(\bm{b}_t(\bm{x},\bm{y}|\bm{z},\bm{z}')=\bm{b}_t(\bm{x})\). It is easy to check that it satisfies the condition. Other kind of decompositions have to be designed carefully to satisfy this condition, in order to be learned by flow matching.

\section{The BB-form}
\label{app:BB-form}
For convenience, we restate the three forms below. The static form is 
\begin{equation}
\label{eq:static QOT E}
\begin{aligned}
&\text{QOT}_S(\mu_0,\mu_1) =\inf_{\gamma\in\Pi(\mu_0,\mu_1)} \int_{\mathcal{X}^4} \mathcal{A}(\bm{z},\bm{z}')\gamma(\bm{z})\gamma(\bm{z}')\mathrm{d} \bm{z}\mathrm{d} \bm{z}'.\\
\end{aligned}
\end{equation}
The dynamic form is
\begin{equation}
\label{eq:dynamic QOT E}
\begin{aligned}
&\text{QOT}_{D}(\mu_0,\mu_1) = \\&\inf\int_0^1\int\int_{\mathcal{X}^2} \mathcal{L}(t,\bm{x},\bm{y},u^1_t(\bm{x},\bm{y}|\bm{z},\bm{z}'),u^2_t(\bm{x},\bm{y}|\bm{z},\bm{z}'))\pi_t(\bm{x},\bm{y}|\bm{z},\bm{z}')\gamma(\bm{z})\gamma(\bm{z}')\mathrm{d} \bm{x}\mathrm{d} \bm{y}\mathrm{d} \bm{z}\mathrm{d} \bm{z}'\mathrm{d}t.
\end{aligned}
\end{equation}
The BB-form is
\begin{equation}
\label{eq:BB QOT E}
\begin{aligned}
&\text{QOT}_{BB}(\mu_0,\mu_1) =\inf_{\pi,\bm{u}^1,\bm{u}^2} \int_0^1\int_{\mathcal{X}^2} \mathcal{L}(t,\bm{x},\bm{y},u^1_t(\bm{x},\bm{y}),u^2_t(\bm{x},\bm{y}))\pi_t(\bm{x},\bm{y})\mathrm{d} \bm{x}\mathrm{d} \bm{y}\mathrm{d}t\\
&\partial_t\pi_t(\bm{x},\bm{y})+\nabla_{\bm{x}}\cdot(\pi_t(\bm{x},\bm{y})u^1_t(\bm{x},\bm{y}))+\nabla_{\bm{y}}\cdot(\pi_t(\bm{x},\bm{y})u^2_t(\bm{x},\bm{y}))=0.
\end{aligned}
\end{equation}
\subsection{Difficulty of the equivalence between static QOT and BB-form}
Inspired by the pioneering work of \citep{benamou2000computational}, one may expect to have an equivalence between the static form and the BB-form also for QOT. We explain why it can not. Let \(\bm{z}=(\bm{x},\bm{y})\), and \(\bm{u}_t(\bm{z})=\begin{pmatrix}
    u^1_t(\bm{x},\bm{y})\\
    u^2_t(\bm{x},\bm{y})
\end{pmatrix}\), the BB-form can be reformulated as
\begin{equation}
\begin{aligned}
&\text{QOT}_{BB}(\mu_0,\mu_1) =\inf_{\pi,\bm{u}} \int_0^1\int_{\mathcal{X}^2} \mathcal{L}(t,\bm{z},u_t(\bm{z}))\pi_t(\bm{z})\mathrm{d} \bm{z}\mathrm{d}t\\
&\text{s.t.}\ \ \ \ \ \ \ \ \ \ \ \ \partial_t\pi_t(\bm{z})+\nabla_{\bm{z}}\cdot(\pi_t(\bm{z})u_t(\bm{z}))=0.
\end{aligned}
\end{equation}
which is just the standard OT on the two-particle space \(\mathcal{X}^2\). When choosing \(\mathcal{L}(t,\bm{z},u_t(\bm{z}))=\frac{1}{2}\Vert\dot{\bm{z}}\Vert^2\), \citep{benamou2000computational} proved that it is equivalent to the static OT on the two-particle space
\begin{equation}
\label{eq:static OT E}
\begin{aligned}
&\text{OT}(\mu_0,\mu_1) =\inf_{q\in\Pi(\mu_0\otimes\mu_0,\mu_1\otimes\mu_1)} \int_{\mathcal{X}^4} \mathcal{A}(\bm{z},\bm{z}')q(\bm{z},\bm{z}')\mathrm{d} \bm{z}\mathrm{d} \bm{z}'.\\
\end{aligned}
\end{equation}
This static form is slightly different to the static QOT. In static QOT (\ref{eq:static QOT E}), we restrict the coupling \(q\) to a smaller set
\begin{equation}
\Pi^2(\mu_0,\mu_1)\triangleq\{q=\gamma\otimes\gamma\vert\gamma\in\Pi(\mu_0,\mu_1)\}\subset\Pi(\mu_0\otimes\mu_0,\mu_1\otimes\mu_1).
\end{equation}
From this perspective, the static QOT on \(\mathcal{X}\) can be viewed as kind of special OT on \(\mathcal{X}^2\) with more strict coupling constraint. Therefore, it is naturally to have only \(\text{QOT}_S(\mu_0,\mu_1)\ge\text{QOT}_{BB}(\mu_0,\mu_1)\), and it is essentially hard to improve it.

\subsection{Difficulty of the equivalence between static QOT and the constrained BB-form}
One may naturally ask whether the BB-form can also be equipped with additional constraints so that it becomes compatible with the extra coupling constraint imposed in static QOT. The answer is also no. 

By superposition principle \citep{ambrosio2008gradient,lisini2007characterization}, for any admissible probability path \(\pi_t(\bm{x},\bm{y})\) satisfying the continuity equation \(\partial_t\pi_t(\bm{x},\bm{y})+\nabla_{\bm{x}}\cdot(\pi_t(\bm{x},\bm{y})\bm{u}_t^1(\bm{x},\bm{y}))+\nabla_{\bm{y}}\cdot(\pi_t(\bm{x},\bm{y})\bm{u}_t^2(\bm{x},\bm{y}))=0\), the superposition principle states that there exists a probability measure \(\eta\) over the path space \(C([0,1],\mathcal{X}^2)\) such that for a path \((\bm{x}_t,\bm{y}_t)\), \(\dot{\bm{x}_t}=\bm{u}^1_t(\bm{x}_t,\bm{y}_t),\dot{\bm{y}_t}=\bm{u}^2_t(\bm{x}_t,\bm{y}_t)\) holds \(\eta-a.e.\), and the projection of \(\eta\) at time \(t\) is \(\pi_t\). That means, for path \(\Gamma\) and evaluation mapping \(e_t: e_t(\Gamma)=\Gamma_t\), we have \((e_t)_\#\eta=\pi_t\). Intuitively, it says that a probability path generated by a continuity equation can be viewed as a superposition of several Dirac-Dirac paths which in this paper we also called the conditional paths. 

To be compatible to the \(\gamma\otimes\gamma\) structure of static QOT, the condition 
\begin{equation}
\exists \gamma\in\Pi(\mu_0,\mu_1),\quad\text{s.t.}\quad\pi_{0,1}\triangleq(e_0,e_1)_\#\eta=\gamma\otimes\gamma
\end{equation}
naturally follows. We call it the QOT compatibility condition. With this condition, we define the QOT compatible BB-form
\begin{equation}
\label{eq:QOT compatible BB-form}
\begin{aligned}
&\text{QOT}_C(\mu_0,\mu_1) =\inf_{\pi,\bm{u}^1,\bm{u}^2} \int_0^1\int_{\mathcal{X}^2} \mathcal{L}(t,\bm{x},\bm{y},u^1_t(\bm{x},\bm{y}),u^2_t(\bm{x},\bm{y}))\pi_t(\bm{x},\bm{y})\mathrm{d} \bm{x}\mathrm{d} \bm{y}\mathrm{d}t\\
&\text{s.t.}\quad \exists\gamma\in\Pi(\mu_0,\mu_1),\quad \pi_{0,1}=\gamma\otimes\gamma \\
&\partial_t\pi_t(\bm{x},\bm{y})+\nabla_{\bm{x}}\cdot(\pi_t(\bm{x},\bm{y})u^1_t(\bm{x},\bm{y}))+\nabla_{\bm{y}}\cdot(\pi_t(\bm{x},\bm{y})u^2_t(\bm{x},\bm{y}))=0.
\end{aligned}
\end{equation}
With this condition, one can prove the inverse inequality \(\text{QOT}_S(\mu_0,\mu_1)\le\text{QOT}_{C}(\mu_0,\mu_1)\).
\begin{proof}
For any conditional path \((\bm{x}_{[0,1]},\bm{y}_{[0,1]})\), let \(\bm{z}=(\bm{x}_0,\bm{x}_1),\bm{z}'=(\bm{y}_0,\bm{y}_1)\), we have
\begin{equation}
\label{eq:def}
\int_0^1\mathcal{L}(t,\bm{x}_t,\bm{y}_t,\dot{\bm{x}}_t,\dot{\bm{y}}_t)\mathrm{d}t\ge\mathcal{A}(\bm{z},\bm{z}')    
\end{equation}
by definition. Now consider the projection of \(\eta\) on the two time points \(\{0,1\}\). Since \((e_0,e_1)_\#\eta=\pi_{0,1}\), there exists a coupling \(\gamma\in\Pi(\mu_0,\mu_1)\) such that \((e_0,e_1)_\#\eta=\gamma\otimes\gamma\). By (\ref{eq:def}), we have
\begin{equation}
\begin{aligned}
&\int\mathcal{L}(t,\bm{x},\bm{y},\bm{u}_t^1(\bm{x},\bm{y}),\bm{u}_t^2(\bm{x},\bm{y}))\pi_t(\bm{x},\bm{y})\mathrm{d}\bm{x}\mathrm{d}\bm{y}\mathrm{d}t\\
&=\int\mathcal{L}(t,\bm{x}_t,\bm{y}_t,\bm{u}_t^1(\bm{x}_t,\bm{y}_t),\bm{u}_t^2(\bm{x}_t,\bm{y}_t))\mathrm{d}\eta(\bm{x}_{[0,1]},\bm{y}_{[0,1]})\mathrm{d}t\\
&=\int\mathcal{L}(t,\bm{x}_t,\bm{y}_t,\dot{\bm{x}}_t,\dot{\bm{y}}_t)\mathrm{d}\eta(\bm{x}_{[0,1]},\bm{y}_{[0,1]})\mathrm{d}t\\
&=\int\Big(\int_0^1\mathcal{L}(t,\bm{x}_t,\bm{y}_t,\dot{\bm{x}}_t,\dot{\bm{y}}_t)\mathrm{d}t\Big)\mathrm{d}\eta(\bm{x}_{[0,1]},\bm{y}_{[0,1]})\\
&=\int\Big(\int_0^1\mathcal{L}(t,\bm{x}_t,\bm{y}_t,\dot{\bm{x}}_t,\dot{\bm{y}}_t)\mathrm{d}t\Big)\mathrm{d}(e_0,e_1)_\#\eta(\bm{x}_{[0,1]},\bm{y}_{[0,1]})\\
&=\int\Big(\int_0^1\mathcal{L}(t,\bm{x}_t,\bm{y}_t,\dot{\bm{x}}_t,\dot{\bm{y}}_t)\mathrm{d}t\Big)\gamma(\bm{z})\gamma(\bm{z}')\mathrm{d}\bm{z}\mathrm{d}\bm{z}'\\
&\ge\int\mathcal{A}(\bm{z},\bm{z}')\gamma(\bm{z})\gamma(\bm{z}')\mathrm{d}\bm{z}\mathrm{d}\bm{z}'\\
&\ge\text{QOT}_S(\mu_0,\mu_1)
\end{aligned}
\end{equation}
By taking infimum w.r.t \(\pi_t(\bm{x},\bm{y})\), we have \(\text{QOT}_C(\mu_0,\mu_1)\ge\text{QOT}_S(\mu_0,\mu_1)\).
\end{proof}
Although we can prove the inverse inequality, unfortunately, the original side cannot be guaranteed in general under this condition. That is, we no longer have \(\text{QOT}_C(\mu_0,\mu_1)\le\text{QOT}_S(\mu_0,\mu_1)\). The reason is that, we proved  the original inequality by Jensen's inequality in \ref{pf:static dynamic equivalence}. To use Jensen's inequality, we marginalize all the travelling pair conditional velocities into a marginal velocity pair
\begin{equation}\bm{u}^i_t(\bm{x},\bm{y})=\int\bm{u}^i_t(\bm{x},\bm{y}|\bm{z},\bm{z}')\frac{\pi_t(\bm{x},\bm{y}|\bm{z},\bm{z}')q(\bm{z},\bm{z}')}{\pi_t(\bm{x},\bm{y})}\mathrm{d}\bm{z}\mathrm{d}\bm{z}'
\end{equation}
where \(q=\gamma\otimes\gamma\). The point here is that although \(q\) has the tensor structure, the probability flow \(\pi_t(\bm{x},\bm{y})\) generated by the marginal velocity pair does not naturally satisfy the QOT compatible condition. Therefore, we can only prove  \(\text{QOT}_S(\mu_0,\mu_1)\ge\text{QOT}_{BB}(\mu_0,\mu_1)\), but can never prove  \(\text{QOT}_S(\mu_0,\mu_1)\ge\text{QOT}_C(\mu_0,\mu_1)\). The optimal result we can obtain is therefore \(\text{QOT}_C(\mu_0,\mu_1)\ge\text{QOT}_S(\mu_0,\mu_1)\ge\text{QOT}_{BB}(\mu_0,\mu_1)\).

\subsection{The intuitive relation between the three forms}
In the main text, we define the static QOT, the dynamic QOT and the BB-form. The intuition of the BB-form is to minimize the population level action induced by the Lagrangian during the transport, which is what, in practice, one may actually want to minimize. Since \(\text{QOT}_D(\mu_0,\mu_1)=\text{QOT}_S(\mu_0,\mu_1)\ge\text{QOT}_{BB}(\mu_0,\mu_1)\), minimizing the static QOT action is equivalent to minimizing a tractable upper bound of the BB-form. Intuitively, it also implicitly tries to minimize the BB-form action. But, the static form contains no dynamic information. To model the continuous time dynamics, we have to extend it to the dynamic form, which is equivalent to the static form. In the main text, we introduce flow matching to solve the dynamic QOT flow by obtaining a Markovian projection. By Jensen's inequality, the BB-form action induced by this marginal flow is less than the corresponding static action, but greater than the infimum BB-form action by definition. Therefore, the learned flow can be viewed as a finer upper bound approximation of the true BB-form action.


\end{document}

%% file: math_commands.tex
\usepackage{amsmath,amsfonts,bm}

\def\eqref#1{equation~\ref{#1}}

\def\1{\bm{1}}

\DeclareMathAlphabet{\mathsfit}{\encodingdefault}{\sfdefault}{m}{sl}
\SetMathAlphabet{\mathsfit}{bold}{\encodingdefault}{\sfdefault}{bx}{n}

